\documentclass[11pt, a4paper]{article}
\usepackage{lineno,hyperref}

\usepackage[hmarginratio=1:1,top=15mm,right=18mm,left=18mm,bottom=15mm,columnsep=20pt, headsep=7mm]{geometry}

\providecommand*{\Dom}[1]{{\rm dom}}							
\providecommand{\diam}{\operatorname{diam}}						
\providecommand{\argmin}{\operatorname*{argmin}}				

\providecommand*{\cspan}[1]{\overline{\operatorname{span}}\left\{{#1}\right\}}
\providecommand{\diag}{\operatorname{diag}}						

\providecommand*{\N}[1]{\left\|{#1}\right\|} 					
\providecommand*{\textSN}[1]{\lvert{#1}\rvert} 					
\newcommand*{\DP}[2]{\left<{#1},{#2}\right>} 					

\newcommand{\VA}{{\mathbf{A}}}
\newcommand{\VB}{{\mathbf{B}}}

\newcommand{\VH}{{\mathbf{H}}}
\newcommand{\VI}{{\mathbf{I}}}

\newcommand{\VK}{{\mathbf{K}}}

\newcommand{\VP}{{\mathbf{P}}}

\newcommand{\VU}{{\mathbf{U}}}

\newcommand{\VSigma}  {\mathbf{\Sigma}}

\newcommand{\Valpha}     {{\boldsymbol{\upalpha}}}
\newcommand{\Vbeta}      {{\boldsymbol{\upbeta}}}

\newcommand{\abs}{{\boldsymbol{\mathsf{a}}}}
\newcommand{\bbs}{{\boldsymbol{\mathsf{b}}}}

\newcommand{\dbs}{{\boldsymbol{\mathsf{d}}}}

\newcommand{\sbs}{{\boldsymbol{\mathsf{s}}}}

\newcommand{\xbs}{{\boldsymbol{\mathsf{x}}}}
\newcommand{\ybs}{{\boldsymbol{\mathsf{y}}}}
\newcommand{\zbs}{{\boldsymbol{\mathsf{z}}}}

\providecommand{\CA}{{\cal A}}
\providecommand{\CB}{{\cal B}}
\providecommand{\CC}{{\cal C}}
\providecommand{\CD}{{\cal D}}

\providecommand{\CH}{{\cal H}}
\providecommand{\CI}{{\cal I}}

\providecommand{\CL}{{\cal L}}

\providecommand{\CN}{{\cal N}}
\providecommand{\CO}{{\cal O}}

\providecommand{\CS}{{\cal S}}

\providecommand{\CX}{{\cal X}}
\providecommand{\CY}{{\cal Y}}

\providecommand{\bbE}{\mathbb{E}}

\providecommand{\bbN}{\mathbb{N}}

\providecommand{\bbR}{\mathbb{R}}

\providecommand{\Fc}{\mathfrak{c}}

\providecommand{\Ff}{\mathfrak{f}}

\newcommand{\dd}{\texttt{ddim}}
\newcommand{\childp}{\widetilde{C}_p}

\newcommand{\vol}{\operatorname*{Vol}}
\newcommand{\method}{\textsf{StreaMRAK }}
\newcommand{\falkon}{\textsf{FALKON }}

\newcommand{\lapkrr}{\textsf{LP-KRR }}

\newcommand{\methodE}{\textsf{StreaMRAK}}
\newcommand{\falkonE}{\textsf{FALKON}}

\newcommand{\lapkrrE}{\textsf{LP-KRR}}

\newcommand{\cof}{\Fc\Ff}
\newcommand{\cofp}{\Fc\Ff(p)}
\newcommand{\coflvl}{\Fc\Ff(Q_l)}

\newcommand{\rl}{r_{l}}
\newcommand{\rlp}{r_{l+1}}

\newcommand{\figcap}{\vspace{-5pt}}

\newcommand{\tabcap}{\vspace{-7pt}}

\usepackage{amsfonts, amsmath, amssymb, amsthm}
\usepackage{mathtools}
\usepackage{enumerate,enumitem}
\usepackage{upgreek}
\usepackage{bbm}

\usepackage{xcolor}
\usepackage{graphics}

\usepackage{array}
\usepackage{etoolbox}

\usepackage{caption,float}
\usepackage{subcaption}
\usepackage{bm}
\usepackage[section]{placeins} 

\usepackage{multirow}

\usepackage{siunitx}
\usepackage{textcomp} 
\usepackage{gensymb}

\usepackage{algorithm}
\usepackage{algorithmic}

\theoremstyle{plain}
	\newtheorem{theorem}{Theorem}[section]
	\newtheorem{proposition}[theorem]{Proposition}
	\newtheorem{lemma}[theorem]{Lemma}
	\newtheorem{corollary}[theorem]{Corollary}
	\newtheorem{assumption}{Assumption}
	\newtheorem{remark}[theorem]{Remark}
	\newtheorem{definition}[theorem]{Definition}
	
	\newtheorem{problem}{Problem}

\allowdisplaybreaks

\usepackage{layout}
\numberwithin{equation}{section}

\begin{document}

\title{\vspace{0.8cm}StreaMRAK a Streaming Multi-Resolution Adaptive Kernel Algorithm\vspace{0.5cm}}

\newcommand{\footremember}[2]{%
    \footnote{#2}
    \newcounter{#1}
    \setcounter{#1}{\value{footnote}}%
}
\newcommand{\footrecall}[1]{%
    \footnotemark[\value{#1}]%
} 

\author{%
  Andreas Oslandsbotn\footremember{uio}{University of Oslo, Problemveien 7, 0315 Oslo, Norway}\footremember{simula}{Simula School of Research and Innovation, Martin Linges Vei 25, 1364 Fornebu, Norway}%
  \and 
  {\v Z}eljko Kereta\footremember{ucl}{University College London, Gower St, WC1E 6BT London, England}%
  \and 
  Valeriya Naumova\footremember{simula2}{Simula Research Laboratory, Martin Linges vei 25, 1364 Fornebu, Norway}%
  \and 
  Yoav Freund\footremember{ucsd}{University of California San Diego, 9500 Gilman Dr, La Jolla, CA 92093, United States}%
  \and Alexander Cloninger\footrecall{ucsd}%
  }
\date{}

\maketitle
\begin{abstract}
Kernel ridge regression (KRR) is a popular scheme for non-linear non-parametric learning. 
However, existing implementations of KRR require that all the data is stored in the main memory, which severely limits the use of KRR in contexts where  data size far exceeds the memory size.
Such applications are increasingly common in data mining, bioinformatics, and control. 
A powerful paradigm for computing on data sets that are too large for memory is the {\em streaming model of computation}, where we process one data sample at a time, discarding each sample before moving on to the next one.

In this paper, we propose \textsf{StreaMRAK} - a streaming version of KRR. \method improves on existing KRR schemes by dividing the problem into several levels of resolution, which allows continual refinement to the predictions. 
The algorithm reduces the memory requirement by continuously and efficiently integrating new samples into the training model. 
With a novel sub-sampling scheme, \method reduces memory and computational complexities by creating a {\em sketch} of the original data, where the sub-sampling density is adapted to the bandwidth of the kernel and the local dimensionality of the data.

We present a showcase study on two synthetic problems and the prediction of the trajectory of a double pendulum. The results show that the proposed algorithm is fast and accurate.
\end{abstract}

\section{Introduction}
Machine learning algorithms based on kernel ridge regression (KRR) \cite{scholkopf2002learning} is
an active field of research \cite{rudi2017falkon, Alaoui2014, Zhang2015divide, Avron2017, Burnaev2017}, with applications ranging from time series prediction in finance \cite{Exterkate2016Nonlinear}, parameter inference in dynamical systems \cite{Niu2016}, to pairwise learning \cite{Stock2018comparative}, face recognition \cite{An2007face} and drug estimation and gene analysis in biomedicine \cite{Li2015, Mohapatra2016Microarray}.
This paper develops a streaming variation of KRR using a radial kernel, a new sub-sampling scheme, and a multi-resolution formulation of the learning model.

Many popular data analysis software packages, such as $\mbox{Matlab}^{\mbox{\tiny TM}}$ require loading the entire dataset into memory. While the size of computer memory is growing fast, the size of available data sets is growing much faster, limiting the applicability of \textit{in-memory} methods.~\footnote{$\mbox{Simulink}^{\mbox{\tiny TM}}$, a companion software to $\mbox{Matlab}^{\mbox{\tiny TM}}$ supports streaming but has a much more limited computational model, targeted at signal processing applications.}

Streaming~\cite{muthukrishnan2005data} is a computational model where the input size is much larger than the size of memory. Streaming algorithms read one input item at a time, update their memory, and discard the item. The computer memory is used to store a {\em model} or a {\em sketch} of the overall data distribution, which is orders of magnitude smaller than the data itself. The development of streaming algorithms is experiencing increased popularity in the face of big data applications such as data mining \cite{fan2013mining} and bioinformatics \cite{lan2018survey}, where data sets are typically too large to be kept \textit{in-memory}. 
Many big data applications call for non-linear and involved models, and thus, the development of non-parametric and non-linear models is critical for successful learning.

Among the most popular non-parametric learning algorithms are kernel methods, which include well-known learning schemes such as the support vector machine (SVM) and KRR, to name a few. The appeal of kernel methods lies in their strong theoretical foundation \cite{scholkopf2002learning, kivinen2002online}, as well as their ability to map complex problems to a linear space without requiring an explicit mapping. A common class of kernels are radial kernels $k(\xbs,\tilde\xbs)=\Phi(\|\xbs- \tilde\xbs\|/r)$ for $\xbs, \tilde\xbs\in\CX\subseteq\bbR^D$ and $r>0$ \cite{scovel2010radial}. An example is a Gaussian kernel, for which the shape parameter $r$ is referred to as the kernel bandwidth. What is more, radial kernels are universal kernels (with a few exceptions \cite{micchelli2006universal}), meaning that they can approximate any bounded continuous function on $\CX$ arbitrarily well. However, in high dimensions kernel methods suffer from the "curse of dimensionality" and require large amounts of training data to converge. Furthermore, the computational complexity, memory requirement, and the number of parameters to learn grow unbounded with the number of training samples, a drawback known as the "curse of kernelization" \cite{wang2012breaking}. In the context of streaming, the prospect of unbounded data streams makes this shortcoming even more detrimental.

Although kernel-based learning schemes are typically formulated as convex optimization problems, which do not require tuning hyper-parameters such as learning rate etc., there is still a need to determine the optimal kernel. For the Gaussian kernel, this amounts to selecting the bandwidth. Classically, an optimal kernel is chosen through batch techniques such as leave-one-out and k-fold cross-validation \cite{loader1999bandwidth, cawley2004fast, arlot2010survey}. However, these approaches are inefficient as they spend significant time evaluating bad kernel hypotheses and often use multiple runs over the data, which is impossible in a streaming setting.

Despite the universality of radial kernels on $\CX\in\bbR^D$, this only guarantees the convergence of the model in the asymptotic regime and does not provide finite sample bounds. 
As a reaction, several works have shown the benefit of combining multiple kernels from a dictionary of kernel hypotheses. 
These strategies include multi-kernel learning (MKL) \cite{Zhang2015divide, lanckriet2004learning, bach2004multiple, sonnenburg2006large, buazuavan2012fourier}, multi-scale analysis \cite{bermanis2013multiscale, Rabin2018multiscale}, and the Laplacian pyramid (LP) \cite{Rabin2012, Leeb2019}. 
Combining these strategies with a localized kernel gives a frequency and location-based discretization similar to multi-resolution analysis, a well-established concept in signal processing and functional approximation through concepts such as wavelets \cite{graps1995introduction, akansu2010emerging}, diffusion wavelets \cite{coifman2006diffusion, maggioni2008diffusion}, and graph wavelets \cite{hammond2011wavelets, cloninger2021natural, vito2021waveletframes}.

To meet a need for non-linear non-parametric algorithms for streaming data, we propose the \textit{streaming multi-resolution adaptive kernel algorithm} (\textsf{StreaMRAK}) - a computationally and memory-efficient streaming variation of KRR. \method is a streaming algorithm that combines a streaming sub-sampling scheme with a multi-resolution kernel selection strategy and adapts the kernel bandwidth $r$ and the sub-sample density to each other over several levels of resolution. Furthermore, \method addresses the curse of dimensionality in a novel way, through the sub-sampling scheme and multi-resolution formulation.

\subsection{Setting}
\label{subsection:setting}
We consider a finite sample data-cloud $\CX$, $|\CX|=n$, that is sampled i.i.d. according to a fixed but unknown distribution ${\cal P}$ over $\bbR^D$. The target is a bounded and continuous function $f:\bbR^D \to \bbR$. 
We assume that the points in $\CX$ are placed 
in a {\em sequence} and that their order is random.~\footnote{The assumption that the sequence is randomly ordered allows us to draw statistical conclusions from prefixes.} 
Each instance $\xbs_i\in\CX$, for $i\in[n]$,  paired with a label $y_i$ where $y_i=f(\xbs_i)+\varepsilon_i$ and $\varepsilon_i\sim\CN(0,\sigma)$ represents noise. The task of learning is to train a model $\widehat{f}$ that is a good approximation of the target function $f$.

In this work, we think about the intrinsic dimension of $\CX$ as a local quantity, meaning it depends on the region $\CA \subseteq\CX$ and the radius $r$ at which we consider the point cloud.
Rooted in this way of thinking about the intrinsic dimension, \method is designed to handle domains where the local intrinsic dimension changes across different regions and resolutions.

To estimate the local intrinsic dimension in a "location and resolution sensitive" manner, we use the concept of the doubling dimension of a set, defined in Def. \ref{def:doubling_dim_at_range}. We note that our definition of the doubling dimension is related to the definition used in \cite{krauthgamer2004navigating, Beygelzimer2006}.

\begin{definition}[Covering number]
Consider a set $\CA$ and a ball $\CB(\xbs, r)$, with $r>0$ and $\xbs\in\CA$. 
We say that a finite set $\CS\subset\CB(x,r)$ is a covering of $\CB(x,r)$ in $\CA$ if $\CA\cap\CB(x,r)\subset \cup_{\xbs_i\in\CS} \CB(\xbs_i, r/2)$.
We define the covering number $\kappa(\CA, \xbs, r)$ as the minimum cardinality of any covering of $\CB(\xbs,r)$ in $\CA$.
\label{def:covering_number}
\end{definition}

\begin{definition}(Doubling dimension)
The doubling dimension $\dd(\CA, r)$  of a set $\CA$ is defined as $\dd(\CA, r) = \lceil\log \kappa(\CA,\xbs,r)\rceil$.
For an interval $\CI\subset \bbR_{>0}$ we define the doubling dimension as the least upper bound over $r\in\CI$, that is $\dd(\CA,\CI)=\max_{r\in\CI}\dd(\CA,r)$.
\label{def:doubling_dim_at_range}
\end{definition}

Using Def. \ref{def:doubling_dim_at_range} we say that the intrinsic dimension of $\CX$ changes with the location if there exist $\CA_1,\CA_2\subset\CX$ such that $\dd(\CA_1,r)\neq \dd(\CA_2,r)$ for some $r>0$. Similarly, we say that the intrinsic dimensionality of $\CX$ changes with the resolution, if there exist $r_1 \neq r_2$ such that the doubling dimension $\dd(\CA, r_1) \neq \dd(\CA, r_2)$ for $\CA\subset\CX$. 

In Fig. \ref{fig:Examples_on_intrinsic_dim} we consider three examples to provide further insight for the doubling dimension. In Fig. \ref{subfig:change_loc} we see a domain shaped like a dumbbell, where the spheres are high dimensional, and the bar connecting them is lower-dimensional, showing how the dimension can change with the location. Meanwhile, Fig. \ref{subfig:change_noise} illustrates a lower-dimensional manifold, embedded in $\bbR^3$, with manifold noise $\zeta_m$. We see that when the resolution is sufficiently small, so that $r\approx\zeta_m$, the doubling dimensionality increases towards the dimension of the ambient space $\bbR^D$. Furthermore, Fig. \ref{subfig:change_scale} shows a point cloud that is locally $2$-dimensional, but is embedded in a $3$-dimensional space. By reducing $r$ we can resolve this lower dimensionality, but if it is reduced further, we would eventually resolve the noise level, and the doubling dimension increases again. 

We also mention two special cases.
First, for large enough $r$ any set in $\bbR^D$ has a doubling dimension of at most $D$.
Second, if $\CA$ is a finite set of points in $\bbR^D$ and $r$ is smaller than the minimal distance between two points, then the number of balls of radius $r'<r$ required to cover $\CA$ is at most the number of points. Therefore the doubling dimension of $\CA$ at the range $(0,r]$ is zero. In other words, {\em any} actual (and therefore finite) training set has dimension zero for a small enough $r$, as illustrated in Fig. \ref{subfig:change_scale}. 

\begin{figure}[htb!]
    \centering
    \subfloat[\label{subfig:change_loc}\centering Intrinsic dimension changes with location.]{{\includegraphics[width=0.26\textwidth]{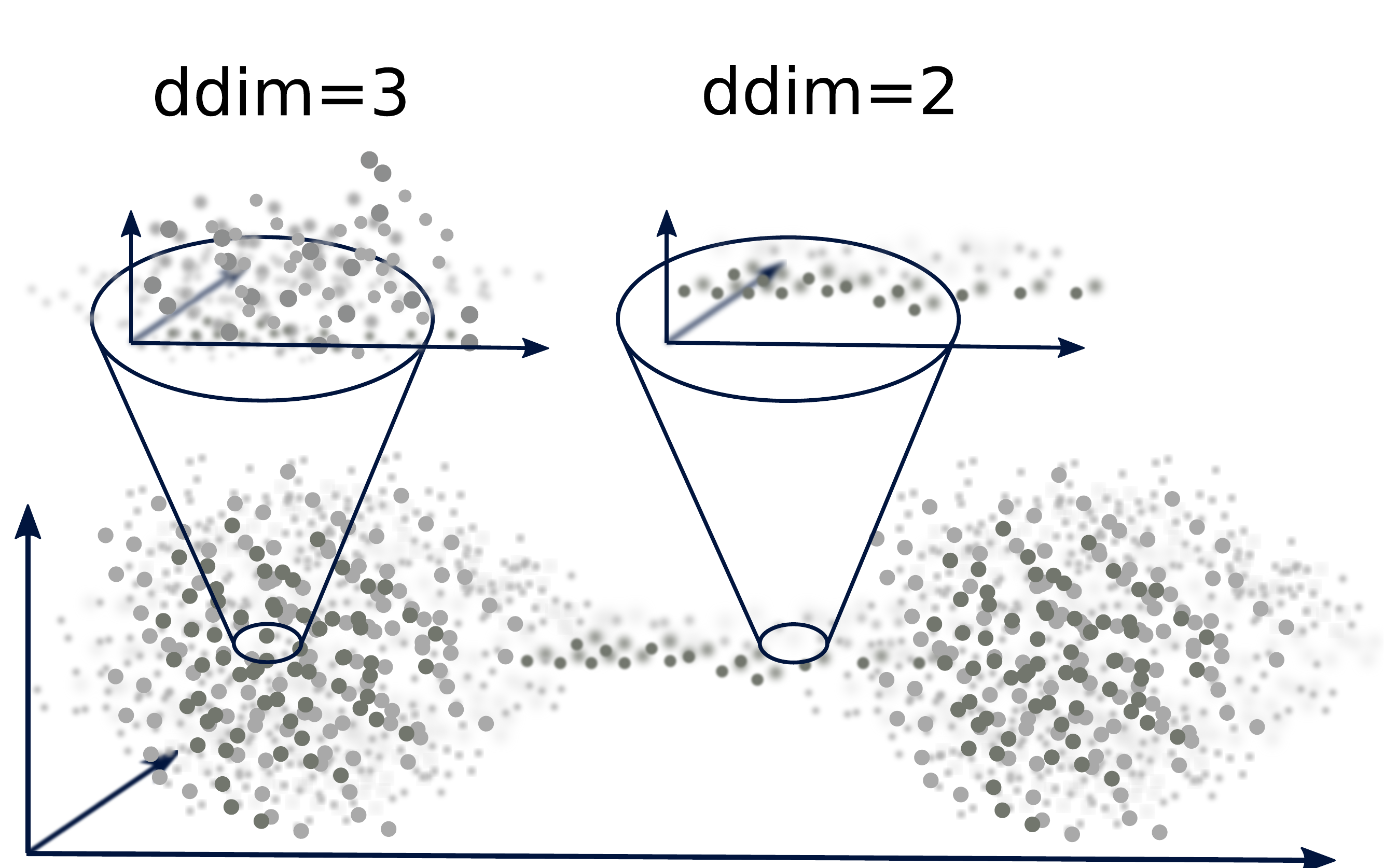} }}%
    \hfill
    \subfloat[\label{subfig:change_noise}\centering Noise has high intrinsic dimension.]{{\includegraphics[width=0.26\textwidth]{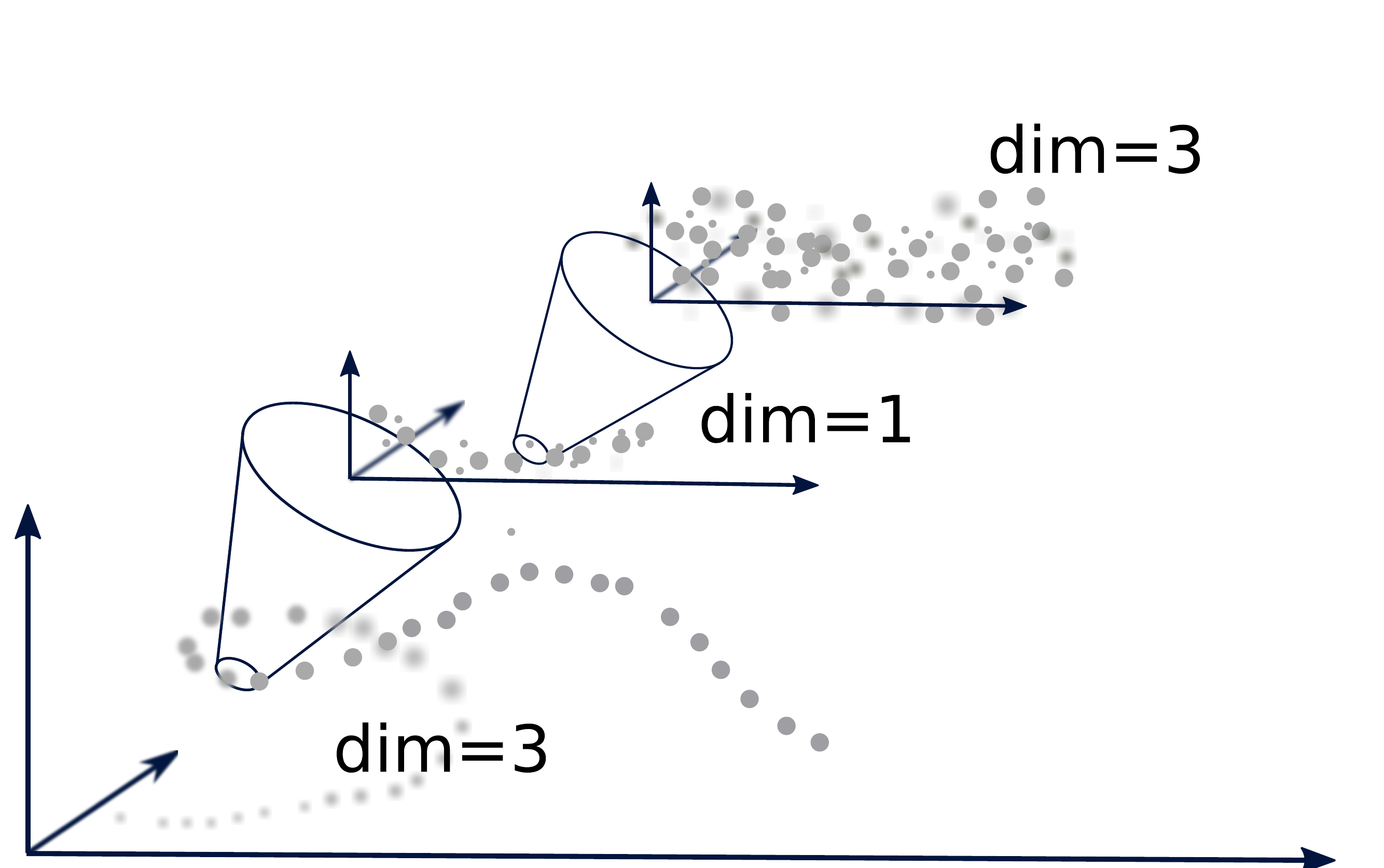} }}%
    \hfill
    \subfloat[\label{subfig:change_scale}\centering Intrinsic dimension changes with the resolution.]{{\includegraphics[width=0.26\textwidth]{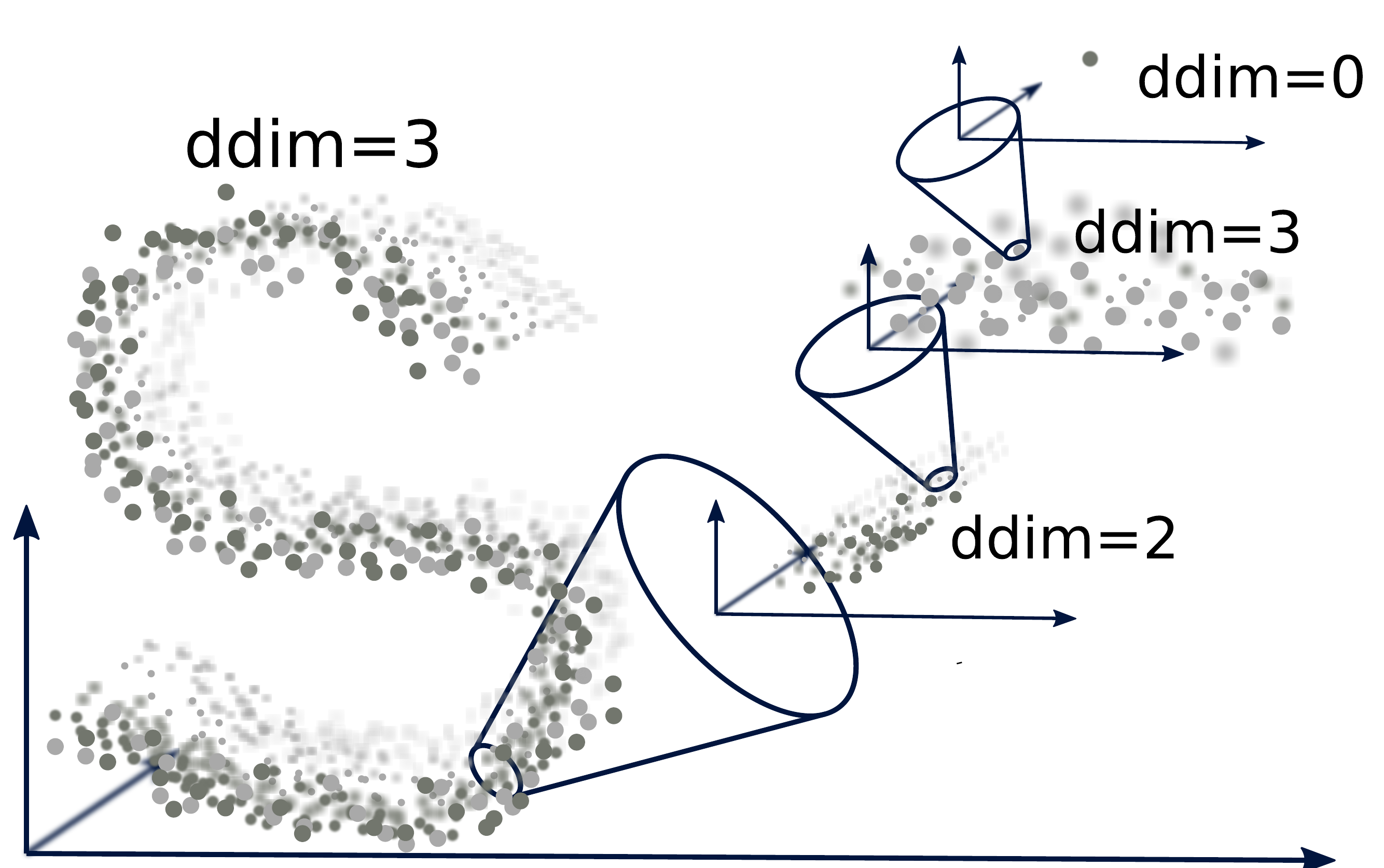} }}%
    \figcap
    \caption{\label{fig:Examples_on_intrinsic_dim}Three examples of variation in the intrinsic dimension. The coloring of the point clouds illustrates depth.}
\end{figure}

As an example of how the intrinsic dimension might change with respect to regions and resolutions, we consider a double pendulum system, a well-known chaotic system that depends heavily on its initial conditions \cite{shinbrot1992chaos}.
Systems with multiple pendulum elements are well known in engineering applications such as mechanical and robotic systems with several joints and are studied for their chaotic properties \cite{marcelo2016chaos}.

Let $\omega_i=\dot{\theta_i}$, for $i=1,2$, be angular velocities. We initialize $500$ pendulums with $\theta_1,\theta_2=180$, $\omega_1=10^{-2}$ and $\omega_2=10^{-1}$, measured in degrees, and perturb the angular velocities by $\varepsilon_i \sim \CN(0, 10^{-2}|\omega_i|)$. We iterate the system for $T=500$ time steps and let $\sbs^{(i)}_t=[\theta_1(t),\, \theta_2(t),\, \omega_1(t),\, \omega_2(t)]\in \bbR^4$ be the state of pendulum $i$ at $t\in \bbN$. We think of each state as a training point in a point cloud $\CX\subset\bbR^4$, where for instance the target function can be $\sbs^{(i)}_{t+\Delta} = f(\sbs^{(i)}_t)$, with $\Delta\in\bbN$. 

In Fig. \ref{subfig:phaseDiagramIntroduction_2d} we visualize the trajectory of four pendulums $P_0, P_1, P_2, P_3$, for which the trajectories are indistinguishable until a bifurcation occurs around $T=300$ time steps, and the trajectories start to diverge. In Fig. \ref{subfig:phaseDiagramIntroduction_3d_before_bifurcation} and Fig. \ref{subfig:phaseDiagramIntroduction_3d} we zoom in on the trajectory of all $500$ pendulums in regions before and after the bifurcation. These two regions, $\CA_1$ and $\CA_2$, are indicated by a blue and red circle, respectively, in Fig. \ref{subfig:phaseDiagramIntroduction_2d}. From the figures, it is clear that learning the trajectory in $\CA_1$ is significantly easier than in $\CA_2$, where learning the trajectory is more affected by the curse of dimensionality. 

We also note that the trajectories that remain close after the bifurcation will remain so until a new bifurcation occurs. The take-home message is that predicting the trajectory of a double pendulum is a hard problem because of a few regions where the intrinsic dimension blows up and makes the prediction hard. However, between these regions, the trajectory is easier to describe.
The spirit of this work aims to reduce the effort in such regions $\CA_2$ where the training data exhibit high intrinsic dimension and focus more on those regions $\CA_1$ where the data has a lower intrinsic dimension, i.e. is more well behaved.

\begin{figure}[htb!] 
    \centering
    \subfloat[\label{fig:DoublePendIllustration}\centering Double pendulum]{{\includegraphics[width=0.19\textwidth]{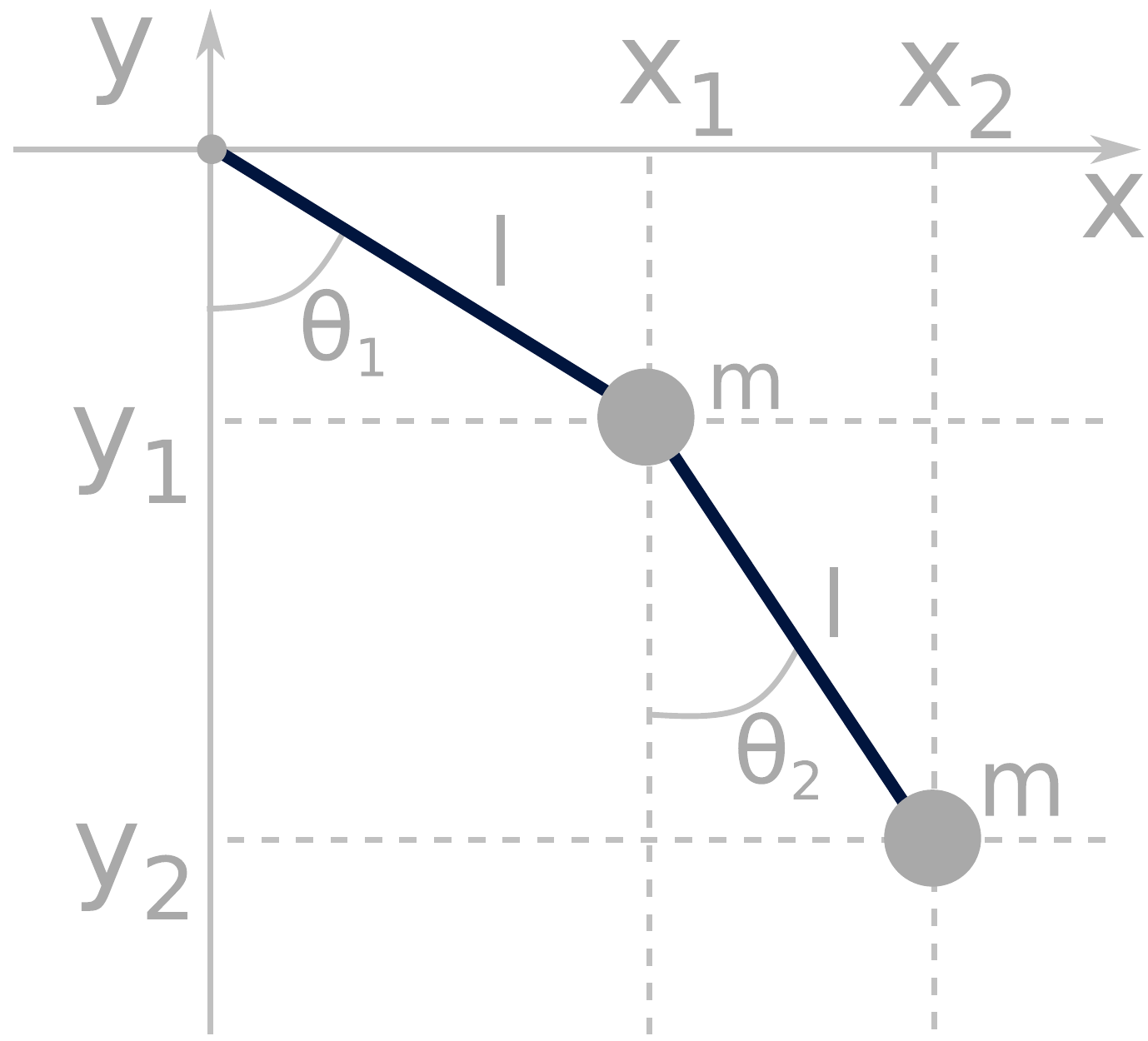} }}%
    \hfill
    \subfloat[\label{subfig:phaseDiagramIntroduction_2d}\centering 2-dim projection of the pendulum training data $\CX$.]{{\includegraphics[width=0.21\textwidth]{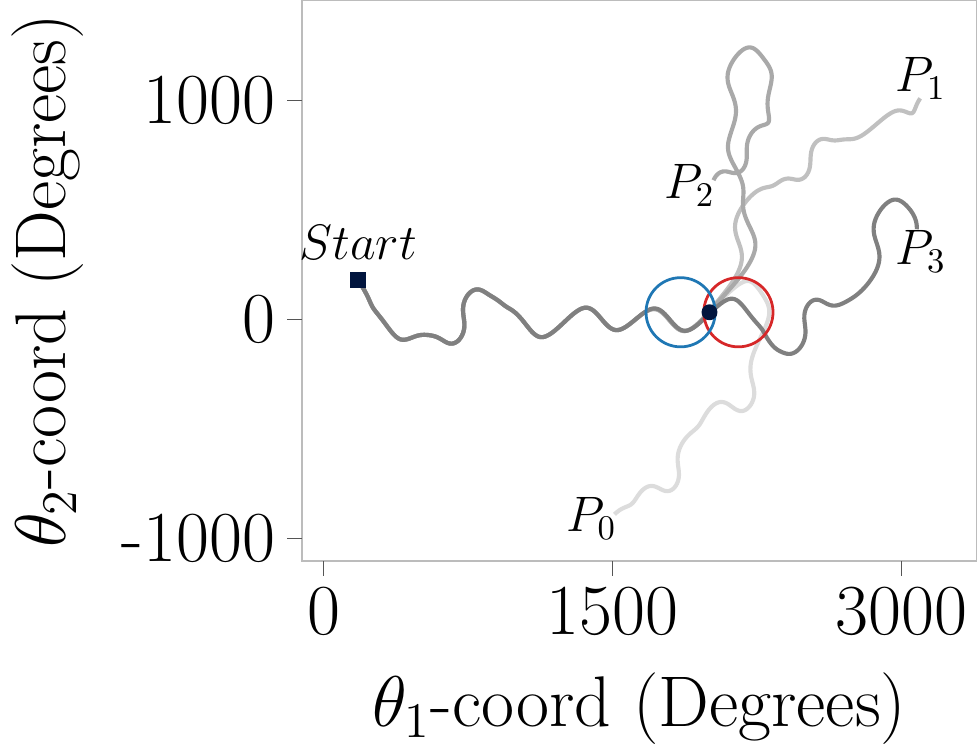} }}%
    \hfill
    \subfloat[\label{subfig:phaseDiagramIntroduction_3d_before_bifurcation}\centering 3-dim projection of the training data in the blue circle in (b).]{{\includegraphics[width=0.26\textwidth]{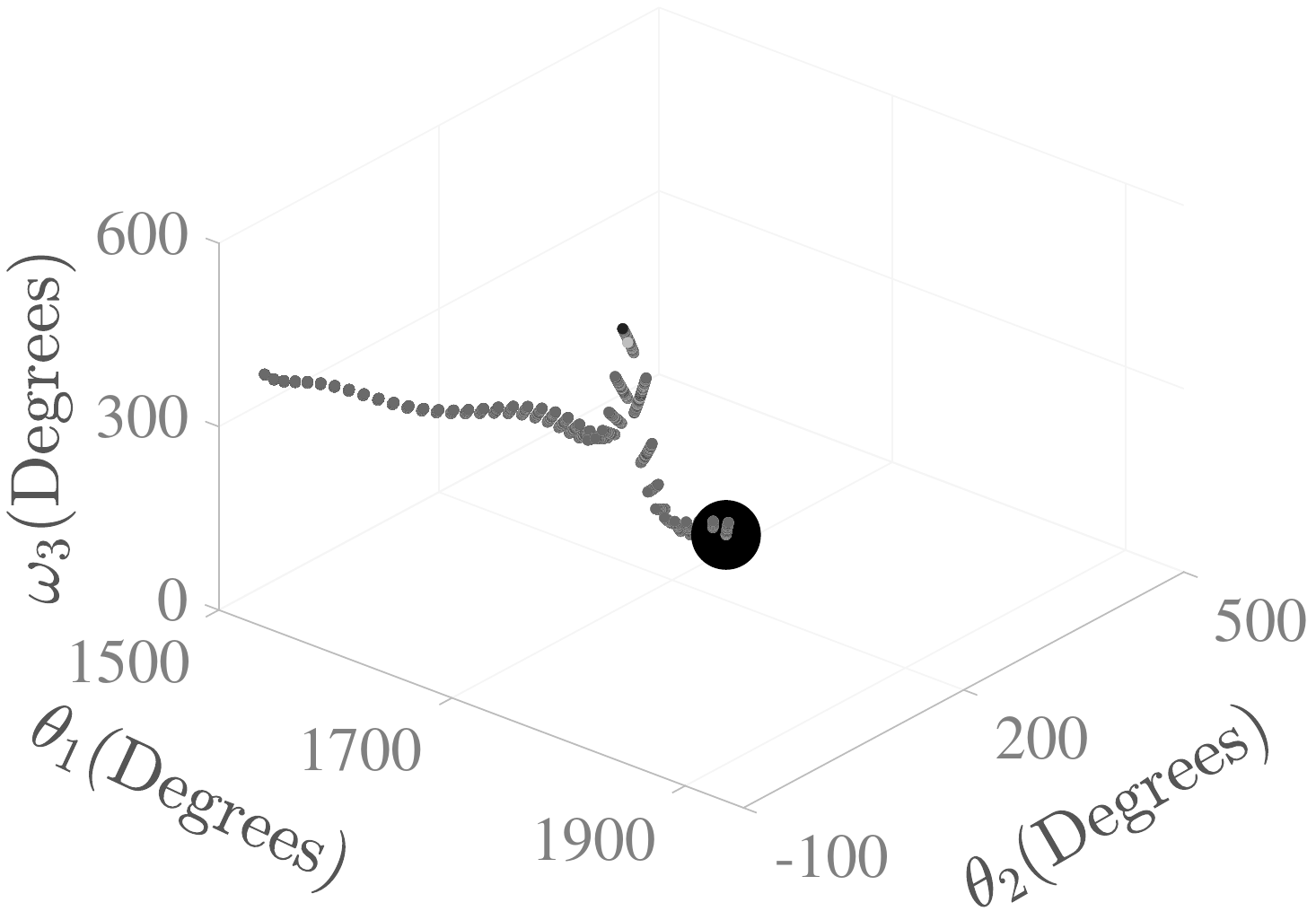} }}%
    \hfill
    \subfloat[\label{subfig:phaseDiagramIntroduction_3d}\centering 3-dim projection of the training data in the red circle in (b).]{{\includegraphics[width=0.26\textwidth]{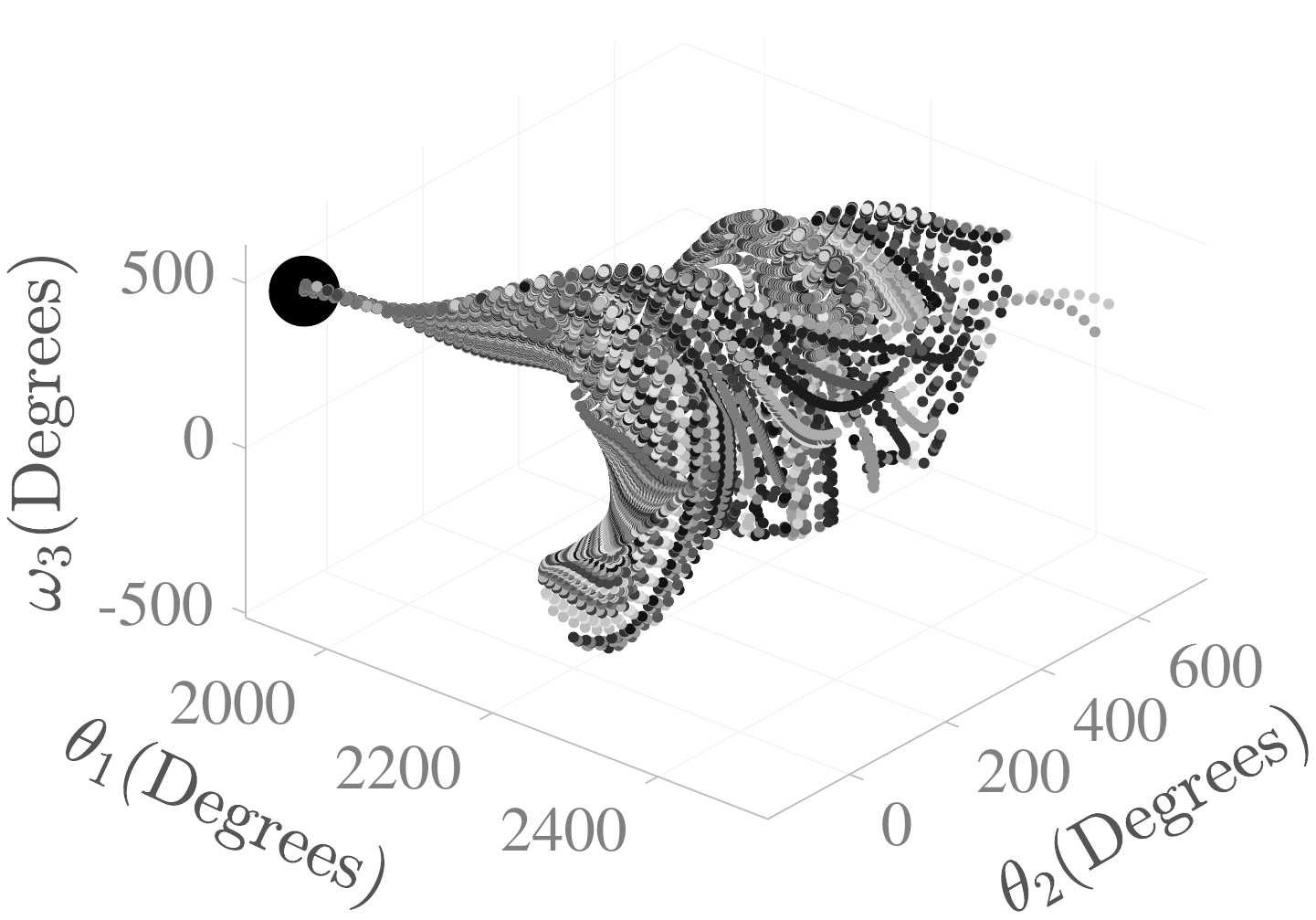} }}%
    \figcap
    \caption{(a) Illustration of a double pendulum. Here $l$ and $m$ are the length and mass of the pendulum rods, and $\theta_1, \theta_2$ are the angles. Furthermore, $(x_1, y_1)$ and $(x_2,\,y_2)$ are the positions of the point masses of the two pendulums. (b) Phase diagram of four double pendulums $P_0, P_1,P_2,P_3$, iterated for $T=500$ time steps. The bifurcation point at step $T=300$ is indicated with a black solid circle. (c) and (d) includes the $\omega_1$ axis and zoom in on respectively the blue and red circles in (b).}
    \label{fig:IntroIllustration}%
\end{figure}

\subsection{Contribution and comparison to related work}
Contributions of this work can be divided into three components. 

\begin{enumerate}[label=(C\arabic*)]
    \item \label{contribution_LP} A multi-resolution variation of the state-of-the-art KRR solver \falkon \cite{rudi2017falkon}, using the LP, which refines the predictions at each level of resolution by regressing on the errors from the previous level.

    \item \label{contribution_subSamp} A novel sub-sampling scheme for kernel methods, tailored for use in combination with the LP, that can handle the curse of dimensionality and does not require the data to be \textit{in-memory}.
    
    \item \label{contribution_Stream} Development of a streaming variation of \falkonE, where the time and memory requirements depend on the doubling dimensionality and the level of resolution, instead of the number of training points. see Props. \ref{prop:memory_streamrak_and_DCT}-\ref{prop:Time_req_solver}.
\end{enumerate}

In the following, we give further details on these contributions and compare them to related work. 

The computational backbone of \method is based on the state-of-the-art KRR solver \falkon \cite{rudi2017falkon}, which among other things combines sub-sampling and preconditioning to process large data sets efficiently.
However, \falkon relies on selecting an optimal kernel bandwidth, which can be inefficient and hard within a streaming setting.

Inspired by the success of existing multi-resolution approaches \cite{Zhang2015divide, lanckriet2004learning, bach2004multiple, sonnenburg2006large, buazuavan2012fourier,bermanis2013multiscale}, our first contribution \ref{contribution_LP} addresses the issue of selecting an optimal kernel bandwidth by introducing a multi-resolution reformulation of \falkon using a variation of the LP.

The LP scheme originated in image representation \cite{Burt2009} and was introduced to machine learning by \cite{Rabin2012} for efficient data representation.
The LP refines the prediction at several levels of resolution, and at each level, reduces the bandwidth used at the previous level by a constant factor.
In doing so, the selection of a single optimal bandwidth is avoided, and the resulting approach has greater flexibility.
The LP is similar to ideas in wavelet analysis that have shown great success in numerous applications.
However, typical wavelet architectures \cite{graps1995introduction, akansu2010emerging,coifman2006diffusion, maggioni2008diffusion, hammond2011wavelets, cloninger2021natural} require upfront construction of a wavelet basis, which is not compatible with a data-adaptive kernel. In this work, we aim to show that the LP is a viable multi-resolution scheme and can be modified to the streaming setting. Furthermore, we experimentally show that it significantly improves the estimation accuracy, and inspired by  \cite{Leeb2019} we provide convergence bounds for the LP in the context of radial kernels and KRR.

Let us now discuss our second contribution \ref{contribution_subSamp}.
\falkon addresses the curse of kernelization by combining Nystr\"{o}m sub-sampling, conjugate gradient, and preconditioning, and achieves time and memory requirements of $\mathcal{O}(n\sqrt{n})$ and $\mathcal{O}(n)$ respectively, where $n$ is the number of samples. 
In recent years there have been several efforts to address the curse of kernelization in similar ways through sub-sampling techniques such as sketching \cite{Alaoui2014, Avron2017}, randomized features \cite{rahimi2008random, le2013fastfood, yang2015carte, zhang2021sigma} and Nystr\"{o}m sub-sampling \cite{williams2001using, smola2000sparse, cloninger2017prediction, ma2018power, ma2018kernel}. However, despite their successes, these techniques are in principle \textit{in-memory} type algorithms since they require access to the training data in advance of the training and are not optimized for streaming. 

Furthermore, \falkon selects the sub-samples uniformly over the input domain $\CX$, and the LP uses the same training set for each level. However, when learning with a radial kernel, the density of samples should be related to the bandwidth of the kernel. Otherwise, a too-small bandwidth will lead to bad interpolation properties, while a too-large bandwidth gives an ill-conditioned system \cite{bermanis2013multiscale}. 
Since the LP scheme reduces the kernel bandwidth at each level of resolution, it would be problematic to use the same sub-sample density. Furthermore, due to the curse of dimensionality, the covering number increases exponentially with the doubling dimension. 
Therefore, if doubling dimensionality varies across different regions of the domain $\CX$, as illustrated by Fig. \ref{subfig:change_loc}, then the number of sub-samples necessary to maintain the density for a given bandwidth will also vary.

Our second contribution \ref{contribution_subSamp} provides an alternative sub-sampling strategy, which adapts the sub-sampling density to the kernel bandwidth. This strategy is similar to tuning the kernel bandwidth to the data, which is used in online algorithms to avoid the use of cross-validation \cite{Zhang2020, chen2016kernel, fan2016kernel}. Although expensive to calculate, especially in high dimensions, a similar strategy is used in graph-based methods, where the kernel bandwidth is adapted to the $k$-nearest neighbor distance of each training point  \cite{cheng2020convergence}.

The sub-sampling strategy we propose is based on a damped cover-tree (DCT), which is a modified version of the cover-tree (CT) \cite{Beygelzimer2006}. The CT is a tree-based data structure originally intended for nearest neighbor search in metric spaces. It is closely related to navigation nets \cite{krauthgamer2004navigating}, but with improved space requirements:  $\mathcal{O}(n)$ in memory and $\mathcal{O}(c^6n\log n)$ in time. In this work, we show how the CT structure can be used to organize the samples hierarchically with increasing density for each new level in the tree and how it can adapt the sub-sample density to kernel bandwidths in the LP. 

However, the problem with an adaptive sub-sampling strategy is its vulnerability to the curse of dimensionality. In regions of high doubling dimensions, the number of samples to achieve a certain density increases exponentially, as quantified by Def. \ref{def:doubling_dim_at_range}. This means that the number of sub-samples from the CT will quickly grow too large for efficient computing. The danger is to waste resources on samples from subsets and levels where the doubling dimension is so large that good interpolation cannot be achieved for any viable sample sizes. This would only serve to slow down the computation and not increase the precision. 

Due to this, the DCT introduces a damping property, which gradually suppresses the selection of sub-samples where the doubling dimensionality is large. This has the additional advantage of allowing to choose more sub-samples from regions where the doubling dimensionality is small. Thus, the DCT can diminish the impact of the curse of dimensionality. Furthermore, the DCT can be built continuously as new samples come in, making it ideal for a streaming computational model.

Our third contribution \ref{contribution_Stream}, relies on the changes implemented with \ref{contribution_LP} and \ref{contribution_subSamp}. In particular, \method can operate as a streaming algorithm and efficiently organizes the sub-samples as it builds the multi-resolution kernel. Furthermore, the sub-sampling and kernel construction allows for continuous integration of new training points into the kernel matrix. Moreover, the DCT, the multi-resolution construction, and the KRR solver can all be multi-threaded and parallelized.

\subsection{Organization of the paper}
The paper is organized as follows. Section \ref{section:kernel_methods} introduces kernel methods and the \falkon algorithm, as well as the LP.  Section \ref{sect:Adaptive sub-sampling} introduces the adaptive sub-sampling scheme and the DCT. \method is described in Section \ref{section:TheStreaMRAKalgo} and an analysis of the algorithm is given in Section \ref{section:analysis}. Finally, Section \ref{section:ExperimentsMain} presents several numerical experiments and Section \ref{section:Outlook} gives an outlook for further work. The Appendix includes further mathematical background and the proofs.

\subsection{Notation}
We denote vectors $\abs \in \bbR^D$ with boldface and matrices $ \VA \in \bbR^{n\times m}$ with bold uppercase, and $\VA^\top$ denotes the matrix transpose.  We use $\VK_{nm}$ for kernel matrices,  where the dimensionality is indicated by the subscripts. We reserve $n$ for the number of training samples and $m$ for the number of sub-samples. 
The $ij$-th element of a kernel matrix is denoted $[\VK_{nm}]_{ij}$, while for other matrices we use $\VA_{ij}$. The notation $a_i$ indicates $i$-th element of a vector $\abs$.
Furthermore, we use $f([\xbs_n])$ to denote $(f(\xbs_1),\ldots, f(\xbs_n))^\top \in \bbR^n$, and $[m]$ to denote $\{i\}^{m}_{i=1}$. The notation $\xbs_i$ indicates the $i$-th training example. 
We use $\abs^{(l)}$ and $\VA^{(l)}$, where $l$ refers to a specific level in the LP and the DCT. We take $\|\cdot\|$ to be the $L^2$ norm and $\|\cdot\|_{\CH}$ to be the RKHS norm.
We denote the intrinsic dimension of a manifold with $d$ and the dimension of the embedding with $D$.
By $\mathbbm{1}_{\CS}(\xbs)$ we denote the indicator function, which evaluates to $1$ if $\xbs\in\CS$ and $0$ otherwise, of a set $\CS\subset \bbR^D$.

\section{Kernel methods}
\label{section:kernel_methods}
Consider a positive definite kernel $k:\CX\times\CX\rightarrow\bbR$, defined on an input space $\CX\subset\bbR^D$.
Given data $\{(\xbs_i,y_i):i\in[n]\}$ of samples from $\CX\times\bbR^D$, kernel ridge regression computes an estimator by minimising 
\begin{align*}
        \widehat f_{n,\lambda} =\argmin_{f\in\widetilde\CH} \frac{1}{n}\sum_{i=1}^n (f(\xbs_i)-y_i)^2  + \lambda \N{f}_\CH^2,
\end{align*}
where $\CH$ is the Hilbert space induced by the kernel.
This allows to reduce the problem to a linear system
\begin{align}\label{eqn:KRR}
    (\VK_{nn}+\lambda n \VI_n)\Valpha = \ybs, \text{ for } [\VK_{nn}]_{ij}=k(\xbs_i,\xbs_j),\text{ and } \ybs=(y_1,\ldots,y_n)^\top.
\end{align}
Coefficients $\Valpha=(\alpha_1,\ldots,\alpha_n)^\top$ define the estimator by $f(\xbs)=\sum_{i=1}^n \alpha_ik(\xbs,\xbs_i)$.
However, solving \eqref{eqn:KRR} using traditional methods has a time complexity of $\CO(n^2)$, which can be costly for large $n$ \cite{rudi2017falkon}. 

\falkon \cite{rudi2017falkon} addresses this issue by sub-sampling the columns of $\VK_{nn}$, which reduces the effective complexity while maintaining accuracy.
Namely, denote $\Gamma_n=\{\xbs_1,\ldots,\xbs_n\}$ and for $m\ll n$ let $\widetilde\Gamma_m=\{\widetilde{\xbs}_1,\ldots, \widetilde{\xbs}_m\}$ be Nystr\"{o}m centers (i.e. a randomly selected subset of $\Gamma_n$).
Minimizing
\begin{equation}
    \widehat f_{n,m,\lambda} =\argmin_{f\in\widetilde\CH_M} \frac{1}{n}\sum_{i=1}^n (f(\xbs_i)-y_i)^2  + \lambda \N{f}_\CH^2,
\end{equation}
where $\widetilde\CH_m=\cspan{k(\cdot, \widetilde{\xbs}_j):j\in[m]}$, leads to a linear system 
\begin{align*}
\VH \widetilde{\Valpha} = \zbs, \text{ for } \VH = \VK_{nm}^\top\VK_{nm} + \lambda n\VK_{mm}, \text{ and } \zbs =\VK_{nm}\ybs.
\end{align*}
Here $[\VK_{nm}]_{ij}=k(\xbs_i, \widetilde{\xbs}_j)\in\bbR^{n\times m}$ is the column-subsampled matrix and the estimator is given by $\widehat f_{n,m,\lambda}(\xbs)=\sum_{j=1}^m \widetilde{\alpha}_j k(\xbs,\widetilde{\xbs}_j)$.
To further reduce the time complexity \falkon uses a suitable preconditioner to reduce the condition number.
The preconditioner is defined as $  \VB\VB^\top = (n/m \VK_{mm}^2+\lambda n\VK_{mm})^{-1}$,
which is a natural (lower complexity) approximation of the ideal preconditioner $\VA\VA^\top = (\VK_{nm}^\top\VK_{nm}+\lambda n\VK_{mm})^{-1}$.
We now solve for $\widetilde{\Valpha}$ from the system of equations
\begin{equation}
\VB^\top\VH\VB\Vbeta = \VB^\top\zbs, \text{ for } \VH = \VK_{nm}^\top\VK_{nm} + \lambda n\VK_{mm}, \text{ } \zbs =\VK_{nm}\ybs, \text{ and } \widetilde{\Valpha}=\VB\Vbeta.
\label{eq:Falkon_system}
\end{equation}
This is solved iteratively, using the conjugate gradients with early stopping.
Choosing $m=\CO(\sqrt{n})$ still ensures optimal generalisation (i.e. same as KRR), while reducing the computational complexity to $\CO(n\sqrt{n})$.

\subsection{Streaming adaptation of \falkon}
Matrices and vectors involved in the linear system in \eqref{eq:Falkon_system} can be separated into two classes: those that depend only on sub-samples in $\widetilde\Gamma_m$; and those ($\VK_{nm}^\top\VK_{nm}$ and $\zbs$) that also depend on all the training points $\Gamma_n$.
Critically, terms in both groups are all of size $m$, which allows to reduce the complexity.
Consider now the set of sub-samples $\widetilde\Gamma_m$ to be fixed, and assume new training points, in the form $\{(\xbs_q, y_q):q=n+1,\ldots, n+t)\}$, are coming in a stream.
We can then update the second class of terms according to
\begin{align}
    \big[ (\VK_{(n+t)m})^\top \VK_{(n+t)m} \big]_{ij} &= \big[ (\VK_{nm})^\top \VK_{nm} \big]_{ij} + \sum_{q=n+1}^{n+t} k(\xbs_q, \widetilde{\xbs}_i)k(\xbs_q, \widetilde{\xbs}_j), 
    \label{eq:updateFormula_KnmTKnm}\\
\big[(\VK_{(n+t)m})^\top \ybs\big]_i&= z_i+\sum_{q=n+1}^{n+t} k(\xbs_q, \widetilde{\xbs}_i)y_q.
    \label{eq:updateFromualte_Zm}
\end{align}

Thus, only sub-samples $\widetilde\Gamma_m$, matrices $\big(\VK_{nm}\big)^\top\VK_{nm}$, $\VK_{mm}$ and $\zbs$, need to be stored.
However, in order to continuously incorporate new training points into Eqs. \eqref{eq:updateFormula_KnmTKnm} and  \eqref{eq:updateFromualte_Zm}, sub-samples $\widetilde\Gamma_m$ must be determined in advance. 
Whereas this works if all the data is provided beforehand, it cannot be done if the data arrives sequentially. 
In this work, we address this through a multi-resolution framework. The overall estimator is composed of a sequence of estimators defined at different resolution levels of the domain. 
Correspondingly, the set of sub-samples $\widetilde \Gamma_m$ consists of smaller sets $\widetilde \Gamma_{m^{(l)}}^{(l)}$ that correspond to individual levels of resolution.
The sets $\widetilde \Gamma_{m^{(l)}}^{(l)}$ are filled as the data streams in, and once a set for a given level is deemed complete, we proceed with updating  \eqref{eq:updateFormula_KnmTKnm} and  \eqref{eq:updateFromualte_Zm}.

Further details of how the sets $\widetilde \Gamma_{m^{(l)}}^{(l)}$ are constructed, and the corresponding criteria, are provided in Sections \ref{sect:Adaptive sub-sampling} and \ref{section:TheStreaMRAKalgo}. 
We begin by describing the multi-resolution framework of estimators.

\subsection{The Laplacian pyramid}
\label{section:TheLaplacianPyramid}
The LP \cite{Burt2009, Rabin2012} is a multi-resolution regression method for extending a model $\widehat{f}$ to out-of-sample data points $\xbs\in \CX/\Gamma_n$. The LP can be formulated for radial kernels in the form
\begin{equation}
    k(\xbs_i, \xbs_j) = \Phi\bigg(\frac{\|\xbs_i-\xbs_j\|}{r}\bigg),
    \label{eq:radialBasisFunction}
\end{equation}
where $r>0$ is a shape parameter that determines the decay of $\Phi$ with respect to $\|\xbs_i-\xbs_j\|$, see \cite{scovel2010radial}. 
The idea underpinning the LP is to approximate the target function sequentially,  where at each stage we regress on the errors from the previous stage. In other words, we begin with a rough approximation using a large shape parameter for which $\Phi$ decays slowly and then improve the approximation by fitting the resulting error and reducing the shape parameter.
In the LP, the estimator at level $L\in\bbN$ is defined recursively as
\begin{equation}
    \widehat{f}^{(L)}(\xbs) = \sum_{l=0}^L s^{(l)}(\xbs) = s^{(L)}(\xbs)+\widehat f^{(L)}(\xbs),
    \label{eq:laplacian_pyramid_model}
\end{equation}
where  $\hat f^{(0)}=s^{(0)}$, and $s^{(l)}(\xbs)$ is a correction term defined by
\begin{equation}
    s^{(l)}(\xbs) = \sum_{i=1}^n \alpha^{(l)}_i k^{(l)}(\xbs,\xbs_i).
    \label{eq:LP_approx_at_specific_level}
\end{equation}
The coefficients $\Valpha^{(l)}=(\alpha^{(l)}_1,\ldots,\alpha^{(l)}_n)^\top$ are computed by conducting KRR on the residuals, i.e. errors, from the estimator at the previous level. Namely, 
$\Valpha^{(l)} = (\VK^{(l)}_{nn} + \lambda n \VI)^{-1}\dbs^{(l)}$, where 
\begin{equation}
    \dbs^{(l)} = \begin{cases} 
    \ybs, & \text{if}\quad l = 0 \\ 
    \ybs-\widehat{f}^{(l-1)}([\xbs_n]), & \text{otherwise}
    \end{cases}.
    \label{eq:residual_at_level}
\end{equation}

For a \falkon adaption of this scheme, we only need to modify how per-level coefficients are computed.
Following \eqref{eq:Falkon_system} we iteratively solve
\begin{equation}
    (\VB^{(l)})^\top \VH^{(l)} \VB^{(l)} \Vbeta^{(l)} = (\VB^{(l)})^\top \big(\VK^{(l)}_{nm}\big)^\top\dbs^{(l)},
\label{eq:FALKONlinearSystem_in_LP}
\end{equation} 
where $\VB^{(l)}$ is the corresponding preconditioner, and $\VH^{(l)} = (\VK_{nm}^{(l)})^\top\VK_{nm}^{(l)} + \lambda n\VK_{mm}^{(l)}$, and set $\widetilde{\Valpha}^{(l)} = \VB^{(l)}\Vbeta^{(l)}$.

\begin{remark} In this paper, we construct the kernel matrices $\VK^{(l)}$ on a particular class of radial kernels, namely the Gaussian kernel
\begin{equation*}
    k^{(l)}(\xbs, \widetilde{\xbs}_i) = \exp{\bigg(-\frac{\|\xbs-\widetilde{\xbs}_i\|^2}{2r_l^2}\bigg)},
\end{equation*}
where $r_l>0$ is the shape parameter (the kernel bandwidth) at level $l$.
\label{remark:we_use_radial_basis_kernels}
\end{remark}

\section{The damped cover tree}
\label{sect:Adaptive sub-sampling}
This work introduces a data-driven sub-sampling method that we call the damped cover-tree (DCT). The DCT is a modification of the cover-tree (CT) \cite{Beygelzimer2006}, a data structure based on partitioning a metric space, initially designed to facilitate nearest neighbor search.
The goal of the DCT is to modify and simplify the CT to allow a viable sub-sampling scheme. 

Let $(\CX, \|\cdot\|)$ be a normed space where the input domain $\CX\subset \bbR^D$ is bounded, such that the diameter $r_0=\diam(\CX)$ is finite. The DCT is a tree structure where each node $p$ of the tree is associated with a point $\xbs_p\in \CX$, and which is built sequentially as data points arrive.
Furthermore, let $Q_l$ be a set (herein called a cover-set) containing all the nodes at a level $l\ge0$ in the given tree.
A level is associated with an integer $l$ and a radius $\rl=2^{-l}r_0$, where $l=0$ denotes the root level containing only one node and $l$ increases as we descend deeper into the tree.
DCT has three invariants, of which the first two are also invariants of the CT.

\begin{enumerate}[label=({I\arabic*})]
  \setcounter{enumi}{0}
 \item\label{inv:covering} (\textsf{Covering invariant}) For all $p \in Q_{l+1}$ there exists $q \in Q_l$ such that $\|\xbs_q-\xbs_p\| < \rl$.
 \item\label{inv:separation} (\textsf{Separation invariant}) For all $q, p \in Q_l$ where $\xbs_q\neq \xbs_p$, we have $\|\xbs_q-\xbs_p\| > \rl$.
\end{enumerate}
We add that the standard CT includes a third invariant, the so-called nesting invariant, which requires $Q_l\subseteq Q_{l+1}$, but this is not desired for our purpose.

To introduce the last invariant of the DCT, we first need the following definition.
\begin{definition}[The covering fraction] Let $p\in Q_l$ be a node, and $\xbs_p$ the associated point in $\CX$.
Furthermore, let $\childp=\{c_i\}_{i=1}^k$ be the children of $p$, and $\xbs_{c_i}$ the corresponding points in $\CX$. The covering fraction of a node $p$ is defined as
\begin{equation*}\label{eqn:cover_fraction}
    \cofp = \frac{\vol\bigg(\CB(\xbs_p,\rl) \cap \bigcup\limits_{c_i\in \childp} \CB(\xbs_{c_i}, \rlp)\bigg)}{\vol{\big(\CB(\xbs_{p}, \rl)\big)}}.
\end{equation*}
\label{def:cover_fraction}
\end{definition}

The covering fraction is the proportion of the volume of $\CB(\xbs_p, \rl)$ that is covered by balls around its children of half the radius.
This quantity is directly related to \ref{inv:separation}, which enforces the radius $\rl$ to reduce by a factor of $2$ for each new level, starting from an initial radius $r_0>0$. The
covering fraction allows us to capture the vulnerability of the standard CT to the curse of dimensionality.

For example, consider two regions $\CA_1,\CA_2\subseteq\CX$, for which the doubling dimension at radius $r_l$ is $\dd(\CA_1, r_l) > \dd(\CA_2, r_l)$. A node $p\in\CA_1$ at level $l$ will then need exponentially more children to be covered, than a node $q\in\CA_2$ at the same level $l$. 
This exacerbates the deeper we go into the tree. Therefore, the CT would have significantly more nodes from regions where the doubling dimension is large.

We recall now that sub-sampling is in kernel methods intended to reduce the computational complexity. 
For this purpose, it is desirable to keep the number of sub-samples from each level within a budget of reasonable size. 
On the other hand, a too low sub-sample density will lead to poor interpolation performance. 
Due to the exponential growth of the number of nodes with respect to the doubling dimension, it would be desirable to avoid wasting our budget on sub-samples from regions and radii with a large doubling dimension, 
as this would require dedicating an (exponentially) large number of points to achieve good interpolation, which is not feasible. 
Moreover, in high dimensional regions, we likely cannot learn anything more than a simple function, for which a lower sampling density would suffice.

To reduce the number of sub-samples from regions of large doubling dimensionality, we introduce the following damping invariant as the third invariant of the DCT.

\begin{enumerate}[label=({I\arabic*})]
  \setcounter{enumi}{2}
 \item\label{inv:damping}(\textsf{Damping invariant})
Let $\CD_{\cof}\in (0,1)$ be some threshold and let $\childp$ and $\cofp$ be as in Def. \ref{def:cover_fraction}. 
Then any node $q$ whose parent node $p$ does not satisfy $\cofp \geq \CD_{\cof}$ does not have children of its own.
\end{enumerate}

The damping invariant forces the tree to devote more resources to regions of lower doubling dimension by making it harder for nodes in regions with higher doubling dimensions to have children. 
In other words, the practical effect of the damping invariant is to stop the vertical growth of the DCT if the doubling dimension becomes large.
This is because the covering number grows exponentially with the dimensionality, ensuring $\cofp \geq \CD_{\cof}$ gets correspondingly harder to achieve.

\begin{remark}
In Section \ref{section:analysis_of_the_DCT}, we analyze the damping invariant in more detail and show how the damping suppresses vertical growth of the DCT more for regions of high doubling dimension than for regions of lower doubling dimensionality.
\end{remark}

\subsection{Construction of the DCT}
\label{subsection:DCT_construction}
We now discuss how the DCT is constructed and updated as the data streams in.
First, it is important to restate that we use the DCT to replace the Nystr\"om sampling, which was in \falkon used to reduce the complexity of the ridge regressor.
Consequently, not all of the streamed data (that is, not every training point) will be added to the tree, but only those whose inclusion into the tree would not violate the invariants \ref{inv:covering}-\ref{inv:damping}.
In other words, the tree consists of only those training points that help resolve the data space at the relevant resolution level.
Thus, each node $p$ in the DCT is associated with a unique training sample $\xbs_p$, but not every training sample will be represented by a node in the tree.
Note that this is different from the standard CT, which aims to organize all of the training data into a geometrical leveled data structure.

The construction of the DCT consists of a series of checks which examine whether adding a given data point to the DCT would, or would not, violate invariants \ref{inv:covering}- \ref{inv:damping}. 
When a new point $\xbs_q$ arrives from the data stream the goal is to identify  the deepest level $l$ for which there exists a node $p$ such that $\|\xbs_q-\xbs_p\|\leq \rl$. This corresponds to finding the nearest node in the tree that could serve as a parent. 

We achieve this in the following way. The first training point is identified as the root node to which we associate the radius $r_0$. 
For each new point, we proceed in a top-down manner, starting from the root node\footnote{We assume that all new points $\xbs_q$ are within a ball of radius $r_0$ around this node, which holds for a large enough $r_0$}. 
We then check whether $\xbs_q$ would violate the separation invariant at the next level. In other words, if there exists a node $p$ such that $\|\xbs_q-\xbs_p\|<\rl$. 
If such a node does not exist, then $\xbs_q$ is added to the set of children of the root node, and we update the covering fraction estimate for the root node.
Otherwise, if such a node does exist, we repeat the process, checking the separation invariant among the children of the corresponding node, and proceed further down the tree.

Assume we arrived to a node $p$ at level $l\ge 1$, and we have $\|\xbs_q-\xbs_p\|\le \rl$.
We then check if $p$ is allowed to have children, that is if the damping invariant is satisfied.
If it is not satisfied, the point $\xbs_q$ is dismissed (it is not added to the tree).
On the other hand, if $p$ is allowed to have children, we check whether the separation invariant holds, i.e., if there exists a child $c$ of the node $p$ such that $\|\xbs_q-\xbs_c\|<\rlp$.
If that were the case, the separation invariant would be violated, and the recursion is applied again by considering $c$ as the potential parent node.
However, if such a child does not exist, that is, if the separation invariant is not violated, then $\xbs_q$ is added to the set of children of the node $p$.
More details are given in Alg. \ref{alg:DampedCoverTree}. 

Some comments are needed to elucidate how are the steps described above applied in practice.
First, note that the covering fraction from Def. \ref{def:cover_fraction} cannot be calculated explicitly, since the volume terms require knowing the intrinsic dimensionality. 
Therefore, it is necessary to use an estimator instead. 
For this purpose, we interpret $\cofp$ as the probability that a sample $\xbs\sim\textrm{Uni}(\CB(\xbs_p, r))$ will be within $\CB_{c} \coloneqq \bigcup_{c_i\in \childp}\CB(\xbs_{c_i},r/2)$, where $\childp$ are the children of $p$. 
This probability can be estimated by considering the checks of the separation invariant \ref{inv:separation}, conducted on the last $N$ points that were inside $\CB(\xbs_p,r)$, as a series of independent random trials. 
We use the following running average as an estimator of the covering fraction
\begin{equation}
    (\cofp)_t = (1-\alpha)(\cofp)_{t-1} + \alpha \mathbbm{1}_{\CB_c}(\xbs_t),
    \label{eq:estimator_of_covering_fraction}
\end{equation}
where $\mathbbm{1}_{\CB_c}(\xbs_t)$ is the indicator function, and $\alpha>0$ is a weighting parameter. 
This approximates a weighted average of the outcome of the $N$ last draws (cf. Appendix \ref{appendixB}). 
Note that this reduces the memory requirements, since instead of storing $N$ trial outcomes for each node in the tree, as required had we used an average of the last $N$ trials, we store only a single value for each node in the tree.

Second, the separation invariant is in practice too strict since it results in too few points added to the tree, and thus a worse kernel estimator.
Moreover, checking the separation invariant adds to the computational complexity.
Therefore, we introduce the following relaxation. 
Assume we have a new point $\xbs_q$ and arrived at a node $p$ at level $l$.
We then first conduct a random Bernoulli trial, with the failure probability 
\begin{equation}
    q_\xbs = \frac{1}{1+\exp{\big[\-h\tan\big(\pi(\|\xbs_q- \xbs_p\|/\rl-\frac{1}{2})\big) \big]}},
\end{equation}
where $h$ is the hardness of the threshold. 
In other words, the probability of failure is proportional to the distance between $\xbs_q$ and $\xbs_p$ - the larger the distance, the more likely the failure.
If the trial's outcome is a failure, then the check for the separation invariant is ignored, and the algorithm continues.
If it is a success, we proceed by first checking the separation invariant.
This means that the probability to ignore the separation invariant increases as $\xbs_q$ gets farther from $\xbs_p$.

\subsection{Sub-sampling from the DCT}
\label{subsection:subsampling_from_DCT}
We now discuss how the DCT is used for sub-sampling the training points. 
By organizing the training points into cover-sets $Q_l$ the DCT allows a hierarchical sub-sampling.
Even though cover-sets $Q_l$ significantly reduce the number of training points, they are for practical purposes still too large for efficient sub-sampling. Due to this, we restrict ourselves to a subset $\widetilde\Gamma^{(l)}\subseteq Q_l$ of candidate sub-samples called landmarks.

\begin{definition}[Landmarks]
Let $Q_l$ be the cover-set at level $l$ in a DCT. 
We define the set of candidate landmarks at level $l$ as $\widetilde\Gamma^{(l)}=\{\xbs_p \mid p\in Q_l \text{ and }\cofp \geq \mathcal{D}_{\cof} \}$, and the set of landmarks (of size $m$) as any subset $\widetilde\Gamma^{(l)}_m = \{\widetilde{\xbs}^{(l)}_1,\ldots, \widetilde{\xbs}^{(l)}_m\}\subset \widetilde\Gamma^{(l)}$ of size $m$.
\label{def:landmarks}
\end{definition}

Some remarks are in order.
First, by Def. \ref{def:landmarks}, candidates for landmarks at level $l$ are only those nodes allowed to have children (according to the damping invariant \ref{inv:damping}).
This design choice implies that the set of candidate landmarks will contain more points from regions with a lower doubling dimension than points from regions with a higher doubling dimension. 
This is because the larger the doubling dimension is, the more children nodes are needed to cover a given parent node.

Second, Def. \ref{def:landmarks} suggests using only a subset of candidate landmarks as sub-samples. 
We refer to a result from \cite{rudi2017falkon} which states that good statistical accuracy of the estimator is achieved if the number of sub-samples is proportional to the square root of the number of samples.
At level $l$ we therefore use a set of landmarks which is of size  $m^{(l)} = \delta_0 \sqrt{\textSN{Q_l}}$, where $\delta_0>0$ is a constant.

The third point that requires attention concerns the question of when the landmarks should be selected.
To that end, we use the covering fraction of a level, which,  with a slight abuse of notation, we denote as $\coflvl$. Moreover, we compute $\coflvl$ as
\begin{equation}
    (\coflvl)_t = (1-\alpha)(\coflvl)_{t-1} + \alpha \mathbbm{1}_{\CB_\text{level}}(\xbs_t),
    \label{eq:estimator_of_covering_fraction_for_level}
\end{equation}
where $\mathcal{B}_{level} = \bigcup\limits_{p\in Q_l} \mathcal{B}(\xbs_{p}^{(l)}, \rl)$.
Moreover, analogously to the damping invariant, let $\mathcal{D}_{level}\in(0,1)$ be some threshold.
We then say that a level $l$ is sufficiently covered when $\coflvl \geq \mathcal{D}_{level}$.

\begin{remark} We note that as the level increases, our estimate of $\coflvl$ through Eq. \eqref{eq:estimator_of_covering_fraction_for_level} will be increasingly more sensitive to subsets $\CA\subset\CX$ of low doubling dimension than to subsets of large doubling dimension. 
This is because the damping invariant \ref{inv:damping} makes it harder for nodes in high dimensions to have children.
Consequently, we will have fewer points in deeper levels that belong to high dimensional regions. Because of this, the estimator in Eq. \eqref{eq:estimator_of_covering_fraction_for_level} is biased towards using more sub-samples from lower dimensional regions.
\end{remark}

Sub-sampling from a level $l$ goes as follows. 
As training points arrive, we build the tree and continuously update the covering fraction of a level. 
Once that level is sufficiently covered, that is, once $\coflvl \geq \mathcal{D}_{level}$, we extract the set of landmarks by sub-sampling $m^{(l)}$ points from the pool of candidate landmarks $\widetilde\Gamma^{(l)}$.

\section{StreaMRAK}
\label{section:TheStreaMRAKalgo}
In this section, we present \method and clarify how it synthesizes concepts from Sections \ref{section:kernel_methods} and \ref{sect:Adaptive sub-sampling}, and utilizes them in a streaming context. 
The workflow of \method can be divided into three threads that can run in parallel,  subject to some inter-dependencies. 
These are the sub-sampling thread, the training thread, and the prediction thread. Overviews of these threads are given next, and the reader is referred to Algorithm \ref{alg:PseudoCode_StreaMRAK} in the Appendix for further details.

\subsection{Sub-sampling thread}
In the sub-sampling thread \method collects and organizes the training data into a DCT. 
Namely, as new training pairs are collected, the covering \ref{inv:covering} and separation \ref{inv:separation} are checked, and the covering fraction is updated as described in Section \ref{subsection:DCT_construction}.
Moreover, the set of landmarks for each level is updated, as described in Section \ref{subsection:subsampling_from_DCT}. Once the set of landmarks for a given level $\widetilde\Gamma^{(l)}_m$ is completed, the landmarks and the estimator for the corresponding level can be used in the remaining two threads.

\subsection{Training thread} 
The model is trained at level $l$ when two conditions are met. First, coefficients of the previous level $l-1$ in the LP must have been calculated, i.e. previous training thread must finish. 
Second, landmarks $\widetilde\Gamma^{(l)}_{m^{(l)}}$ at level $l$ must be ready. 

In the first step, we define the kernel matrix on the landmarks by
\begin{equation}\label{eq:kernelmatrix_on_landmarks}
[\VK_{mm}^{(l)}]_{ij} = k^{(l)}(\widetilde{\xbs}_i,\widetilde{\xbs}_j),\text{ for }\widetilde{\xbs}_i \in \widetilde\Gamma^{(l)}.
\end{equation}
In the second step we consider $\big(\VK^{(l)}_{nm}\big)^\top\VK^{(l)}_{nm} \in \mathbb{R}^{m^{(l)}\times m^{(l)}}$ and $\big(\VK^{(l)}_{nm}\big)^\top\dbs^{(l)}_n \in \mathbb{R}^{m^{(l)}}$
which in addition to landmarks depend on the training points.
They are updated continuously as new training points come in, according to Eq. \eqref{eq:updateFormula_KnmTKnm} and Eq. \eqref{eq:updateFromualte_Zm}. 
However, they are not updated indefinitely, but only until new training points do not significantly alter the matrices according to the following criterion.

\begin{definition}(Sufficient training points) Let $\VA_n:=(\VK^{(l)}_{nm})^\top \VK^{(l)}_{nm}$, and $\bbs_n:=\big(\VK^{(l)}_{nm}\big)^\top\dbs^{(l)}_n$. Let $\delta_1,\delta_2,\delta_3>0$ be three constants. We consider the number of training points at a level $l$ sufficient when either $n\geq \delta_3$ or
\begin{equation*}
    \bigg\|\frac{\VA_{n}}{n} -\frac{ \VA_{n+1}}{n+1}\bigg\|_{\infty} \leq \delta_1 \text{ and } \quad     \bigg\|\frac{\bbs_{n}}{n} - \frac{\bbs_{n+1}}{n+1}\bigg\| \leq \delta_2.
\end{equation*}
\label{def:sufficien_training_points}
\end{definition}

After enough training samples are collected according to Def. \ref{def:sufficien_training_points}, the correction term $s^{(l)}$ is obtained by solving for the coefficients $\widetilde\alpha^{(l)}_{1},\dots,\widetilde\alpha^{(l)}_{m^{(l)}}$ using Eq. \eqref{eq:FALKONlinearSystem_in_LP}. 
The new prediction model $\widehat{f}^{(L)}$ is obtained by adding $s^{(l)}$ to the previous model, according to Eq. \eqref{eq:laplacian_pyramid_model}. 

\subsection{Prediction thread}
In this thread \method makes provides the latest version of the trained LP model in Eq. \eqref{eq:laplacian_pyramid_model}. This means that if $L$ is currently the highest level that has been trained, the prediction for new points $\xbs$ is made using the model $\widehat{f}^{(L)}(\xbs)$.

\section{Analysis}
\label{section:analysis}
In this section, we first analyze the damping invariant of the DCT. We then offer theoretical results on the convergence properties of the LP in the context of KRR. Finally, we offer estimates of the time and memory requirements of \method.

\subsection{Analysis of the DCT} 
\label{section:analysis_of_the_DCT}
As discussed in Section \ref{sect:Adaptive sub-sampling}, the DCT adds a given training point to the set of nodes of the tree if conditions \ref{inv:separation} and \ref{inv:damping} are satisfied, and the points are otherwise discarded.
In particular, the damping invariant \ref{inv:damping} makes it harder for a node to have children. 
The guiding idea is that damping should reduce the impact of the curse of dimensionality by making it harder for nodes in regions of higher doubling dimension to have children, and in doing so it should effectively stop the vertical growth of the tree in corresponding regions. 
Therefore, it is critical to understand how and to what degree the damping affects high dimensional regions more than low dimensional ones. 

In a statistical sense, the damping should treat all nodes in regions of the same doubling dimension equally. 
Therefore, to gain insight into the damping, it suffices to analyze its effects concerning the doubling dimension on a single node $p$. 
In this case, the effect of damping can be measured by analyzing how many training points must pass through $p$, in the sense of Alg. \ref{alg:DampedCoverTree}, before children of $p$ are allowed to have children of their own. 
This can be modeled by considering the expected number of training points $\xbs_i\sim{\rm Uni}\big( \CB(\xbs_p,r)\big)$ necessary to cover $\CB(\xbs_p, r)$  with balls of radius $r/2$ around points $\xbs_i$. 

Consider $\xbs_i\sim {\rm Uni}\big( \CB(\xbs_p,r)\big)$, and let a set $\CS_p$ be built in a succession of trials $i=1,\ldots,N_{t}$ so that 
\[\xbs_i\in\CS_p \text{ if } \|\xbs_i-\xbs\|\geq \frac{r}{2} \text{ for all } \xbs \in\CS_p.\]
In other words, a newly sampled point $\xbs_i$ will only be added to the set $\CS_p$ if it its pairwise distances from all the points that are already in $\CS_p$ are at least $r/2$.

\begin{problem}\label{problem:Problem_setup_CT_filling_node}
Let $\widetilde C_p$ denote the set of children of the node $p$, constructed from the above-described trials. 
What is the expected number of trials $N_{t}$ needed to ensure $\cofp=1$?
\end{problem}

Since there is no unique set $\CS_p$ such that the corresponding set of children $\widetilde C_p$ ensures $\cofp=1$, the sample space for Problem \ref{problem:Problem_setup_CT_filling_node} corresponds to all admissible sets $\CS_p$, which vary in both the number and the location of points they contain. 
Characterizing all such sets corresponds to a disordered sphere packing problem \cite{jeffrey2012statistical}, which is an NP-hard combinatorial problem \cite{Hifi2009}. 
For a theoretical analysis of this problem, defining a probability measure over the sample space is necessary. 
However, in this level of generality, neither the sample space nor the probability measure admit a workable definition, with currently available mathematical tools \cite{jeffrey2012statistical}. 
Although some theoretical insights are possible under simplifications on the sample space, this analysis is restrained to a limited number of spheres and configurations. 

Due to these difficulties, we consider a simplified setting where we instead consider an average case. 
If the set $\CS_p$ is such that $\CB(\xbs_i,r)\subset\bigcup_{\xbs_i\in\CS_p}\CB(\xbs_i,r/2)$, which corresponds to $\cof(p)=1$,
then each of the balls $\CB(\xbs_i,r/2)$ occupies on average $\frac{1}{|\CS_p|}$ of the total volume of $\CB(\xbs_p,r)$, assuming none of the balls are covered by a union of other balls. 
Therefore, as $\CS_p$ is being built, adding a point to $\CS_p$ will, on average, reduce the unoccupied volume of $\CB(\xbs_p,r)$ by $\frac{1}{|\CS_p|}$. 
Moreover, it can be shown that the number of elements in such a set satisfies $2^{\dd-1}\leq |\CS_p| \leq 5^\dd$, see Lemma \ref{lemma:Bound_on_c_D}, where $\dd\coloneqq \dd(\CS_p, r)$ is the doubling dimension of $\CS_p$.
Based on these considerations we introduce a simplified setting for the average case of Problem \ref{problem:Problem_setup_CT_filling_node}.

\begin{assumption}
Problem \ref{problem:Problem_setup_CT_filling_node} can be approximated by dividing the ball $\CB(\xbs_p, r)$ into a union of $c_d$ fixed (and known) disjoint bins $\CB_i$ of size $(1/c_d)\vol\big(\CB(\xbs_p, r)\big)$.
\label{assumption:simplified_problem_setting}
\end{assumption}

Note that the bins referred to in Assumption \ref{assumption:simplified_problem_setting} correspond to regions around the children of the node $p$. Assumption \ref{assumption:simplified_problem_setting} reduces the average case of Problem \ref{problem:Problem_setup_CT_filling_node} to a form of the classical coupons collector's problem \cite{flajolet1992birthday}, which considers $n$ coupons with the same probability of being drawn. 
Through a series of randomized trials with replacement, the goal is to obtain a copy of each coupon.
Relevant for Problem \ref{problem:Problem_setup_CT_filling_node} is estimating the stopping time $T$, which counts the number of trials before all coupons are collected, and which satisfies $\bbE[T] = n H_n$, where $n$ denotes the number of coupons and $H_n$ is the $n$-th harmonic number \cite{flajolet1992birthday}. 

In terms of Problem \ref{problem:Problem_setup_CT_filling_node}, and under Assumption \ref{assumption:simplified_problem_setting}, we can therefore identify $T=N_{t}$, $n=|\CS_p|$ and $\bbE[N_{t}|\text{Node}\, p] = |\CS_p| H_{|\CS_p|}$. Combining the bound $\ln(n) + \frac{1}{2} \leq H_n \leq \ln(n) + 1$ (from \cite{klambauer1979problems}), with the bound on $|\CS_p|$ from Lemma \ref{lemma:Bound_on_c_D} we have
\begin{equation}
  2^{\dd-1}((\dd-1)\ln 2+1/2) \leq \bbE[N_{t}|\text{Node}\, p] \leq 5^\dd(\dd\ln 5+1).
  \label{eq:bound_T_p}
\end{equation}
With the same strategy, we can bound the number of trials until the cover-fraction of a level reaches $1$, as
\begin{equation}
  2^{l(\dd-1)}({l(\dd-1)}\ln 2+1/2) \leq \bbE[N_{t}|\text{Level}\,l] \leq 5^{l\dd}(l\dd\ln 5+1).
  \label{eq:bound_T_l}
\end{equation}

From Eq. \eqref{eq:bound_T_p} we see that the number of training points $\bbE[N_{t}|\text{node}\, p]$ grows exponentially with the doubling dimensionallity $d$.
In other words, significantly more trials are needed to achieve $\cofp = \CD_{\cof}$ for nodes in regions with a large doubling dimension than it is for nodes in regions with a lower doubling dimension.
Consequently, through the damping invariant, the DCT restricts the vertical growth of the tree comparatively more the higher the doubling dimension of the local region.

\subsection{Time and memory requirements}
\label{subsect:memory_usage_analysis}
This section analyzes the memory requirements of \method, which involve storing the DCT and the linear system components used to update the coefficients. Furthermore, we consider the computational requirements, which consist in solving the coefficient equations.
Both the memory and computational requirements need to be analyzed per level $l$ of the tree due to the multi-resolution nature of the estimator and the tree organization of the data.

For the analysis, we consider a simplified setting where we assume that the doubling dimension is constant for all levels and all subsets of $\CX$, and that the number of children $c_d$ is the same for all nodes. At the end of the section we describe a more general setting.

In the following, we assume that the growth of the DCT stops at a level $L$.
In other words, level $L$ is the last level at which there are nodes. 
In practice, the growth of the DCT slows down exponentially fast with the product of the doubling dimension $\dd\coloneqq\dd(\CX, r_L)$ and the level $l$. 
This can be seen from Eq. \eqref{eq:bound_T_l}, which shows that the number of training points necessary to fill up a level grows exponentially with $l\dd$. 
Therefore, in practice, no new levels will be added to the DCT when $l\dd$ is large enough, which effectively makes the last level $L$ independent of the number of training points. 
Furthermore, from Lemma \ref{lemma:Bound_on_c_D} we know that $c_d$ is bounded by $2^{\dd-1} \leq c_d\leq 5^\dd$, which shows that also $c_d$ is independent of the number of training points. 

\begin{proposition}
The memory requirement of \method is
$\mathcal{O}\big(\sum_{l=0}^{L} c_d^{l}\big)$.
\label{prop:memory_streamrak_and_DCT}
\end{proposition}

\begin{proof}
The memory requirement of the DCT is determined by the number of nodes in the tree. Given that the number of children is the same for all nodes.
If the number of children per node is $c_{d}$, then the total number of nodes at level $l$ is $c_d^l$.
Thus, the memory needed to store the DCT with $L$ levels is $\CO(\sum_{l=0}^{L} c_d^l)$.

To store the linear system on level $l$ we need the matrices $\big(\VK_{nm^{(l)}}\big)^\top\VK_{nm^{(l)}},\,\VK_{m^{(l)}m^{(l)}}\in\bbR^{m^{(l)}\times m^{(l)}}$ and the vector $\zbs\in\bbR^{m^{(l)}}$. 
The number of landmarks $m^{(l)}$ at level $l$ is chosen as $m^{(l)}=\delta_0\sqrt{|Q_l|}$, where $|Q_l|$ is the number of nodes at level $l$. Since $|Q_l|$ is $\CO(c_d^l)$, it follows that $m^{(l)}\times m^{(l)}$ is also $\CO(c_d^l)$ per level, and the desired conclusion follows.
\end{proof}

Note that with a fixed $L$ and $n$ larger than $\mathcal{O}\big(\sum_{l=0}^{L} c_d^{l}\big)$, then the memory requirement is independent of $n$.
We also note that if the deepest level satisfies $L\rightarrow\infty$, then the number of nodes is determined by the number of training points, and the memory requirement would thus, in the worst case, become $\mathcal{O}(n)$, the same as for the standard cover-tree.

Next, we discuss the construction of the DCT, where adding a new point to the set of nodes requires a search through the tree.

\begin{proposition}
\label{prop:Time_req_DCT_insertion}
Inserting a new point into the DCT, cf. Algorithm \ref{alg:DampedCoverTree}, requires $\CO(c_d L)$ operations.
\end{proposition}
\begin{proof}
For a point $\xbs_q\in\CX$ to be analyzed at level $L$, we need to have analyzed it at the previous $l<L$ levels.
At each level, we must, in the worst case, check the separation invariant with all children of the current potential parent $p^{(l)}$, before finding a node $c$ such that $\|\xbs_q-\xbs_c\|\leq 2^{-l}r_0$, that would serve as the next potential parent. 
This requires at most $c_d$ operations per level, giving $L c_d$ total operations over the $L$ levels. 
The same number of operations is necessary if a node is discarded at level $L$.
\end{proof}

Lastly, we analyze the computational requirements for solving the linear system.
\begin{proposition}
\label{prop:Time_req_solver}
The time requirement for solving the linear system in Eq. \eqref{eq:Falkon_system} is $\mathcal{O}\big(\delta_3 m^{(l)}+\big(m^{(l)}\big)^3\big)$ per level, where $\delta_3$ is given in Def. \ref{def:sufficien_training_points},
\end{proposition} 
\begin{proof}
The time requirement of \falkon is $\CO(nmt+m^3)$ where $n$ is the number of training points, $m$ the number of landmarks and $t$ the number of iterations of the conjugate gradient (which has an upper bound). 
By Def. \ref{def:sufficien_training_points}, \method uses at most $\delta_3$ training samples at each level. 
Since $m^{(l)}$ is the number of landmarks at level $l$, the result follows.
\end{proof}

Assume that the domain $\CX$ can be divided into disjoint subsets $\CA_1,\dots,\CA_t\subset\CX$ for which the doubling dimension $\dd(\CA_i, r_{l})$ differs based on $\CA_i$ and radius $r_{l}$. Let the number of children of a node $\xbs_p\in \CA_i$ at level $l$ be $c_{d,i,l}$. In this scenario, the growth of the DCT will stop at different levels $L_i$ for different subsets $\CA_i$. The final time and memory requirements would therefore be the sum of the contribution from each subset $\CA_i$. In other words, the memory would be $\mathcal{O}(\sum_{i=1}^t\sum_{l=0}^{L_i}c_{d,i,l}^{l})$, and similarly the time requirement per point insertion would be $\mathcal{O}(\sum_{i=1}^t\sum_{l=0}^{L_i} c_{d,i,l})$. We note that $c_{d,i,l}$ and $L_i$ depend on the dimensionality of the data, but are independent of $n$. Therefore, so are the time and memory requirements.

\subsection{Convergence of the LP formulation of the KRR}
This section analyzes the conditions for which the LP approximates the training data  $y_i=f(\xbs_i)$, with respect to the number of levels.
A similar analysis was previously done for the LP in the context of kernel smoothers \cite{Leeb2019}. However, to the best of our knowledge, this is the first time the LP formulation of KRR has been analyzed in this way.

Consider the LP estimator $\widehat{f}^{(l)}$ as defined in Eq. \eqref{eq:laplacian_pyramid_model}, but without sub-sampling.
From the recurrence relationship for the residuals $\dbs^{(l)}$ in Eq. \eqref{eq:residual_at_level} by induction it follows 
\begin{equation}
    \widehat{f}^{(l+1)}([\xbs_n])-f([\xbs_n])=(\VI-\VP_{nn}^{(l)})( \widehat{f}^{(l)}([\xbs_n])-f([\xbs_n]),
\end{equation}
where $\VP_{nn}^{(l)}\coloneqq\VK^{(l)}_{nn}(\VK^{(l)}_{nn}+\lambda n \VI)^{-1}$, cf. Lemma \ref{lemma:residual_expression}. 

\begin{theorem}
Let $\widehat{f}^{(l)}$ be the LP estimator defined in Eq. \eqref{eq:laplacian_pyramid_model} and let $\lambda$ be a regularization parameter. Furthermore, let $0<\sigma_{l,n} \leq \dots \leq \sigma_{l,1}$ be the eigenvalues of $\VK^{(l)}_{nn}$. For $L>0$ we then have
\begin{equation*}
    \|\widehat{f}^{(L+1)}([\xbs_n])-f([\xbs_n])\| \leq \prod_{l=0}^L(1-\varepsilon(l)) \|\widehat{f}^{(0)}([\xbs_n])-f([\xbs_n])\|, \quad \text{where} \quad \varepsilon(l)=\frac{\sigma_{l,n}}{n\lambda+\sigma_{l,n}}.
\end{equation*}

\label{thm:LP_KRR_convergence}
\end{theorem}

From Thm. \ref{thm:LP_KRR_convergence} it follows that the LP estimator will converge as $l\rightarrow \infty$, since $\sigma_{l,n}>0$ and therefore $1-\varepsilon(l) \in (0, 1)$ for all $l$. In Thm. \ref{thm:LP_KRR_convergence_rate} we characterise how $\varepsilon(l)$ depends on the level $l$ to give insight on the nature of this convergence.

\begin{theorem}
The LP estimator $\widehat{f}^{(l)}$ from Eq. \eqref{eq:laplacian_pyramid_model} converges with increasing level $L$ to the training data $f(\xbs_i)$, cf. Thm. \ref{thm:LP_KRR_convergence}, with the rate $\prod_{l=0}^L(1-\varepsilon(l))$,
where 
\begin{equation}
    1-\varepsilon(l) \leq \big(1+C_{1,D} 2^{-Dl}\exp\big(-C_{2,D} 4^{-l}\big)/n\lambda\big)^{-1},
    \label{eq:first_bound_LP}
\end{equation}
for
\begin{equation*}
    C_{1,D} = \frac{1}{2}(6\sqrt{2})^D\Gamma(D/2+1)^{\frac{D-1}{D+1}}\bigg(\frac{\pi}{9}\bigg)^{\frac{D}{D+1}}\bigg(\frac{r_0}{\delta}\bigg)^D \quad \text{and} \quad C_{2,D} = 1152\bigg(\frac{\pi\Gamma^2(D/2+1)}{9}\bigg)^{\frac{2}{D+1}}\bigg(\frac{r_0}{\delta}\bigg)^2,
\end{equation*}
where $\Gamma$ is the gamma function.

Furthermore, for $l > \log_2(\sqrt{D/2}(r_0/\delta))$ we have the tighter bound
\begin{equation}
    1-\varepsilon(l) < \bigg(1+\big(1-2^{1+\frac{1}{\ln{2}}(C_3 D - g(l))}\big)/n\lambda\bigg)^{-1},
    \label{eq:sec_bound_LP}
\end{equation}
where $g(l)=4^{l-\log_2{r_0/\delta}}$ and $C_3=(\ln{(1+1/4)}+2\ln{2})$.
\label{thm:LP_KRR_convergence_rate}
\end{theorem}

We note that the bound in Eq. \eqref{eq:first_bound_LP} underestimates the rate of convergence for lower levels but improves as the levels increase. 
Furthermore, Thm. \ref{thm:LP_KRR_convergence_rate} shows that the convergence rate increases with the level $l$. 
In fact, the bound in Eq. \eqref{eq:first_bound_LP} can be simplified with an \textit{a fortiori} bound of the same form, where $C_{1,D}=\frac{1}{2}\big(\frac{12.76}{2^{3/2}}\big)^D\big(\frac{D^D}{\Gamma(D/2+1)}\big)\big(\frac{r_0}{\delta}\big)^D$ and $C_{2,D}=(12.76\sqrt{2}D)^2(r_0/\delta)^2$, which ensures that $1-\varepsilon(l)$ decreases monotonically for $l<\log_2(\sqrt{D/2}(r_0/\delta))+\log_2(25.52\sqrt{2})$. see Remark \ref{remark:fortiori_bound} and Corollary \ref{corollary:Monotonicaly_increasing}.

On the other hand, when $l > \log_2(\sqrt{D/2}(r_0/\delta))$ the tighter bound from Eq. \eqref{eq:sec_bound_LP} ensures that $1-\varepsilon(l)$ continues to decreases monotonically. 
Moreover, as $l\rightarrow\infty$ each new level reduces the residual error by $(1+1/n\lambda)^{-1}$. 
We can also observe that the convergence rate is reduced by the number of training points $n$, but this effect can be mitigated by reducing the regularization parameter $\lambda$. 
We also note that Thm. \ref{thm:LP_KRR_convergence} and Thm. \ref{thm:LP_KRR_convergence_rate} are derived for a vector of numbers on the training data $\Gamma_n\subset \CX$, without assumptions on the target function. In other words, the LP estimator can approximate the training data for any function $f:\Gamma_n\rightarrow\bbR$, to arbitrary precision, by including sufficiently many levels.

\begin{corollary} If the residual $\dbs^{(l)}=(\widehat{f}^{(l)}([\xbs_n])-f([\xbs_n]))$ at level $l$ only projects non-trivially onto the eigenvectors with eigenvalue $\sigma_{l,n} \geq \sigma_{\text{cut-off}}$, then we say the residual is spectrally band-limited with respect to the kernel. If the residual $\dbs^{(l)}$ is spectrally band-limited, then $1-\epsilon(l)< n\lambda/(n\lambda + \sigma_{\text{cut-off}})$.
\label{cor:spectrally_bandlim_res}
\end{corollary}

\section{Experiments}
\label{section:ExperimentsMain}
This section presents comparative numerical experiments of the proposed estimator on three problems.
In Section \ref{section:varsinus_experiment} we consider a one-dimensional regression problem, and in Section \ref{section:dumbell_experiment} we consider a dumbbell-shaped domain that consists of two 5-dimensional spheres connected by a 2-dimensional plane.
Lastly, in Section \ref{section:doublePend}, we forecast the trajectory of a double pendulum, which is a well-known chaotic system \cite{shinbrot1992chaos}. 

We compare \method with \falkon \cite{rudi2017falkon} and an LP modification of KRR (\lapkrrE). Both \falkon and \lapkrr rely on the standard Nystr\"{o}m sub-sampling \cite{williams2001using, smola2000sparse}. Furthermore, \falkon does not rely on a multi-resolution scheme but uses instead a single bandwidth, found by cross-validation.

Throughout the experiments, we set the threshold for the number of sub-samples (landmarks) in \method to be $10\sqrt{\textSN{Q_l}}$, where $Q_l$ is the set of nodes at level $l$ in the DCT. We note that to choose the sub-sample size, \falkon and \lapkrr require $n$ to be known beforehand. For \falkon we let the number of Nystr\"{o}m landmarks be $10\sqrt{n}$, where $n$ is the number of training samples. Meanwhile, for \lapkrr{} we sub-sample $\sqrt{n}$ Nystr\"{o}m landmarks, which are then used for all levels. 

We also need to pre-select the number of training points for \lapkrr and \falkonE. For \falkon{} we use the entire training set, as in \cite{rudi2017falkon}. 
Similarly, it is also common for the LP to use the entire training set at each level \cite{Rabin2012, Leeb2019}.
However, for large data sets, it might be better to include fewer data points. Therefore, we also use a version of the \lapkrr where we divide the total training data equally between the levels.

Throughout the experiments, we measure the performance of \methodE, \falkonE, and \lapkrr by estimating the mean square error
\begin{equation}
    MSE(y, y_\textit{pred}) = \frac{1}{\Upsilon\Lambda}\sum_{k=1}^\Upsilon \frac{1}{n_k}||\ybs_{k}-\ybs_{k}^\textit{pred}||^2, \text{ with } 
    \Lambda= \max_{\substack{k\in[\Upsilon]\\ i\in[n_k]}}[\ybs_k]_i - \min_{\substack{k\in[\Upsilon]\\ i\in[n_k]}}[\ybs_k]_i,
    \label{eq:MSEequationExperiments}
\end{equation}
where $\Upsilon$ is the number of test runs we average over, $n_k$ is the number of test points at test run $k$, and $\ybs_{k}, \ybs_{k}^\textit{pred}\in \bbR^{n_t}$ are the target values and predictions respectively, and $\Lambda$ is the normalisation factor.

\subsection{Multi-resolution benchmark}
\label{section:varsinus_experiment}
We consider the function,
\begin{equation}\label{eq:varsinus_target}
    f(x) = \sin{\bigg(\frac{1}{x+0.01}\bigg)}
    , \text{ for } x\in\left[0,\frac{\pi}{4}\right].
\end{equation}
In the experiment  we use a training set of $n=2.2\times 10^{6}$ samples and a test set of $1.3 \times 10^{5}$ samples. We use the non-uniform gamma distribution $\Gamma(\alpha, \beta)$ with $\alpha=1,\,\beta = 2$ to sample the training data.

The number of training points used at each level in \method is determined by setting $\delta_1$ and $\delta_2$ from Def. \ref{def:sufficien_training_points} to $10^{-3}$. With this choice, \method selects between \SI{30244}{} and \SI{40100}{} training points for each level. For comparison, \falkon uses all the $\SI{2.2e6}{}$ training points. Furthermore, for \lapkrr we run two experiments: \lapkrr (1) using $\SI{1.1e5}{}$ training points at each level and \lapkrr (2) using $\SI{2.2e6}{}$ training points at each level.

\begin{table}[htb!]
\caption{Comparison of \methodE, \lapkrr and \falkon for the target in Eq. \eqref{eq:varsinus_target}. For each level $l$ we show the number of landmarks, the mean square error (MSE), and the accumulated time to train the prediction model (Time). 
In parenthesis, in the time column of the \falkon row, is the time to find the optimal bandwidth through cross-validation.}
\tabcap
\begin{center}
\begin{tabular}{c|c|c|c|c}
    & Level &  \# Landmarks & MSE & Time \\ \hline
    \multirow{4}{*}[+0.5mm]{\method} & 5 & \SI{47}{} & \SI{2.55e-1}{} & \SI{77}{\second} \\ 
    & 10 & \SI{392}{} & \SI{3.69e-2}{} & \SI{116}{\second} \\ 
    & 15 & \SI{1525}{} & \SI{8.63e-6}{} & \SI{497}{\second} \\ 
    & 16 & \SI{2302}{} & \SI{6.18e-6}{} & \SI{1194}{\second} \\ \hline
    \multirow{4}{*}[-2mm]{\begin{tabular}{@{}c@{}}\lapkrr (1) \\ $n_l =\SI{1.1e5}{}$ \end{tabular}} & 5 & \SI{1483}{} & \SI{2.56e-1}{} & \SI{143}{\second} \\ 
    & 10 & \SI{1483}{} & \SI{3.65e-2}{} & \SI{413}{\second} \\ 
    & 15 & \SI{1483}{} & \SI{8.72e-6}{} & \SI{825}{\second} \\ 
    & 16 & \SI{1483}{} & \SI{6.85e-6}{} & \SI{922}{\second} \\ 
    & 18 & \SI{1483}{} & \SI{6.55e-6}{} & \SI{1136}{\second} \\ \hline
    \multirow{4}{*}[+1.5mm]{\begin{tabular}{@{}c@{}}\lapkrr (2) \\ $n_l =\SI{2.2e6}{}$ \end{tabular}} & 5 & \SI{1483}{} & \SI{2.56e-1}{} & \SI{2963}{\second} \\ 
    & 10 & \SI{1483}{} & \SI{3.64e-2}{} & \SI{8704}{\second} \\ 
    & 18 & \SI{1483}{} & \SI{8.91e-6}{} & \SI{23113}{\second} \\ \hline
    \methodE & -- & \SI{14830}{} & \SI{5.7e-3}{} & \SI{4642}{\second}+(\SI{27930}{\second})
\end{tabular}
\end{center}
\label{table:compare_falkon_mrfalkon_streamrak_on_varsinus}
\end{table}

\begin{figure}[htb!]
    \centering
    \subfloat[\label{subfig:varsinus_streamrak_pred_lvl3}\centering Level 3]{{\includegraphics[width=0.26\textwidth]{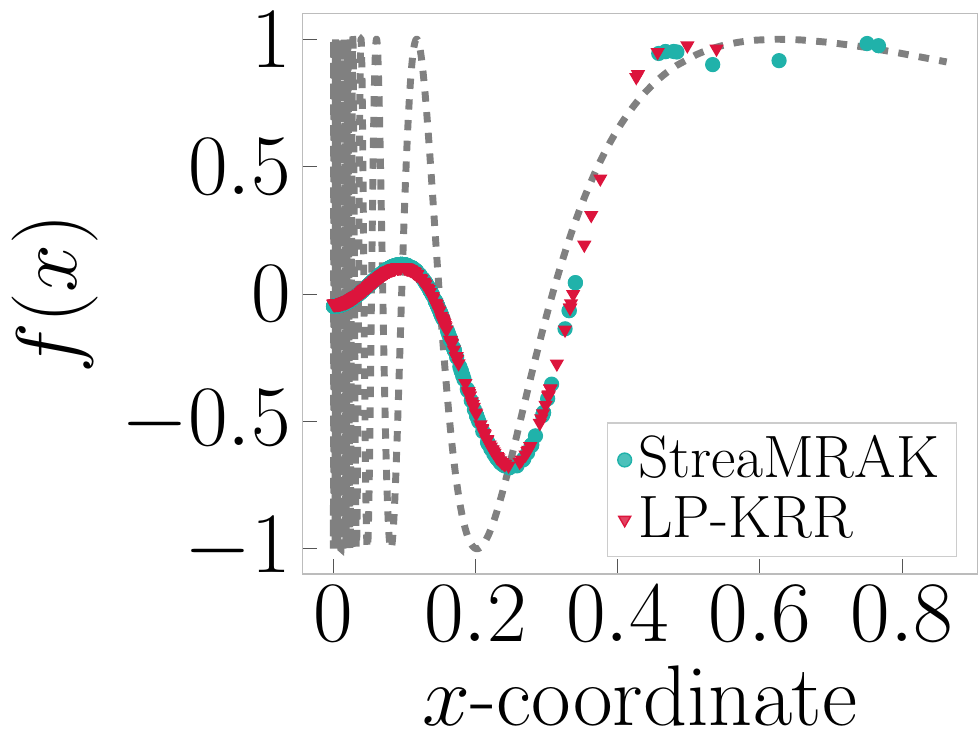} }}%
    \subfloat[\label{subfig:varsinus_streamrak_pred_lvl5}\centering Level 5]{{\includegraphics[width=0.23\textwidth]{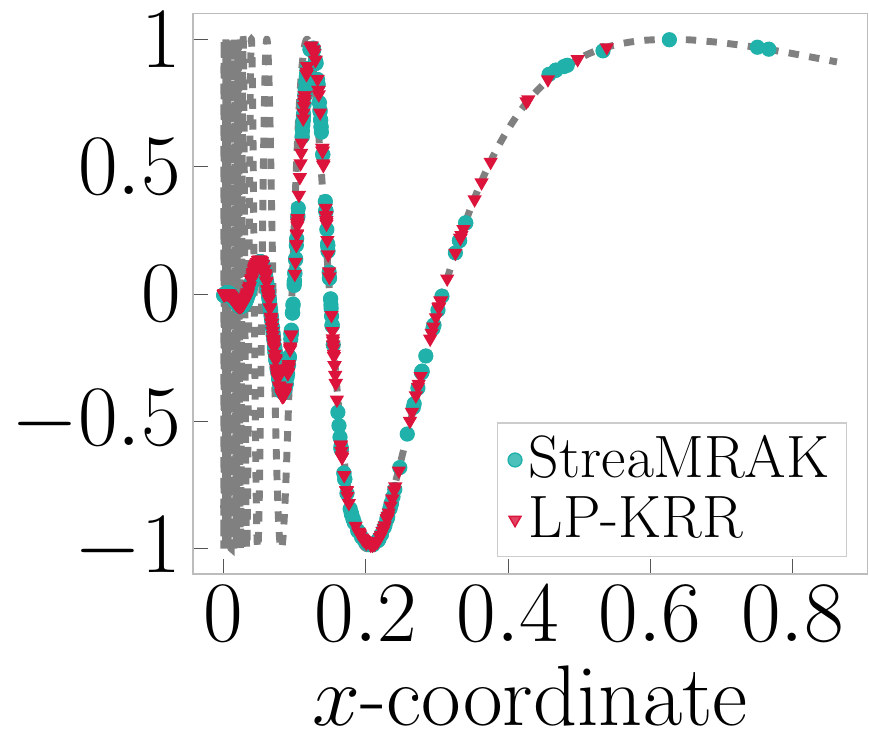} }}%
    \subfloat[\label{subfig:varsinus_streamrak_pred_lvl16}\centering Level 16]{{\includegraphics[width=0.23\textwidth]{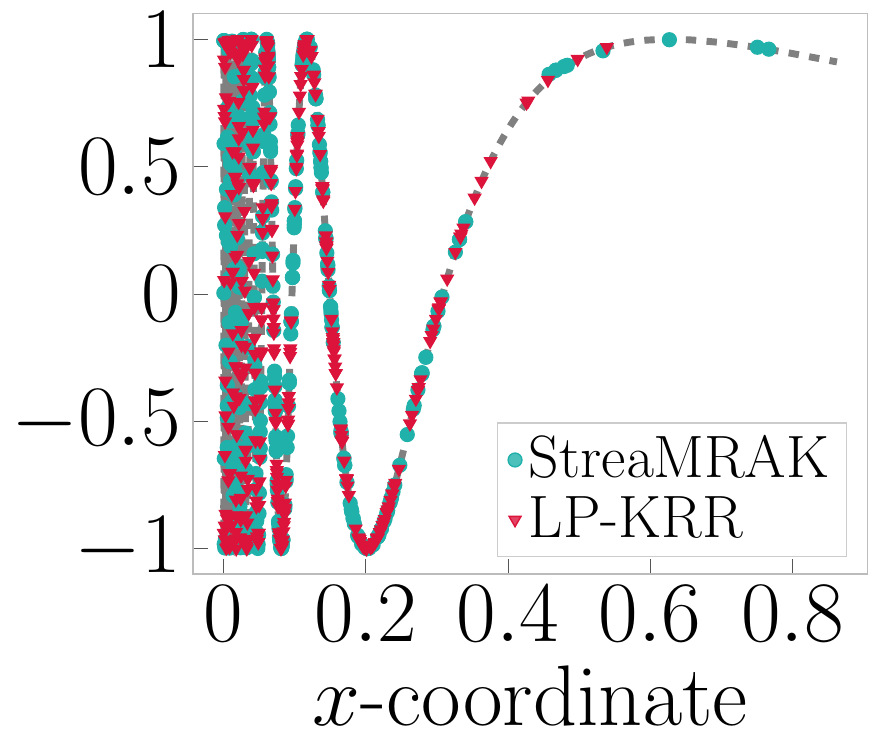} }}%
    \subfloat[\label{subfig:varsinus_falkon_pred}\centering Single level ]{{\includegraphics[width=0.239\textwidth]{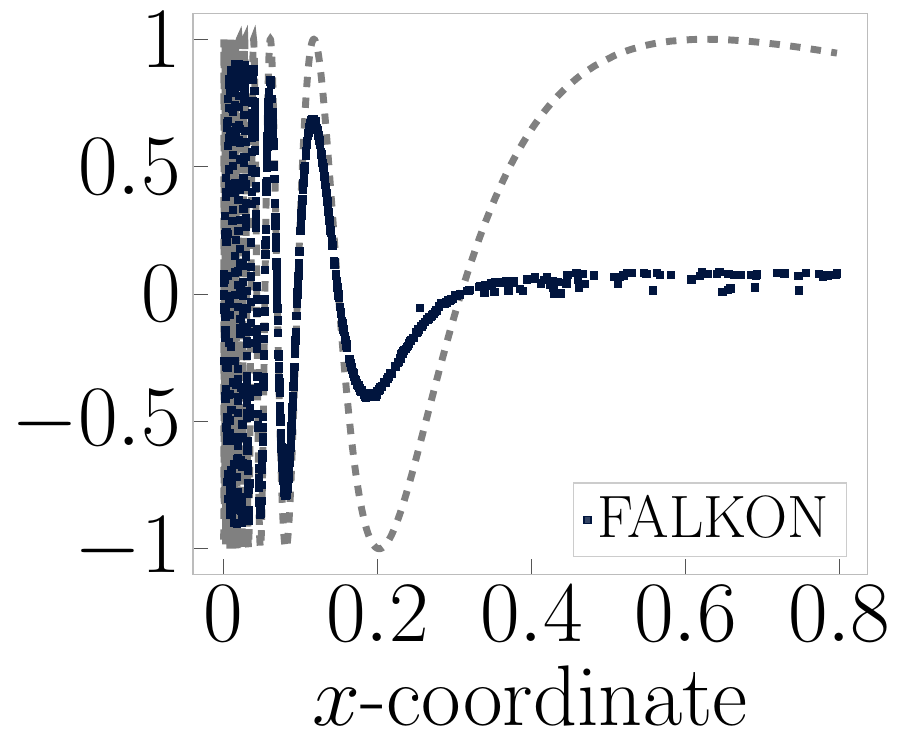} }}%
    \figcap
    \caption{(a)-(d) shows the target function $f(x)$ from Eq. \eqref{eq:varsinus_target} as a grey dotted line. The light-blue circles indicates the predicted values made by \methodE. Similarly the red triangles indicates the predictions made by \lapkrr and the dark blue squares the predictions made by \falkonE.}
    \label{fig:varsinus_predictions}%
\end{figure}

Results are presented in Table \ref{table:compare_falkon_mrfalkon_streamrak_on_varsinus}, and the prediction results are illustrated in Fig. \ref{subfig:varsinus_streamrak_pred_lvl3}-\ref{subfig:varsinus_falkon_pred}. The results show that \method and both \lapkrr schemes perform much better than \falkonE. The reason is that \falkon uses only one bandwidth $r$, while the multi-resolution schemes \method and \lapkrrE, utilize a bandwidth regime $r_l = 2^{-l}r_0$ that varies with the level $l$. The consequence is that \method and \lapkrr approximate the low-frequency components of $f$ when the bandwidth is large, and then target the high-frequency components of $f(x)$ gradually as the bandwidth decreases. These results illustrate the benefits of a multi-resolution scheme over a single bandwidth scheme. 

From Table \ref{table:compare_falkon_mrfalkon_streamrak_on_varsinus}, we also observe that \lapkrr (2) is significantly slower than \method and \lapkrr (1). This is because it uses the entire training set at each level. Therefore, since \lapkrr (1) and \lapkrr (2) achieve comparable precision, we see that including all training points at each level is not always necessary. 

A closer comparison of \method and \lapkrr is given in Fig. \ref{fig:mse_varsin_mrfalkon_vs_streamrak}. In particular, in Fig. \ref{subfig:varsin_mse_vs_level} we see that the two algorithms achieve very similar precision. However, comparing the training times in Fig. \ref{subfig:varsin_mse_vs_time}, we see that \method trains each level faster and therefore achieves better precision earlier than \lapkrr (1).

\begin{figure}[tb!]
    \centering
    \subfloat[\label{subfig:varsin_mse_vs_level}\centering Mean square error]{{\includegraphics[ width=0.435\textwidth]{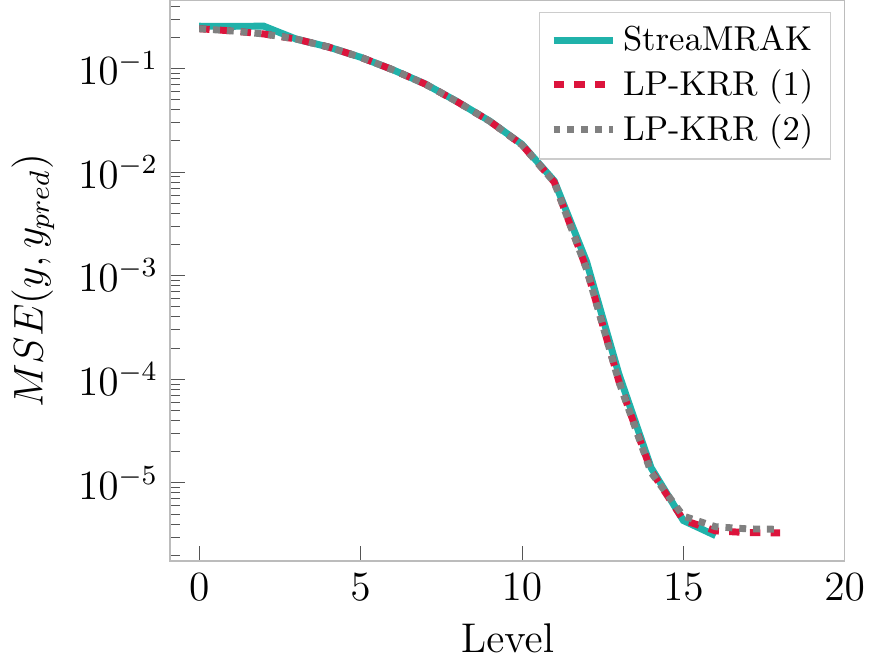} }}%
    \qquad\qquad
    \subfloat[\label{subfig:varsin_mse_vs_time}\centering Training time]{{\includegraphics[width=0.41\textwidth, trim={3pt 3pt 0pt 0pt}]{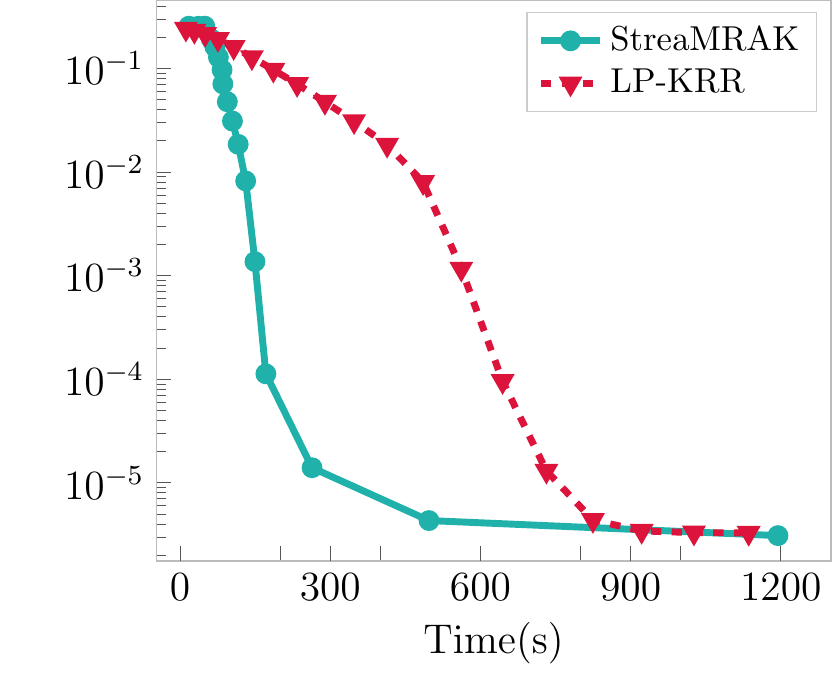} }}%
    \figcap
    \caption{Comparison of \methodE{} and \lapkrrE. (a) shows the mean square error calculated according to Eq. \eqref{eq:MSEequationExperiments} with the target function from Eq. \eqref{eq:varsinus_target}. Along the x-axis is the number of levels included in the model. (b) The x-axis shows the accumulated training time until a level in the LP is completed. The y-axis shows the MSE of the prediction using the currently available model. The blue circles indicate the prediction error of \method and the red triangles indicate the prediction error of \lapkrr (1).}
    \label{fig:mse_varsin_mrfalkon_vs_streamrak}%
\end{figure}

\begin{figure}[htb!]
    \centering
    \subfloat[\label{subfig:avg_nn_lvl_3_streamrak_mrfalkon}\centering Level 3]{{\includegraphics[width=0.305\textwidth, trim={0pt 0pt 1pt 1pt}]{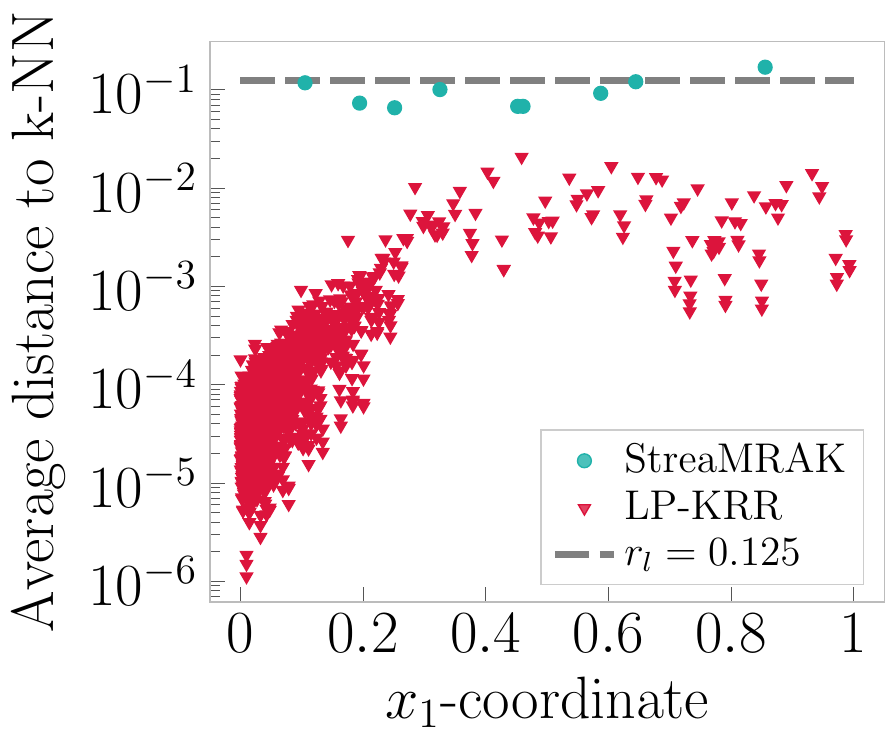} }}%
    \qquad
    \subfloat[\label{subfig:avg_nn_lvl_5_streamrak_mrfalkon}\centering Level 5]{{\includegraphics[width=0.28\textwidth]{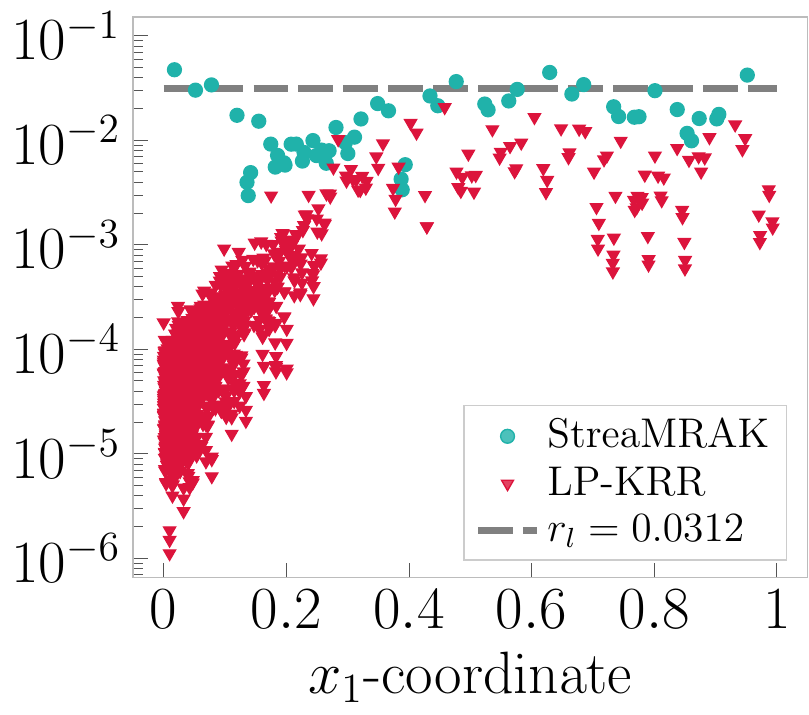} }}%
    \qquad
    \subfloat[\label{subfig:avg_nn_lvl_16_streamrak_mrfalkon}\centering Level 16]{{\includegraphics[width=0.28\textwidth]{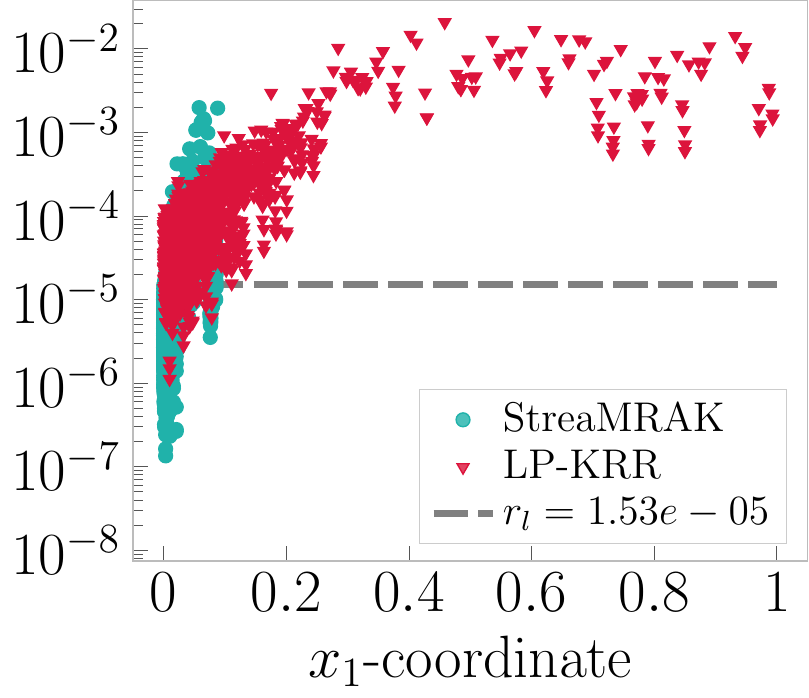} }}%
    \figcap
    \caption{(a)-(c) shows the landmarks with their position along the $x_1$ axis and the average distance to their 2 nearest neighbors along the y-axis. Here the red triangles are the Nystr\"{o}m landmarks of \lapkrr and the light-blue circles the landmarks of \method. The grey dotted line is the bandwidth at the given level.}
    \label{fig:avg_nn_streamrak_mrfalkon}%
\end{figure}

In Fig. \ref{fig:avg_nn_streamrak_mrfalkon} we show the average distance of each landmark to their 2 nearest neighbors (2-NN distance). 
Two aspects of the selection require attention.
As opposed to \lapkrrE, \method selects landmarks such that the 2-NN distance is comparable to the bandwidth used at a specific level. 
In addition, \method saves computational power by not choosing landmarks in regions where the 2-NN distance is too low compared to the bandwidth. 
In  Fig. \ref{subfig:avg_nn_lvl_16_streamrak_mrfalkon} this can be observed for level $l=16$ for landmarks with $x\ge 0.2$.
Due to the non-uniform sample distribution with a higher density around $x=0$, the adaptive sub-sampling is able to select more landmarks in the region close to $x=0$, where $f$ oscillates with high frequency. 
Furthermore, \method stops predicting at level $16$ because level $17$ is not yet covered with a high enough density of landmarks. 
Meanwhile, \lapkrr continues, but as seen from Fig. \ref{subfig:varsin_mse_vs_level} the improvements after level $15$ are not significant because the density of Nystr\"{o}m samples is too low compared to the bandwidth.

\subsection{Adaptive sub-sampling benchmark}
\label{section:dumbell_experiment}
We consider a dumbbell-shaped domain embedded in $\bbR^5$, consisting of two 5-dimensional spheres connected by a 2-dimensional plane. A projection of the input domain in $\bbR^3$ is shown in 
Fig. \ref{fig:Compare_landmarks_selection_StreaMRAK_vs_MRFALKON} (a)-(c). Furthermore, as target we consider the following function,
\begin{equation}\label{eq:dumbell_target}
    f(\xbs) = 
    \begin{cases}
    A\sin(Bx_1+ \phi) + (x_1 + 2), & 1 < x_1 < 3 \\
    1, & \text{ otherwise }\\
    
    \end{cases}, \text{ for } \xbs\in[-1,5]\times[-1, 1]^4,
\end{equation}
where  $A, B$ and $\phi$ are chosen so that $f\in \CC^1([-1,5]\times[-1, 1]^4, \bbR^5)$. For the experiments, we consider a training set of $\SI{1.9e+6}{}$ samples and a test set of $\SI{6e+5}{}$ samples, all sampled uniformly at random from the input domain. We note that we purposefully chose a simple function in the high dimensional regions because complicated functions in high dimensions require far too many points to be satisfactorily learned.

To determine the number of training points for \methodE, we let $\delta_1=\SI{1e-3}{}$ and $\delta_2 = \SI{1e-4}{}$,  cf. Def. \ref{def:sufficien_training_points}.
With this choice \method selects between \SI{30100}{} and \SI{40100}{} training points for each level. \falkon again uses all the $\SI{1.9e6}{}$ training points and for \lapkrr we consider two settings: \lapkrr (1) using  $\SI{1.8e5}{}$ training points at each level, and \lapkrr (2) using  $\SI{1.9e6}{}$ training points at each level.

The results for \textsf{StreaMRAK}, \lapkrrE, and \falkon are presented in Table \ref{table:dumbell_experiment_complexSinus}. We observe that \method achieves a better prediction than both \falkon and \lapkrrE because it adapts the sub-sampling density to the level of resolution.   

\begin{table}[htb!]
\caption{Comparison of \methodE, \lapkrrE, and \falkon predictions of the target function in Eq. \eqref{eq:dumbell_target}. For each level $l$ we show the number of landmarks, the mean square error (MSE), and the accumulated time to train the prediction model (Time). In parenthesis, in the time column of the \falkon row, is the time to find the optimal bandwidth through cross-validation.}
\begin{center}
\tabcap
\begin{tabular}{c|c|c|c|c}
    & Level &  \# Landmarks & MSE & Time \\ \hline
    \multirow{4}{*}[-2mm]{\method}& 4 & \SI{352}{} & \SI{1.29e-3}{} & \SI{64}{\second} \\ 
    & 5 & \SI{2667}{} & \SI{1.27e-3}{} & \SI{1398}{\second} \\ 
    & 6 & \SI{1858}{} & \SI{8.31e-4}{} & \SI{1462}{\second} \\ 
    & 8 & \SI{1329}{} & \SI{2.75e-5}{} & \SI{2307}{\second} \\ \hline
    \multirow{4}{*}[-3mm]{\begin{tabular}{@{}c@{}}\lapkrr (1) \\ $n_l =\SI{1.8e5}{}$ \end{tabular}}& 4 & \SI{1375}{} & \SI{1.28e-3}{} & \SI{386}{\second} \\ 
    & 5 & \SI{1375}{} & \SI{1.26e-3}{} & \SI{520}{\second} \\ 
    & 6 & \SI{1375}{} & \SI{9.10e-4}{} & \SI{671}{\second} \\ 
    & 8 & \SI{1375}{} & \SI{3.30e-4}{} & \SI{1064}{\second} \\ 
    & 9 & \SI{1375}{} & \SI{3.10e-4}{} & \SI{1287}{\second} \\ \hline
    \multirow{4}{*}[-3mm]{\begin{tabular}{@{}c@{}}\lapkrr (2) \\ $n_l =\SI{1.9e6}{}$ \end{tabular}}& 4 & \SI{1375}{} & \SI{1.34e-3}{} & \SI{4160}{\second} \\ 
    & 5 & \SI{1375}{} & \SI{1.30e-3}{} & \SI{5570}{\second} \\ 
    & 6 & \SI{1375}{} & \SI{9.44e-4}{} & \SI{7168}{\second} \\ 
    & 8 & \SI{1375}{} & \SI{3.16e-4}{} & \SI{11125}{\second} \\ 
    & 9 & \SI{1375}{} & \SI{3.01e-4}{} & \SI{13334}{\second} \\ \hline
    FALKON & -- & \SI{14830}{} & \SI{6.8e-4}{} & \SI{6590}{\second}+(\SI{37561}{\second})
\end{tabular}
\end{center}
\label{table:dumbell_experiment_complexSinus}
\end{table}

To understand the improvement in prediction accuracy, we need to discuss the effects of landmark selection.
In Fig. \ref{subfig:dumbell_lm_dist_lvl4}-\ref{subfig:dumbell_lm_dist_lvl6} we show the projections of landmarks for \method and \lapkrr on $\bbR^3$, and in Fig. \ref{subfig:dumbell_lm_avg_distance_lvl4}-\ref{subfig:dumbell_lm_avg_distance_lvl6} the average distance of each landmark to its $7$ nearest neighbors.
These distances are compared with the bandwidth $r_l$ selected for the given level $l$. We see that \method selects landmarks in regions where the average distance to nearest neighbors is comparable to the bandwidth. 
This means that in high dimensional regions, which correspond to $x_1\in[-1,1]\cup[3,5]$, the algorithm effectively stops collecting landmarks since it cannot maintain high enough density. 
On the other hand, \lapkrr uses Nystrom sub-sampling, which imposes a uniform selection of landmarks. 
Consequently, a significant number of landmarks come from high-dimensional regions.  

\begin{figure}[h!]%
    \centering
    \subfloat[\label{fig:comparison_low_and_high_dim_landmarks}\centering Landmarks distribution ]{\includegraphics[width=0.42\textwidth, trim={0pt, 0pt, -3pt, -3pt}]{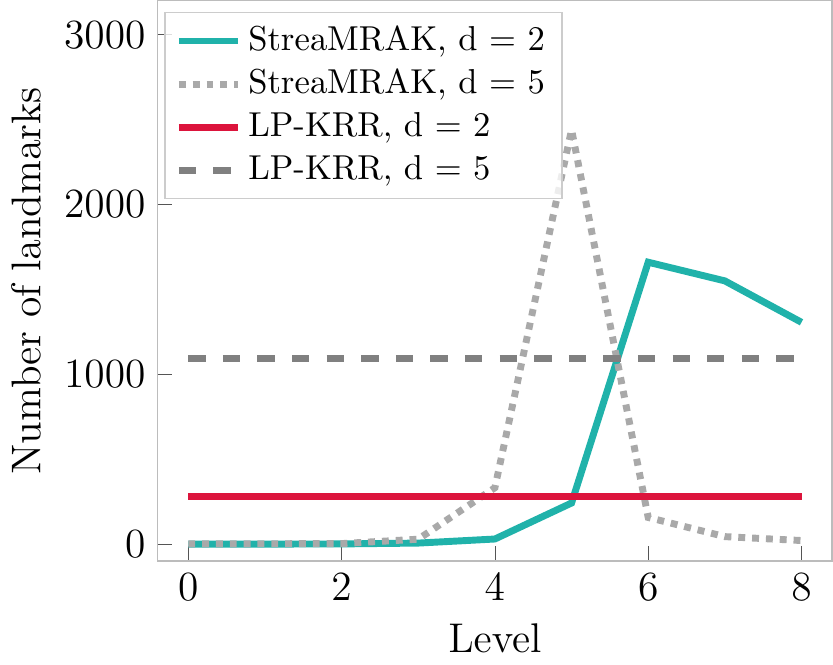}}%
    \qquad\qquad
    \subfloat[\label{fig:mse_low_and_high_dimensions}\centering Mean square error ]{{\includegraphics[width=0.42\textwidth]{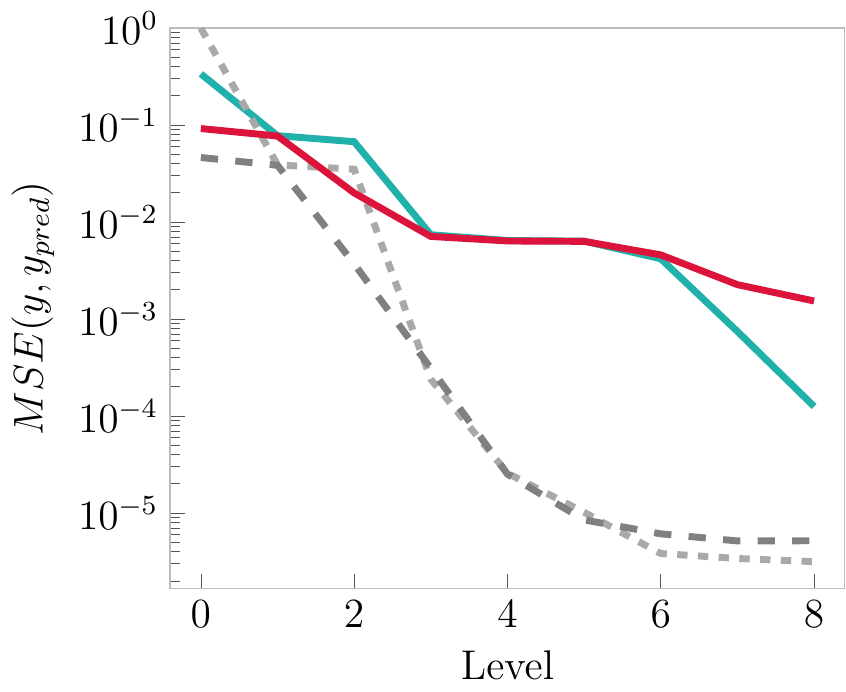} }}%
    \figcap
    \caption{Comparison of \method and \lapkrr (1) in the 2-dim and 5-dim regions of the Dumbbell domain. The solid blue line is \method for dimension $d=2$ while the solid red line is \lapkrr (1) for dimension $d=2$. The grey dotted line is \method for dimension $d=5$ and the dark-grey dashed line is \lapkrr (1) for dimension $d=5$  (a) shows the mean square error calculated according to Eq. \eqref{eq:MSEequationExperiments}. (b) shows the number of landmarks in the 2-dimensional and the 5- dimensional regions.}
    \label{fig:Comparison_high_and_low_dim_dumbell}%
\end{figure}

\begin{figure}[htb!]%
    \centering
    \subfloat[\label{subfig:dumbell_lm_dist_lvl4}\centering Level 4]{{\includegraphics[width=0.29\textwidth]{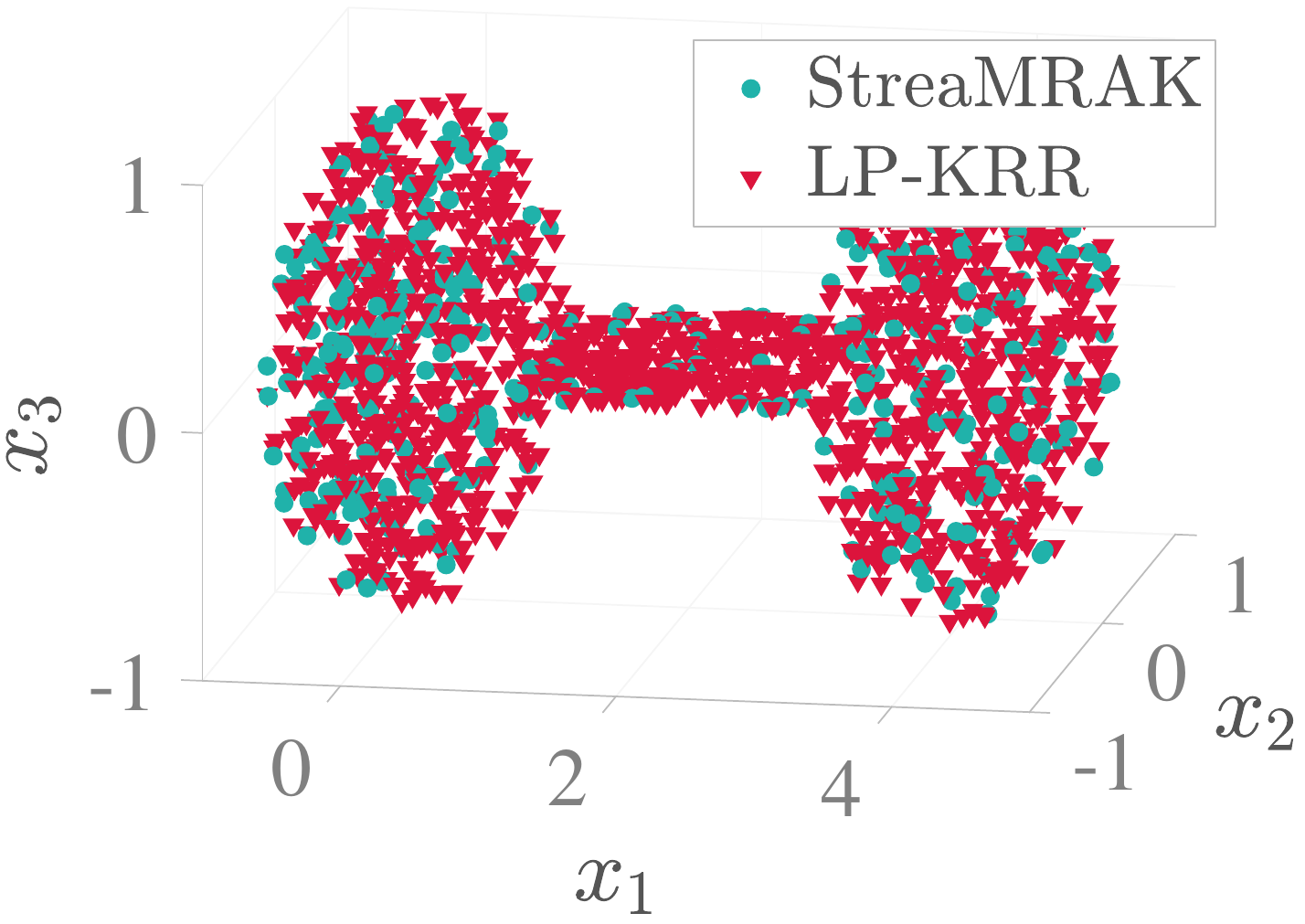} }}%
    \qquad
    \subfloat[\label{subfig:dumbell_lm_dist_lvl5}\centering Level 5]{{\includegraphics[width=0.29\textwidth]{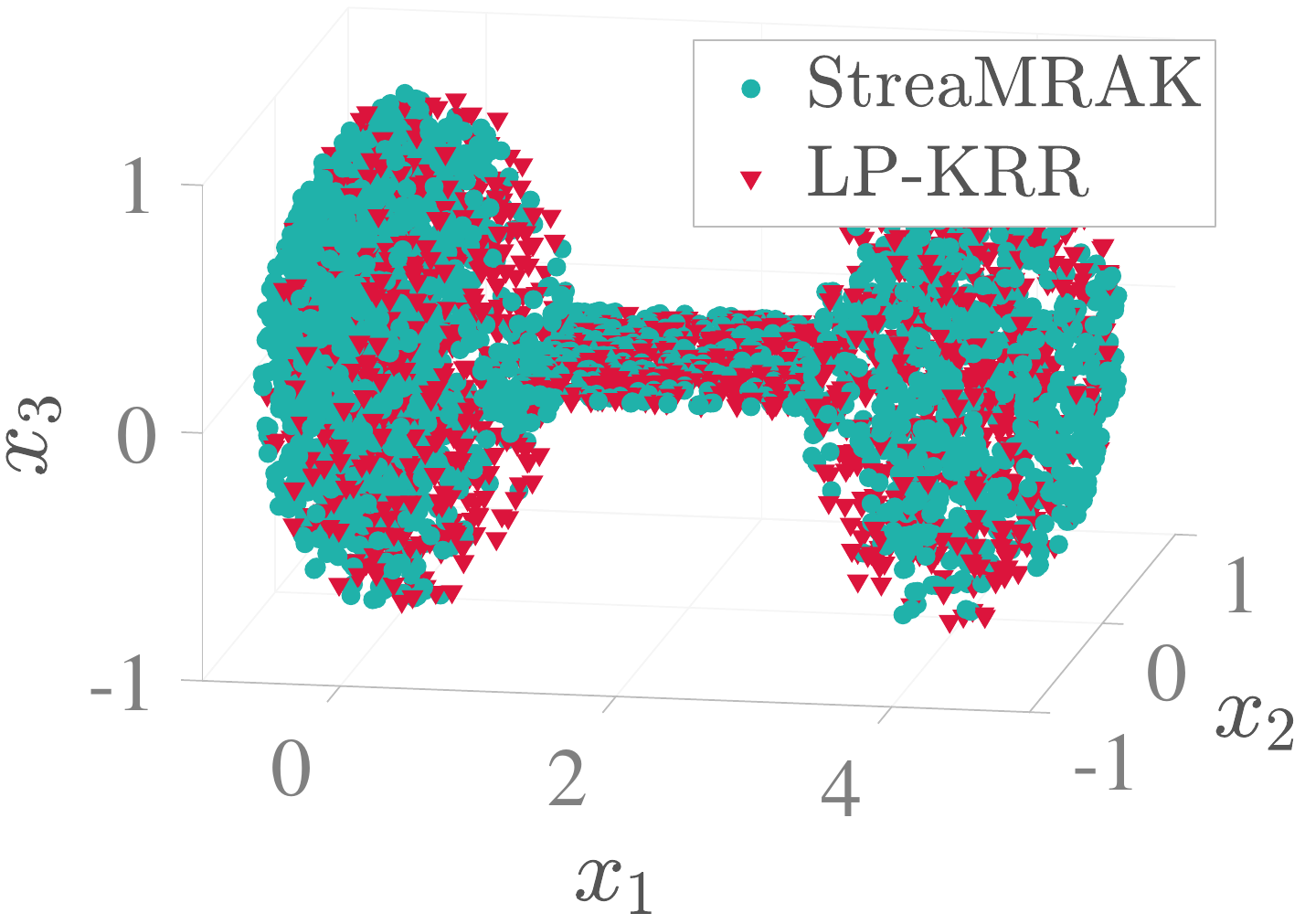} }}%
    \qquad
    \subfloat[\label{subfig:dumbell_lm_dist_lvl6}\centering Level 6]{{\includegraphics[width=0.29\textwidth]{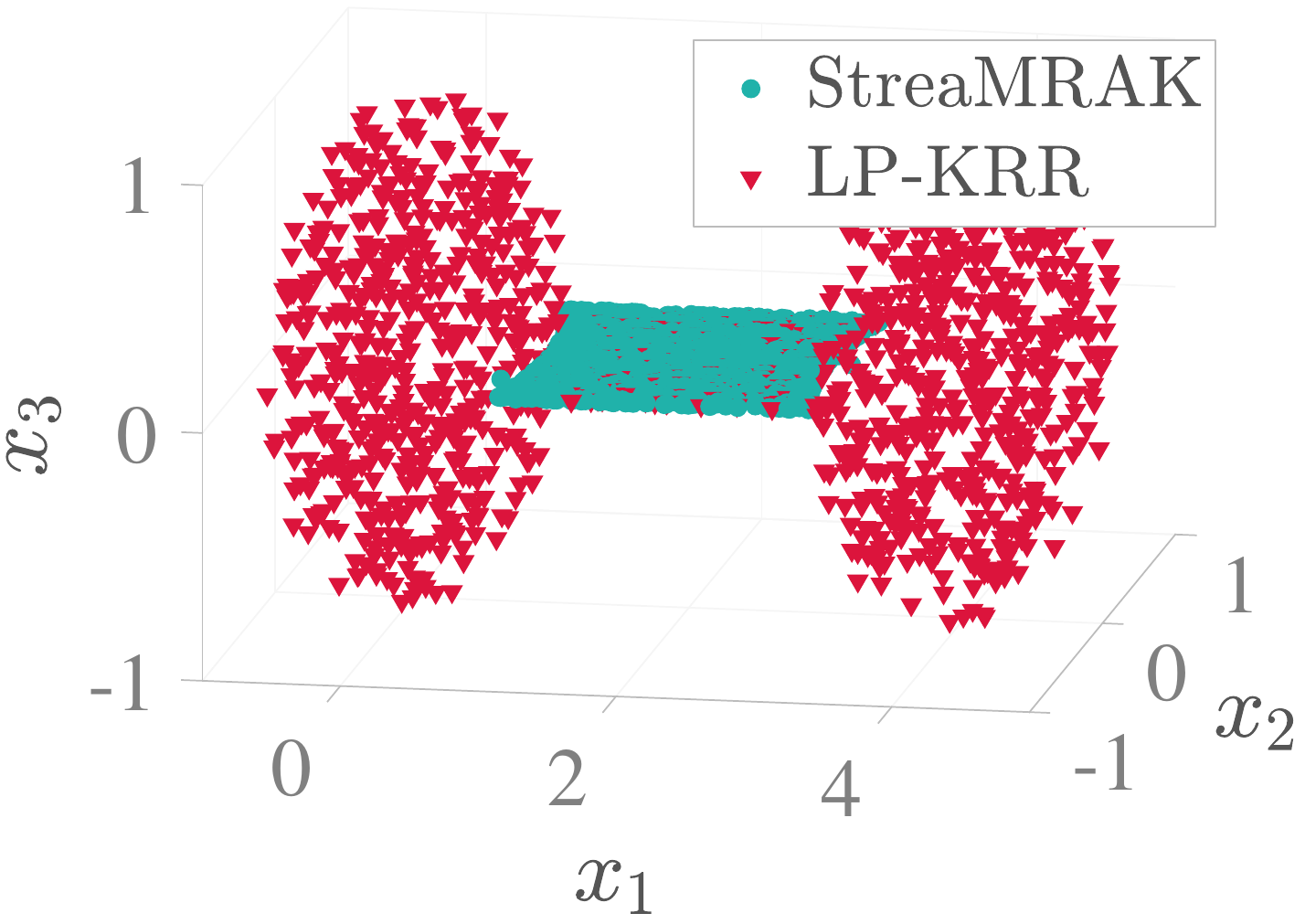} }}%
    \qquad
    \subfloat[\label{subfig:dumbell_lm_avg_distance_lvl4}\centering Level 4]{{\includegraphics[width=0.305\textwidth]{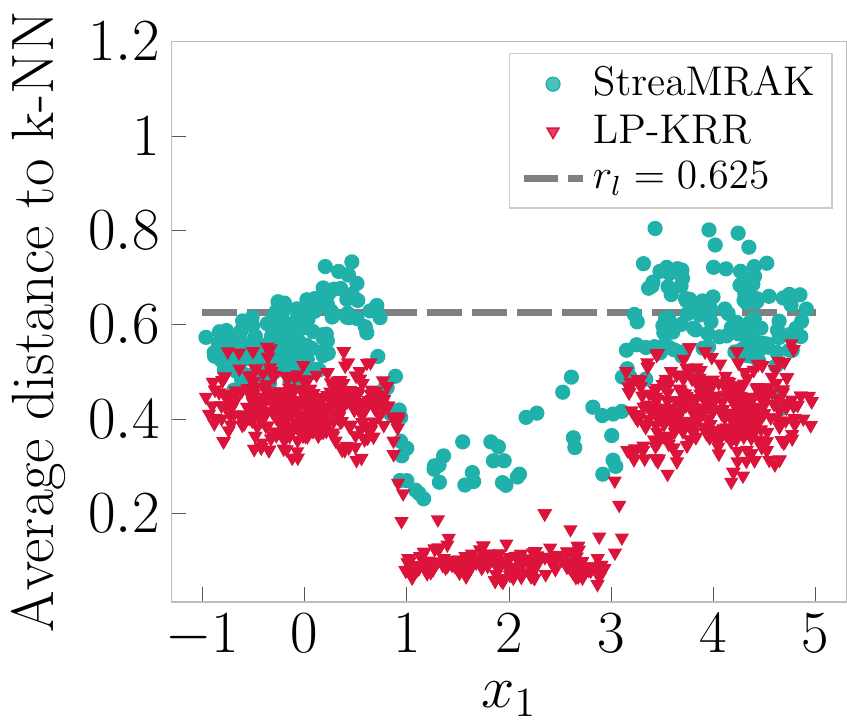} }}%
    \qquad
    \subfloat[\label{subfig:dumbell_lm_avg_distance_lvl5}\centering Level 5]{{\includegraphics[width=0.28\textwidth]{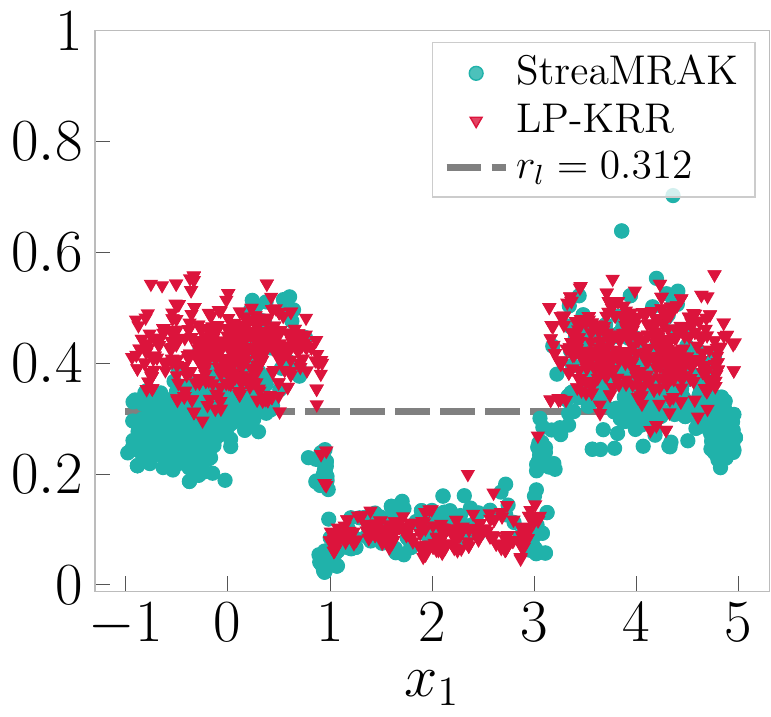} }}%
    \qquad
    \subfloat[\label{subfig:dumbell_lm_avg_distance_lvl6}\centering Level 6]{{\includegraphics[width=0.28\textwidth]{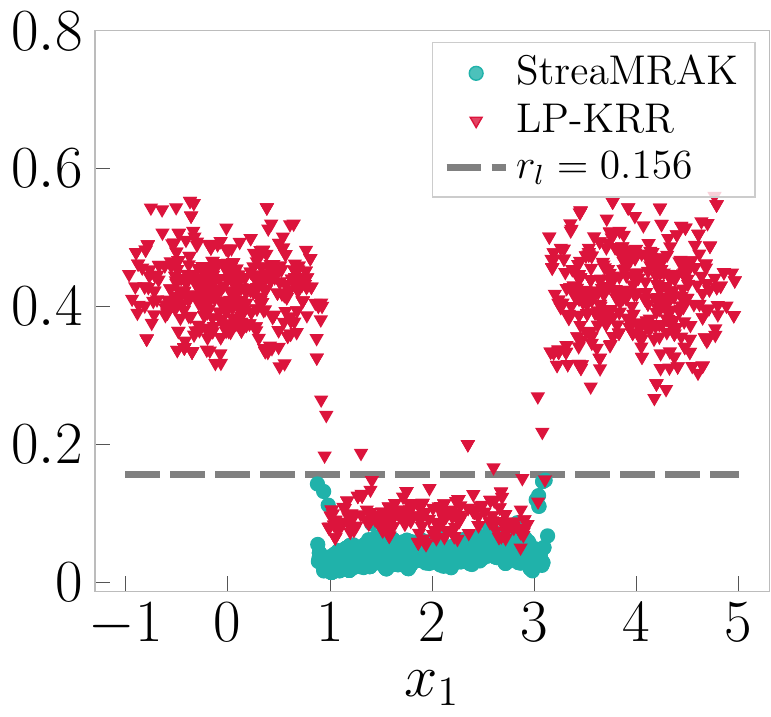} }}%
    \caption{In the figure, red triangles correspond to \lapkrr and light-blue circles to \methodE. (a)-(c) shows the landmark distributions projected on $\bbR^3$ at level $l=4,5,6$ respectively. (d)-(f) shows the average distance between the 7 nearest neighbors; bandwidth $r_l$ is indicated with a dotted line.}
    \label{fig:Compare_landmarks_selection_StreaMRAK_vs_MRFALKON}%
\end{figure}

Moreover, Fig. \ref{fig:Compare_landmarks_selection_StreaMRAK_vs_MRFALKON} shows that in the case of \lapkrrE, the average distance between the landmarks in high dimensional regions is larger than the bandwidth $r_l$ when $l\geq5$. 
As a knock-on effect, \lapkrr makes only small improvements in high dimensional regions for  $l\geq 5$, as seen from Fig. \ref{fig:mse_low_and_high_dimensions}. 
Analogous behavior can be observed for \methodE. 
However, since \method devotes fewer resources to high dimensional regions, it sub-samples more from the low dimensional region, as illustrated in Fig. \ref{fig:comparison_low_and_high_dim_landmarks}. The consequence is that \method makes bigger improvements in the low dimensional region than \lapkrrE, as seen from Fig. \ref{fig:mse_low_and_high_dimensions}. Note that this was not the case in Section \ref{section:varsinus_experiment}, where the two methods had similar behavior, but unlike here, the input domain in Section \ref{section:varsinus_experiment} did not consist of regions with different dimensionalities.

\FloatBarrier
\subsection{Forecasting the trajectory of a double pendulum}
\label{section:doublePend}
We consider the double pendulum, illustrated in Fig. \ref{fig:DoublePendIllustration}, which we model by the Lagrangian system
\begin{equation}
    \CL = ml^2(\omega_1^2+\frac{1}{2}\omega_2^2) +ml^2\omega_1\omega_2\cos{(\theta_1-\theta_2)} +mgl(2\cos{\theta_1}+\cos{\theta_2}),
    \label{eq:lagrangian_dp}
\end{equation}
under the assumption that the pendulums are massless rods of length $l_1=l_2=l$ with masses $m_1=m_2=m$ centered at the end of each rod. Here $g$ is the standard gravity, $\omega_1\coloneqq \dot\theta_1$, $\omega_2\coloneqq \dot\theta_2$ are the angular velocities, and the angles $\theta_1$, $\theta_2$ are as indicated in Fig. \ref{fig:DoublePendIllustration}. For the experiments we let $m=1$, $l=1$ and $g=10$.

The learning task is to forecast the trajectory of the pendulum, given only its initial conditions. We let $\sbs_t=[\theta_1(t),\, \theta_2(t),\, \omega_1(t),\, \omega_2(t)]\in \bbR^4$ be the state of the system at step $t\in \bbN$ and train \methodE, \lapkrr and \falkon to learn how $\sbs_t$ maps to a later state $\sbs_{t+\Delta}$, for $\Delta\in\bbN$. The trained model $\widehat{f}$ is used to forecast the state $\sbs_T$ for $T>>0$ by recursively predicting $\sbs_{t+\Delta} = \widehat{f}(\sbs_t)$ from the initial state $\sbs_0$ until $t=T$.

For the experiments we consider two settings: a low energy system $\sbs_0^\textit{low} = [-20\degree,\, -20\degree,\, 0\degree,\, 0\degree]$ and a high energy system $\sbs_0^\textit{high} = [-120\degree,\, -20\degree,\, -7.57\degree,\, 7.68\degree]$. For these systems, we initialize $8000$ pendulums as $\sbs_0\sim\CN(\sbs, \sigma(\sbs))$ for $\sbs = \sbs_0^\textit{low},\,\sbs_0^\textit{high}$ respectively, where $\sigma(\sbs)=[0.025|\theta_1|,\, 0.15|\theta_2|,\, 0.3|\omega_1|,\, 0.3|\omega_2|]$. Each pendulum is iterated for $500$ steps, which results in $5\times10^6$ training points distributed in $\bbR^4$. Furthermore, for the test data we consider $100$ pendulums $\sbs_0\sim\CN(\sbs, 0.01|\sbs|)$ for $\sbs = \sbs_0^\textit{low},\, \sbs_0^\textit{high}$, iterated for 500 steps.

To determine the number of training points for \methodE, we let $\delta_1,\delta_2=10^{-4}$, cf. Def. \ref{def:sufficien_training_points}. With this choice \method selects between \SI{30219}{} and \SI{70282}{} training points for each level for the low energy system, and between \SI{36300}{} and \SI{130200}{} for the high energy system. Meanwhile, \falkon uses all $\SI{5.0e6}{}$ training points and \lapkrr use $\SI{3.9e5}{}$ training points at each level. 

Results are presented in Table \ref{table:doublePend_falkon_vs_streamrak_lowE} and \ref{table:doublePend_falkon_vs_streamrak_highE}. Furthermore, to illustrate the prediction results we consider the center of mass $\overline{M}_x(\sbs_t)=\frac{1}{2}(x_1(t)+x_2(t))\in \bbR$ at state $\sbs_t$, where $x_1,\, x_2\in\bbR$ are the positions of the two pendulum masses as seen in Fig. \ref{fig:DoublePendIllustration}. The prediction results are illustrated in Fig. \ref{fig:Pred_doublePend_lowE} and \ref{fig:Pred_doublePend_highE} for the low and high energy pendulums respectively. We calculate the MSE at each step $t$ separately, such that for a given $t$ we use Eq. \ref{eq:MSEequationExperiments} with $\ybs_k=\overline{M}_x(\sbs_t)$, $\ybs^\textit{pred}_k=\overline{M}_x(\sbs^\textit{pred}_t)$ and $\Upsilon=100$.

\begin{table}[htb!]
\caption{Comparison of \methodE, \lapkrrE, and \falkon for the low energy system. For each level $l$ we show the number of landmarks, the MSE at step T=50, and the accumulated time to train the prediction model (Time). In parenthesis, in the time column of the \falkon row, is the time to find the optimal bandwidth through cross-validation.}
\tabcap
\begin{center}
\begin{tabular}{c|c|c|c|c}
    & Level &  \# Landmarks & MSE(T=50) & Time \\ \hline
    \multirow{4}{*}{\method} & 2 & \SI{1}{} & \SI{1.12e-1}{} & \SI{31}{\second}\\
    & 5 & \SI{66}{} & \SI{3.39e-5}{} & \SI{498}{\second} \\ 
    & 7 & \SI{659}{} & \SI{1.44e-6}{} & \SI{534}{\second} \\ 
    & 9 & \SI{4085}{} & \SI{3.01e-7}{} & \SI{812}{\second} \\ \hline
    \multirow{4}{*}{\lapkrr} & 2 & \SI{1979}{} & \SI{5.93e-2}{} & \SI{490}{\second}\\
    & 5 & \SI{1979}{} & \SI{2.73e-5}{} & \SI{1463}{\second} \\ 
    & 7 & \SI{1979}{} & \SI{2.16e-7}{} & \SI{2395}{\second} \\
    & 9 & \SI{1979}{} & \SI{1.16e-8}{} & \SI{3550}{\second} \\ \hline
    FALKON & -- & \SI{19790}{} & \SI{5e6}{} & \SI{2934}{\second}+(\SI{1498}{\second})
\end{tabular}
\end{center}
\label{table:doublePend_falkon_vs_streamrak_lowE}
\end{table}

\begin{table}[t!]
\caption{Comparison of the \methodE, \lapkrrE, and \falkon for the high energy system. For each level $l$ we show the number of landmarks, the MSE at step T=50, and the accumulated time to train the prediction model (Time). In parenthesis, in the time column of the \falkon row, is the time to find the optimal bandwidth through cross-validation.}
\tabcap
\begin{center}
\begin{tabular}{c|c|c|c|c}
    & Level &  \# Landmarks & MSE(T=50) & Time \\ \hline
    \multirow{4}{*}[2mm]{\method} & 2 & \SI{1}{} & \SI{2.70e-1}{} & \SI{49}{\second}\\
    & 5 & \SI{1106}{} & \SI{8.53e-3}{} & \SI{915}{\second} \\ 
    & 7 & \SI{6376}{} & \SI{2.16e-4}{} & \SI{1999}{\second} \\ \hline
    \multirow{4}{*}[2mm]{\lapkrr} & 2 & \SI{1979}{} & \SI{1.72e-2}{} & \SI{522}{\second}\\
    & 5 & \SI{1979}{} & \SI{5.09e-3}{} & \SI{1474}{\second} \\ 
    & 7 & \SI{1979}{} & \SI{1.39e-4}{} & \SI{2431}{\second} \\ \hline
    \falkon & -- & \SI{19790}{} & \SI{5e6}{} & \SI{23830}{\second}+(\SI{11050}{\second})
\end{tabular}
\end{center}
\label{table:doublePend_falkon_vs_streamrak_highE}
\end{table}

For the low energy system, we see from Fig. \ref{subfig:mse_vs_time_dp_lowE} how \method is trained significantly faster than \lapkrrE, although at a cost of reduced precision. The reduced training time of \method is a consequence of the low doubling dimension of the training data, which allows the selection of far fewer landmarks for \method than what is used at each level in \lapkrrE.

For the high-energy pendulum, we see from Fig. \ref{subfig:dp_mse_vs_time_highE} that \method is again able to achieve good precision faster than \lapkrr. Furthermore, we see that the number of landmarks selected for \method increases abruptly with the levels, reflecting the high doubling dimension of the training data. Due to this \method stops the training after level $7$, as the next levels require too many landmarks. By continuing for $2$ more levels \lapkrr is able to achieve marginally better precision but at an increased computational cost.

\begin{figure}[h!]
    \centering
    \subfloat[\label{subfig:mse_error_doublePend_lowE}\centering Mean square error ]{{\includegraphics[width=0.30\textwidth, trim={2.5pt, 2.5pt, 0pt, 0pt}]{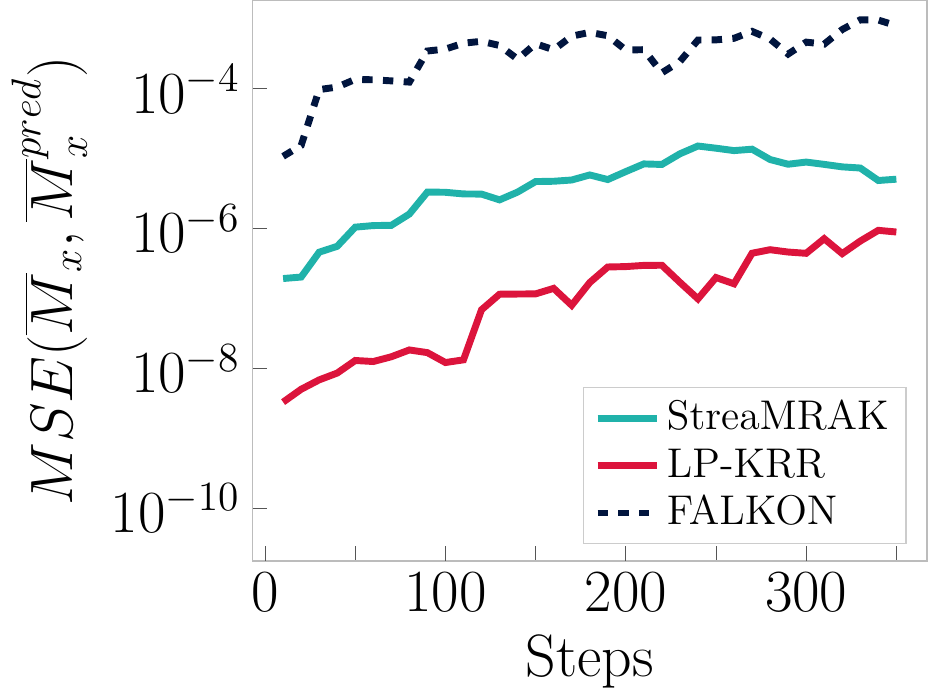} }}%
    \qquad
    \subfloat[\label{subfig:cm_pred_doublePend_lowE}\centering Forecasting trajectory ]{{\includegraphics[width=0.28\textwidth, trim={0pt, 0pt, 9pt, 9pt}]{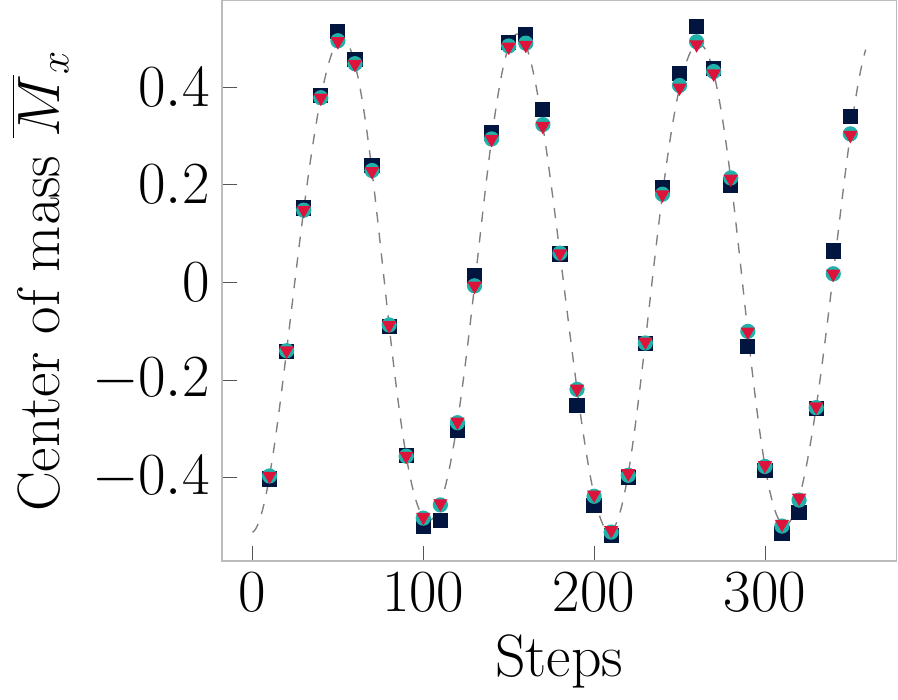} }}%
    \qquad
    \subfloat[\label{subfig:mse_vs_time_dp_lowE}\centering Training time ]{{\includegraphics[width=0.28\textwidth, trim={0pt, 0pt, 9pt, 9pt}]{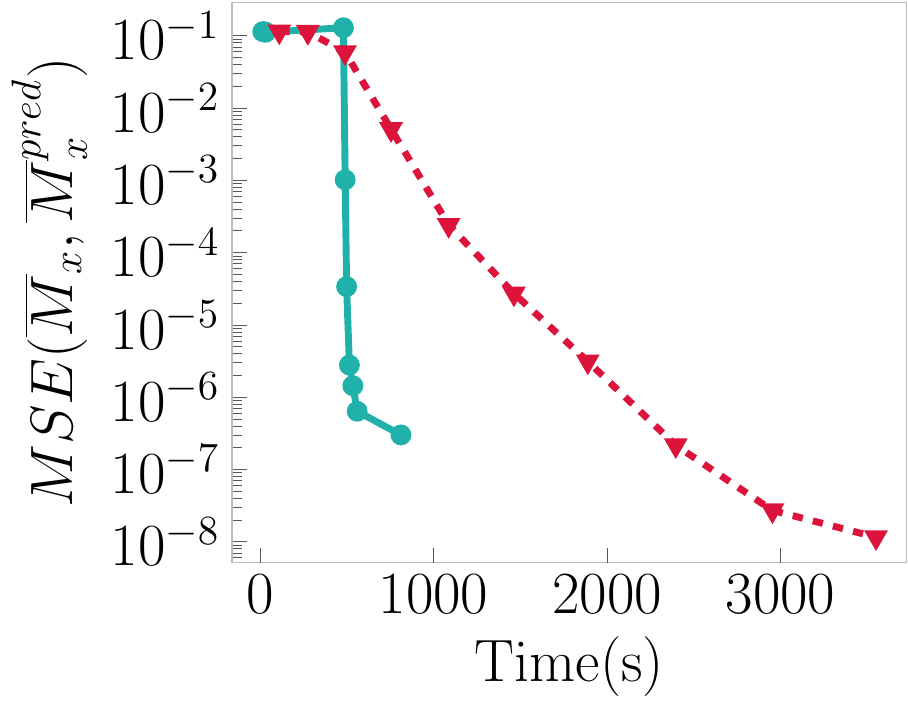} }}%
    \figcap
    \caption{Comparison of \method (light blue lines and circles), \lapkrr (red lines and triangles), and \falkon (dark blue dotted lines and squares) for the low energy pendulum. (a) Shows the mean square error of the center of mass $\overline{M}_x(\sbs_t)$ for the level $7$ prediction, with step $T$ along the x-axis. (b) shows the true center of mass trajectory as a grey dotted line and the predictions of \methodE, \lapkrrE, and \falkon at level $9$. (c) The x-axis shows the accumulated training time until a level in the LP is completed. The y-axis shows the MSE of the predicted system state after $T=50$ steps. We note that \method includes $7$ levels, while \lapkrr includes $9$.}
    \label{fig:Pred_doublePend_lowE}%
\end{figure}

\begin{figure}[tb!]
    \centering
    \subfloat[\label{subfig:mse_error_doublePend_highE}\centering Mean square error ]{{\includegraphics[width=0.30\textwidth, trim={0pt, 0pt, -3pt, -3pt }]{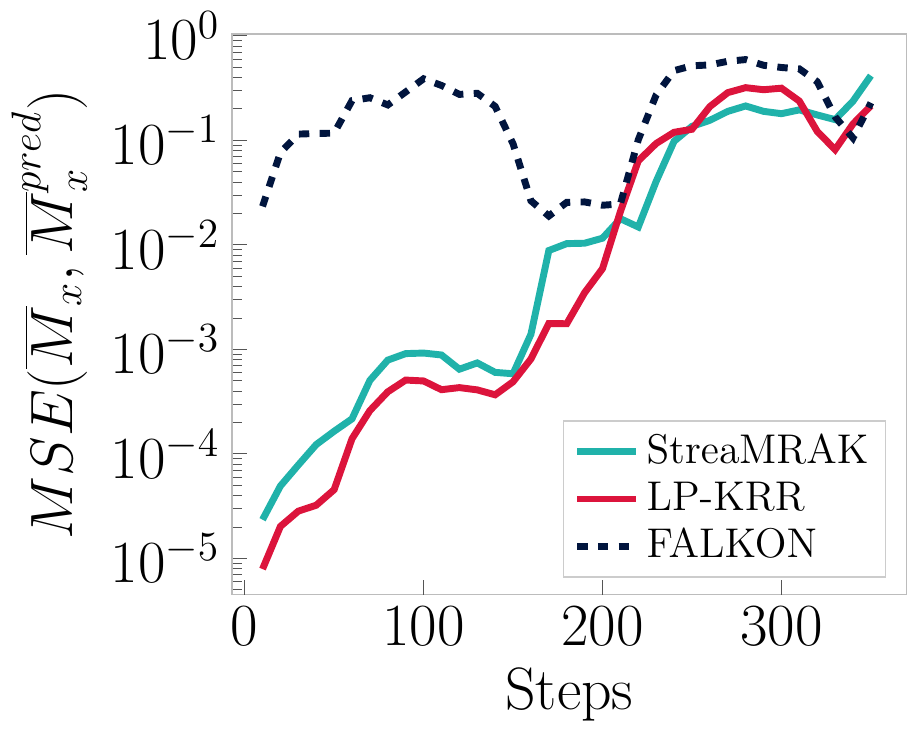} }}%
    \qquad
    \subfloat[\label{subfig:cm_pred_doublePend_highE}\centering Forecasting trajectory ]{{\includegraphics[width=0.28\textwidth]{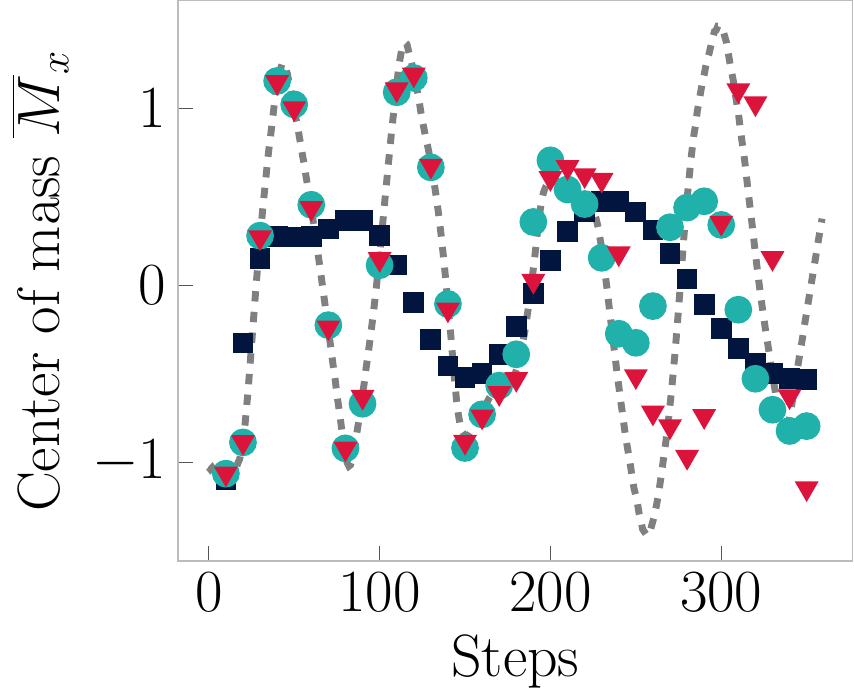} }}%
    \qquad
    \subfloat[\label{subfig:dp_mse_vs_time_highE}\centering Training time]{{\includegraphics[width=0.30\textwidth]{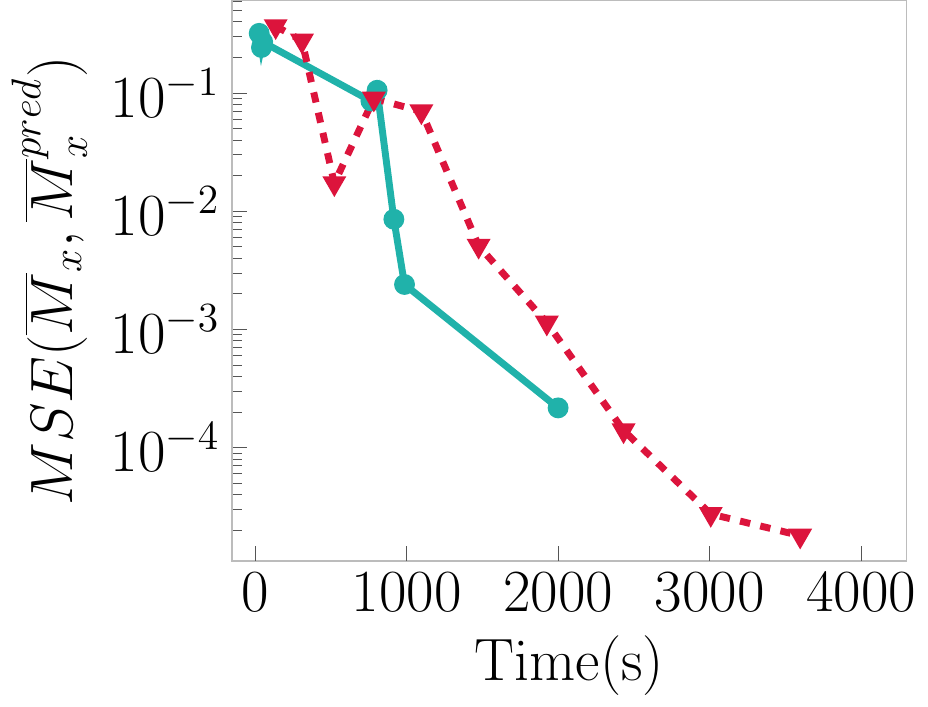} }}%
    \figcap
    \caption{Comparison of \method (light blue lines and circles), \lapkrr (red lines and triangles), and \falkon (dark blue dotted lines and squares) for the high-energy pendulum. (a) Shows the mean square error of the center of mass $\overline{M}_x(\sbs_t)$ for the level $7$ prediction, with step $T$ along the x-axis. (b) shows the true center of mass trajectory as a grey dotted line and the predictions of \methodE, \lapkrrE, and \falkon at level $9$. (c) The x-axis shows the accumulated training time until a level in the LP is completed. The y-axis shows the MSE of the predicted system state after $T=50$ steps.}
    \label{fig:Pred_doublePend_highE}%
\end{figure}

As seen in Fig. \ref{subfig:cm_pred_doublePend_highE}, the forecasting of \method and \lapkrr breaks down after $T\approx 200$ steps. In Fig. \ref{subfig:phase_diagram_with_bifurcation} we observe the trajectory of a pendulum with initial condition $\sbs_0^\textit{high}$, as well as four pendulums with a $0.5\%$ perturbation on the angles $\theta_1$ and $\theta_2$ in $\sbs_0^\textit{high}$. We observe that after roughly $T=205$ time steps the trajectory of the five pendulums diverge significantly from each other. Therefore, it seems that a bifurcation point occurs around this time, which may explain why all the algorithms are unable to make good forecasting beyond this point.


\begin{figure}[t!]
    \centering
    \subfloat[\label{subfig:pendulums_positions_step19}\centering Step $T=200$.]{{\includegraphics[width=0.28\textwidth]{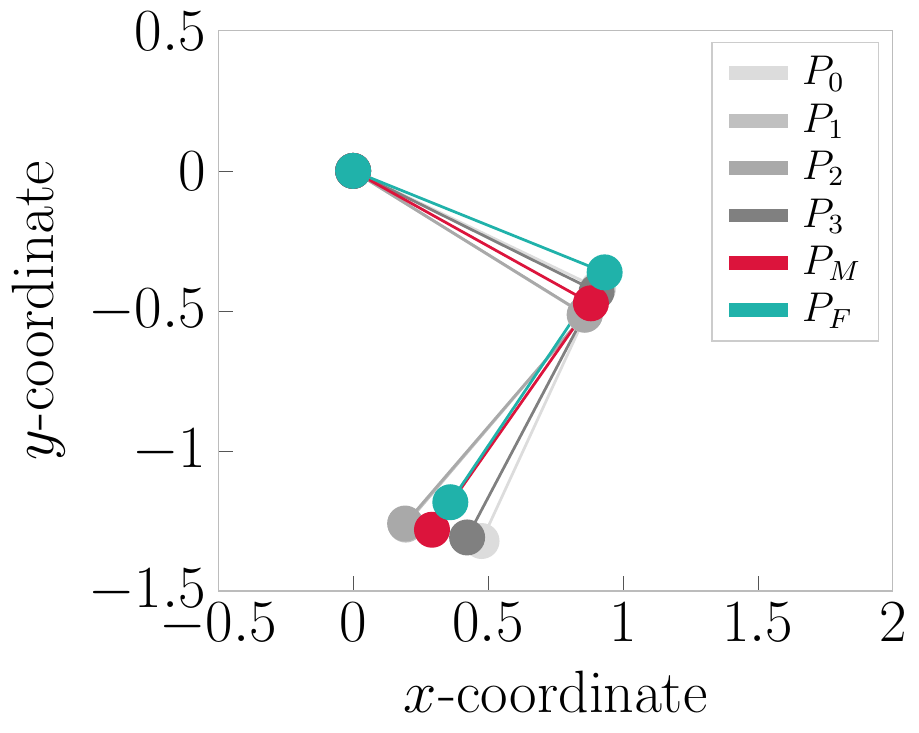} }}%
    \qquad
    \subfloat[\label{subfig:phase_diagram_with_bifurcation}\centering Projection of the pendulum trajectories on the $\theta_1\theta_2$-plane.]{{\includegraphics[width=0.265\textwidth, trim={2pt, 2pt, 0pt, 0pt}]{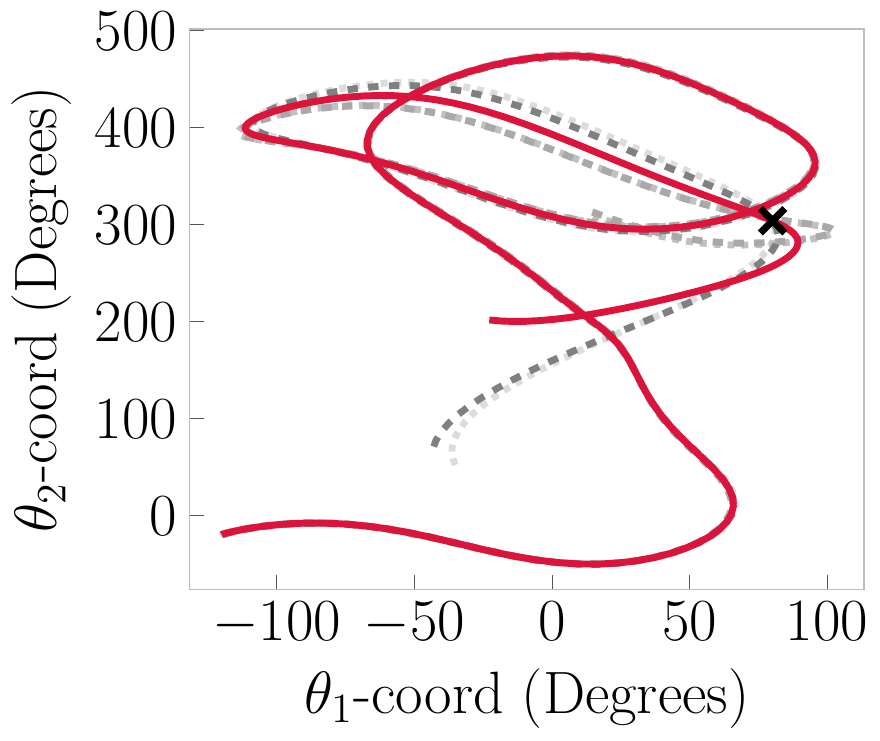} }}%
    \qquad
    \subfloat[\label{subfig:pendulums_positions_step20}\centering Step $T=210$.]{{\includegraphics[width=0.28\textwidth]{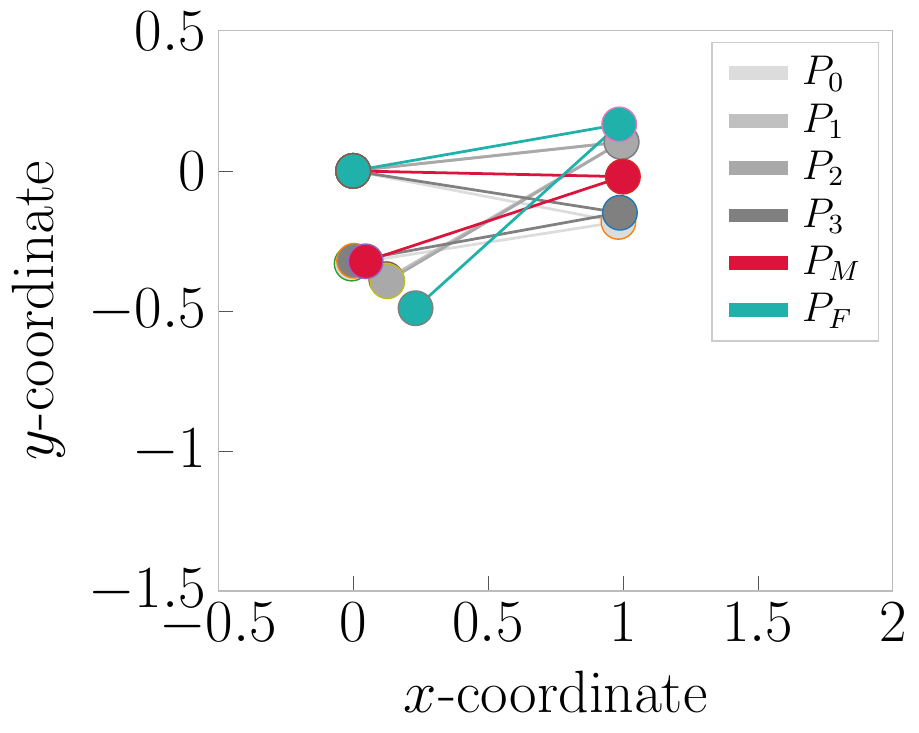} }}%
    \figcap
    \caption{(a) Pendulum positions at $T=200$ and $(c)$ The positions at $T=210$. In (a) and (c), $P_M$ is the main pendulum with initial conditions $\sbs_0^\textit{high}$, while $P_F$ is the \methodE forecast of the pendulum position. Similarly, $P_0$-$P_3$ are four training pendulums with a perturbation of $0.5\%$ on the initial angles $\theta_1$ and $\theta_2$ of the main pendulum. (b) Projection of the training data on the $\theta_1\theta_2$-plane. The thick red line is the main pendulum corresponding to $P_M$ and the four grey dotted lines are the test pendulums $P_0$-$P_3$, where the X indicates the time $T = 205$.}
    \label{fig:Illustration_of_bifurcation_point}%
\end{figure}

\section{Outlook}
\label{section:Outlook}
Further development of \method is intended with focus on four objectives.

\begin{enumerate}[label=(O\arabic*)]
    \item \label{Track_error} Augmentation of the DCT to track the error at each node
    \item \label{Improve_estim} Improve the estimator in Def. \ref{def:sufficien_training_points} and Eq. \ref{eq:estimator_of_covering_fraction}.
    \item \label{refine_prev_lvl} Refinement of previously fitted levels in the LP as new data arrives.
    \item \label{further_theory} Further theoretical analysis of the LP.
\end{enumerate}

Considering objective \ref{Track_error} we intend to develop the DCT to track the error at each node. This way the growth can be restricted in regions where the error is small, which allows for more focus on regions where the error is large. The intention is that this will reduce the problem complexity even further, while also increasing the precision. Regarding objective \ref{Improve_estim}, a drawback with the estimator in Eq. \ref{eq:estimator_of_covering_fraction} was already mentioned in Remark \ref{remark:weakness_with_cf_estimator}. Furthermore, for the estimator in Def. \ref{def:sufficien_training_points}, we intend to implement and evaluate alternative ways to estimate the convergence of the matrices. Another focus area will be objective \ref{refine_prev_lvl}, as we believe new information may be revealed as new training data arrive, and refinement of previously fitted levels can therefore be beneficial. Finally, the theoretical analysis in objective \ref{further_theory} will focus on analyzing the generalization error for the LP, particularly in combination with the adaptive sub-sampling scheme.

\section{Acknowledgement}
We especially would like to thank Prof. Pieter Abeel at UC Berkeley and Asst. Prof. Sicun Gao at UC San Diego for their input on the double pendulum system, and for providing a code example for this system. We would also like to thank Sami Ortoleva at UC San Diego for his discussion on the analysis of the damped cover-tree. AO is part of the Simula-UCSD-UiO Research and Ph.D. training program (SUURPh), an international collaboration in computational biology and medicine funded by the Norwegian Ministry of Education and Research, {\v Z}K is funded by UK EPSRC grant EP/T000864/1, AC is funded by NSF DMS 1819222, 2012266, and Russell Sage Foundation grant 2196 and YF is funded by the NIH grant NINDS (PHS) U19NS107466 Reverse Engineering the Brain Stem Circuits that Govern Exploratory Behavior.

\bibliography{References}
\newpage

\section*{Appendix A. Algorithms}
\label{appendix:A}
\renewcommand{\theequation}{A.\arabic{equation}}

\renewcommand{\thealgorithm }{A.\arabic{algorithm}}

\setcounter{equation}{0}
\setcounter{algorithm}{0}

We here denote nodes by $p,q, c$ and $\xbs_p,\,\xbs_q, \xbs_c\in\CX\subset\bbR^D$ are the corresponding points.

\begin{algorithm}[htb]
\caption{\textsc{Insert}(point $q$, node $p$, level $l$)}
\begin{algorithmic}[1]
\STATE We assume $q$ already satisfies $\|\xbs_q-\xbs_p\| \leq 2^{-l}r_0$.
\IF{$\|\xbs_q-\xbs_c\| > 2^{-(l+1)}r_0$ for all $c\in \textit{Children}(p)$}
    \STATE Insert $q$ into $\textit{Children}(c)$.
    \STATE\textsc{Update\_CoverFraction}(\textit{Parent}($Q_l$), "No parent found")
    \STATE \textbf{Break}
\ELSIF{$\|\xbs_q-\xbs_c\| < 2^{-(l+1)}r_0$ for some $c\in\textit{Children}(p)$}
    \STATE Consider all children of $c$, namely $\textit{Children}(c)$
    \IF{$\textit{Children}(c)$ is empty}
        \IF{Covering fraction of $p$, Def. \ref{def:cover_fraction}, satisfy $\cofp\ge \CD_{\cof}$ for some threshold $\CD_{\cof}$}
            \STATE Insert q into $\textit{Children}(c)$
            \STATE \textbf{Break}
        \ELSE
            \STATE \textsc{Update\_CoverFraction}(p, "parent found")
            \COMMENT{c is found to be a potential parent. However, since $\cofp < \CD_{\cof}$ we can not add $q$ to $\textit{Children(c)}$}
        \ENDIF
    \ELSE
        \STATE\textsc{Insert}($q$, $c$, $l+1$)
    \ENDIF
\ENDIF
\end{algorithmic}
\label{alg:DampedCoverTree}
\end{algorithm}

\begin{algorithm}[htb!]
\caption{\textsc{StreaMRAK}(point $\xbs$, target $y$)}
\begin{algorithmic}[1]
    \STATE Let $l$ be the level. Let $p_{(0)}$ be the root node, $r_0$ the radius of the root node.
    \STATE \textbf{Sub-sampling thread}
    \STATE Insert $\xbs$ into the cover tree with $\textsc{Insert}(\xbs, p_{(0)}, l=0)$. \COMMENT{See Alg. \ref{alg:DampedCoverTree}}
    \IF{a new level has $\coflvl \geq \mathcal{D}_{level}$.}
        \STATE Extract the landmarks at level $l$ as sub-samples, namely $\Gamma^{(l)}_{m^{(l)}}$.
    \ENDIF
    \STATE \textbf{Training thread}
    \STATE Consider level $l$ and assume that the landmarks $\Gamma^{(l)}_{m^{(l)}}$ are extracted.
    \WHILE{$l$ is not sufficiently covered with training points according to Def. \ref{def:sufficien_training_points}.}
        \STATE Update $ \big[ (\VK^{(l)}_{nm})^\top \VK^{(l)}_{nm} \big]_{ij}$ and $\zbs^{(l)}_i$ according to Eq. \eqref{eq:updateFormula_KnmTKnm} and Eq. \eqref{eq:updateFromualte_Zm} as new samples $(\xbs,y)$ arrive, using the landmarks in $\widetilde\Gamma^{(l)}_m$ from Def. \ref{def:landmarks}.
            \STATE Continuously check if matrices have converged.
        \IF{Matrices converge according to Def. \ref{def:sufficien_training_points}}
        \STATE Update the \method regression model $\widetilde{f}^{(L)}$, by including the correction term $s^{(l)}$ into the Laplacian pyramid, as described in Section \ref{section:TheLaplacianPyramid}. Let $L=l$ and update $l=l+1$.
    \ENDIF
\ENDWHILE
\end{algorithmic}
\label{alg:PseudoCode_StreaMRAK}
\end{algorithm}

\begin{algorithm}[htb]
\caption{\textsc{Update\_CoverFraction}(node $p$, string s)}
\begin{algorithmic}[1]
\IF{s="No parent found"}
    \STATE Update covering fraction of $p$ with $\cofp = (1-\alpha)\cofp$
\ELSIF{s= "parent found"}
    \STATE Update covering fraction of $p$ with $\cofp = (1-\alpha)\cofp + \alpha$
\ENDIF
\end{algorithmic}
\label{alg:updateCoverFraction}
\end{algorithm}

\FloatBarrier

\section*{Appendix B. Preparatory material}
\label{appendixB}

\renewcommand{\theequation}{B.\arabic{equation}}

\renewcommand{\thetheorem }{B.\arabic{theorem}}

\setcounter{equation}{0}
\setcounter{theorem}{0}

We offer preparatory material on the damped cover-tree and kernel methods.

\subsection*{B.1 Preparatory material on the damped cover-tree}
This section shows how the recursive formula in Eq. \ref{eq:estimator_of_covering_fraction} approximates the weighted average of the outcome of the last $N$ random trails. 
Where the trails are as described in Section \ref{subsection:DCT_construction}. 
By expanding Eq. \ref{eq:estimator_of_covering_fraction} we have $ (\cofp)_t = (1-\alpha)^t(\cofp)_1 + \alpha \sum_{i=1}^{t-1}(1-\alpha)^i\mathbbm{1}_{\mathcal{B}_c}(\xbs_{t-i})$. Since $(1-\frac{1}{N})^N\approx 1/e$, the first term becomes negligible when $t \gg N$. Similarly, all terms $i > N$ in the sum becomes negligible. This leaves,
\begin{equation*}
    (\cofp)_t \approx \frac{1}{N}\sum_{i=1}^N\bigg(1-\frac{1}{N}\bigg)^i\mathbbm{1}_{\mathcal{B}_c}(\xbs_{t-i})
\end{equation*}
which is a weighted average of the outcome of the $N$ last draws as claimed.

\begin{remark}
We mention a weakness of the estimator in Eq. \eqref{eq:estimator_of_covering_fraction}. As follows from Algorithm \ref{alg:DampedCoverTree}, every time a new point $\xbs$ is not covered by the existing children, a new child is added. This consequently updates $\CB_c$, leading to the posterior distribution $\text{Prob}(\mathbbm{1}_{\CB_c}(\xbs)=0|\xbs)$ to changed every time $\mathbbm{1}_{\CB_c}(\xbs)=0$. 
\label{remark:weakness_with_cf_estimator}
\end{remark}

\subsection*{B.2 Preparatory material on Kernel methods}
Kernel methods in the context of reproducing kernel Hilbert spaces (RKHS) offer a powerful approach to machine learning with a well-established mathematical foundation \cite{hofmann2008kernel, scholkopf2002learning}. In this paper we consider an input space $\CX \subset \bbR^D$, a corresponding target space $\CY \subset \bbR$ and let $\rho$ be the probability distribution on $\CX\times\CY$. Furthermore, we assume an RKHS $\CH_k$ generated by a positive definite kernel $k:\CX\times\CX\rightarrow\bbR$. In other words, the eigenvalues $\sigma_i,\dots,\sigma_n$ of the corresponding kernel matrix $\VK_{nn}=(k(\xbs_i,\xbs_j))_{i,j=1}^n$ satisfies $\sigma_i > 0$ for all $i\in n$. In this setting the inner product between two feature vectors $\phi(\xbs), \phi(\xbs^\prime) \in \CH_k$ satisfies the property that $\DP{\phi(\xbs)}{\phi(\xbs^\prime)}_{\CH_k} = k(\xbs, \xbs^\prime)$. This relation, known as the "kernel trick" \cite{aiserman1964theoretical, boser1992training}, effectively circumvents the need for explicit construction of non-linear mappings $\phi$. 

Given a training set $\{(\xbs_i,y_i):i\in[n]\}$ sampled according to $\rho$ with $\Gamma_n=\{\xbs_i:i\in[n]\}$, we formulate the kernel ridge regression (KRR) problem as
\begin{equation}
    \widehat f_{n,\lambda} =\argmin_{f\in\widehat\CH_n} \frac{1}{n}\sum_{i=1}^n (f(\xbs_i)-y_i)^2  + \lambda \N{f}_\CH^2,
    \label{eq:KRR_argmin_formulation}
\end{equation}
where $\lambda>0$ is a regularisation parameter and $\widehat\CH_n=\overline{\textrm{span}}\{k(\cdot,\xbs_i):i\in[n]\}$ is a finite-dimensional subspace of $\CH_k$. What is more, for all $f\in\widehat\CH_n$ the Representer theorem \cite{kimeldorf1970correspondence, scholkopf2001generalized} guarantees that there exists coefficients $\alpha_1,\ldots,\alpha_n$ such that the solution to Eq. \eqref{eq:KRR_argmin_formulation} is on the form
\begin{equation*}
    f(\xbs) = \sum_{i=1}^n \alpha_i k(\xbs,\xbs_i).
\end{equation*}
Computing the KRR estimator is therefore reduced to solving the linear system
\begin{equation*}
    (\VK_{nn}+\lambda \VI_n)\Valpha = \ybs,
\end{equation*} 
where $\ybs=(y_1,\ldots,y_n)^\top$, $\Valpha=(\alpha_1,\ldots,\alpha_n)^\top$, and $[\VK_{nn}]_{ij}=k(\xbs_i,\xbs_j)$.

\section*{Appendix C. Proofs and definitions}
\label{appendixC}

\renewcommand{\theequation}{C.\arabic{equation}}

\renewcommand{\thetheorem }{C.\arabic{theorem}}

\setcounter{equation}{0}
\setcounter{theorem}{0}

\begin{lemma}
Consider a domain $\CX\in \bbR^D$, a ball $\CB(\xbs_p, r)\subset \CX$ and let $S=\{\xbs_i,\,\xbs_j\in \CB(\xbs_p,r) | \|\xbs_i-\xbs_j\|\geq \delta \, \text{for}\,\, i\neq j \}$. Furthermore, let the doubling dimension of the set $S$ be $\dd\coloneqq\dd(S, r)$. We
let $c_d\coloneqq |S|$ when $\cofp = 1$. We then have $ 2^{\dd-1} \leq c_d \leq 5^\dd$.
\label{lemma:Bound_on_c_D}
\end{lemma}
\begin{proof}
The upper bound on $c_d$ follows from Lemma \ref{lemma:upperBoundOnNumberOfPointsInBall} with $r=r_0$ and $\delta=r_0/2$. The lower bound follows from the definition of the doubling dimension \ref{def:doubling_dim_at_range}.
\end{proof}

\begin{lemma} Let $\dbs^{(l)}$ be the residual at level $l$ as defined in Eq. \eqref{eq:residual_at_level}. We then have,
\begin{equation*}
    \dbs^{(l+1)} = (\VI-\VK_{nn}^{(l)}(\VK_{nn}^{(l)}+\lambda n\VI)^{-1})\dbs^{(l)} 
\end{equation*}
\label{lemma:residual_expression}
\end{lemma}

\begin{proof}
Denote $\sbs^{(l)}=s^{(l)}([\xbs_n])$, and note that $\sbs^{(l)}=\VK_{nn}^{(l)}(\VK_{nn}^{(l)}+\lambda n\VI)^{-1}\dbs^{(l)}$.
For $l=1$, we have
\begin{align*}
\dbs^{(1)}&=\ybs-\sbs^{(0)}=\ybs-\VK_{nn}^{(l)}\Valpha^{(0)}=(\VI-\VK_{nn}^{(0)})(\VK_{nn}^{(0) }+\lambda n \VI)^{-1}\ybs.
\end{align*}
We proceed by induction. Assume the statement holds for an $l\geq2$.
We now have
\begin{align*}
\dbs^{(l+1)}=\ybs - \sum_{j=0}^l\sbs^{(j)}=\dbs^{(l)}-\sbs^{(l)}=\dbs^{(l)}-\VK_{nn}^{(l)}(\VK_{nn}^{(l)}+\lambda n\VI)^{-1}\dbs^{(l)}
=(\VI-\VK_{nn}^{(l)}(\VK_{nn}^{(l)}+\lambda n\VI)^{-1})\dbs^{(l)}.
\end{align*}
\end{proof}

\subsection*{C.1 Proof of Thm. \ref{thm:LP_KRR_convergence}}
Let $\VP_{nn}^{(l)}\coloneqq\VK^{(l)}_{nn}(\VK^{(l)}_{nn}+\lambda n \VI)^{-1}$.
By definition of $\dbs^{(l)}$ and Lemma \ref{lemma:residual_expression} it follows
\begin{equation}
    f([\xbs_n])-\widehat{f}^{(l+1)}([\xbs_n])=\dbs^{(l+1)}=(\VI-\VP_{nn}^{(l)})\dbs^{(l)} =(\VI-\VP_{nn}^{(l)})(f([\xbs_n])-\widehat{f}^{(l)}([\xbs_n])).
    \label{eq:residual_relation_before_norm}
\end{equation}
We then have
\begin{equation}
    \|\widehat{f}^{(l+1)}([\xbs_n])-f([\xbs_n])\| \leq \|\VI-\VP_{nn}^{(l)}\| \|\widehat{f}^{(l)}([\xbs_n])-f([\xbs_n])\|.
    \label{eq:appendix_res_ident}
\end{equation}

Consider the SVD $\VK^{(l)}_{nn} = \VU\VSigma \VU^\top$ where $\VSigma= \diag{(\sigma_{l,i})}$ and $\sigma_{l,n} \leq \dots \leq \sigma_{l,1}$. We then have
\begin{align}
\begin{split}
    \|\VI-\VP^{(l)}_{nn}\| &=\bigg\| \VU\diag{\Big(\frac{n\lambda}{n\lambda+\sigma_{l,i}}\Big)}\VU^\top\bigg\|=\bigg\|\diag{\Big(\frac{n\lambda}{ n\lambda+\sigma_{l,i}}\Big)}\bigg\| \\
    &= \frac{n\lambda}{n\lambda+\sigma_{l,n}}\coloneqq 1-\varepsilon(l),
\end{split}
\label{eq:defining_eps_l}
\end{align}
and Thm. \ref{thm:LP_KRR_convergence} follows recursively from Eq. \eqref{eq:appendix_res_ident} and Eq. \eqref{eq:defining_eps_l}.
\hfill\qedsymbol

\subsection*{C.2 Proof of Thm. \ref{thm:LP_KRR_convergence_rate}}
To bound the smallest eigenvalue of the kernel matrix $[\VK^{(l)}_{nn}]_{ij} = \Phi(\|\xbs_i-\xbs_j\|)$, namely $\sigma_{l,n}$, we will assume that there exists a lower bound on the  minimal distance between any two points $\xbs_i,\xbs_j\in \CX$, namely $\delta\coloneqq\min\limits_{i\neq j\in \CX}\N{\xbs_i-\xbs_j}>0$. Consider the Gaussian $\Phi(\xbs)=\exp(-\beta\|\xbs\|_2^2)$, $\beta>0$, with the Fourier transform $\widehat{\Phi}(\omega) = (\pi/\beta)^{D/2}\exp(-\|\omega\|_2^2/4\beta)$. From \cite[Corollary 12.4]{Wendland2004Scattered} we have the bound
\begin{equation*}
    \sigma_{l,n} \geq C_D 2^D (2\beta)^{-D/2} \delta^{-D}\exp(-4M_D^2/(\delta^2\beta)),
\end{equation*}
where
\begin{equation*}
    M_D = 12\bigg(\frac{\pi\Gamma^2(D/2+1)}{9}\bigg)^{1/(D+1)} \quad \text{and} \quad
    C_D = \frac{1}{2\Gamma(D/2+1)}\bigg(\frac{M_D}{2^{3/2}}\bigg)^D.
\end{equation*}
With $\beta = (\sqrt{2}2^{-l}r_0)^{-2}$ we then have
\begin{align*}
    \begin{split}
    \sigma_{l,n} &\geq C_D 2^D 2^{-Dl}\bigg(\frac{r_0}{\delta}\bigg)^D\exp\big(-(2\sqrt{2}M_D)^2(r_0/\delta)^2 4^{-l}\big) \\
    & = C_{1,D} 2^{-Dl}\exp\big(-C_{2,D} 4^{-l}\big) \coloneqq B(l),
    \end{split}
\end{align*}
where we define
\begin{equation*}
    C_{1,D} = \frac{1}{2}(6\sqrt{2})^D\Gamma(D/2+1)^{\frac{D-1}{D+1}}\bigg(\frac{\pi}{9}\bigg)^{\frac{D}{D+1}}\bigg(\frac{r_0}{\delta}\bigg)^D \quad \text{and} \quad C_{2,D} = 1152\bigg(\frac{\pi\Gamma^2(D/2+1)}{9}\bigg)^{\frac{2}{D+1}}\bigg(\frac{r_0}{\delta}\bigg)^2.
\end{equation*}
The first bound in Thm. \ref{thm:LP_KRR_convergence_rate} follows from this result. \hfill\qedsymbol
\begin{remark}
In \cite[Thm. 12.3]{Wendland2004Scattered} they also offer an a fortiori bound corresponding to $M_D = 6.38D$, $C_{1,D}=\frac{1}{2}\big(\frac{12.76}{2^{3/2}}\big)^D\big(\frac{D^D}{\Gamma(D/2+1)}\big)\big(\frac{r_0}{\delta}\big)^D$ and $C_{2,D}=(12.76\sqrt{2}D)^2(r_0/\delta)^2$.
\label{remark:fortiori_bound}
\end{remark}

\begin{corollary}
We note that $B(l)$ has a maximum at
\begin{align*}
\begin{split}
    l^* &= \frac{1}{2}\log_2\bigg(\frac{C_{2,D}\log 4}{D\log 2}\bigg) =\log_2\bigg(\sqrt{\frac{D}{2}}\bigg(\frac{r_0}{\delta}\bigg)\bigg) + \log_2\bigg(\frac{4M_D}{D}\sqrt{2}\bigg)
\end{split}
\end{align*}
and is monotonically increasing with $l$ on the interval $l\in(0, l^*)$. Furthermore, with the a fortiori expression for $M_D$ from Remark \ref{remark:fortiori_bound} we have
\begin{equation*}
    l^* = \log_2\bigg(\sqrt{\frac{D}{2}}\bigg(\frac{r_0}{\delta}\bigg)\bigg) + \log_2\bigg(25.52\sqrt{2}\bigg).
\end{equation*}
\label{corollary:Monotonicaly_increasing}
\end{corollary}


When the level $l$ becomes sufficiently large, the kernel matrix $\VK^{(l)}_{nn}$ becomes diagonally dominant, and we can therefore bound the eigenvalues using Garschgorins Theorem \cite[Thm. 1.1]{gomez2006more}, which gives
\begin{equation}
    |\sigma_{l,i} - [\VK^{(l)}_{nn}]_{jj}| = |\sigma_{l,i} - 1| < \sum_{\substack{q=1,\\ q\neq j}}^n |[\VK^{(l)}_{nn}]_{jq}|\quad \text{for}\quad i,j\in [n].
    \label{eq:gaurchinThm_bounding_eValues_close_to_1}
\end{equation}

To find a more explicit bound, we analyze the sum on the right-hand side using Lemma \ref{lemma:upperBoundOnNumberOfPointsInBall}.

\begin{lemma} Consider a ball $\CB(\xbs,r) \in \mathbb{R}^D$ and let $\delta>0$.
The number of points in any (discrete) set of points within $\CB(\xbs,r)$ that are at least $\delta$ apart, $S=\{\xbs_i\in \CB(\xbs,r) | d(\xbs_i, \xbs_j)\geq \delta \, \text{for}\,\, i\neq j \}$, is bounded by $ |S| \leq \bigg( \frac{2r}{\delta}+1\bigg)^D$.
\label{lemma:upperBoundOnNumberOfPointsInBall}
\end{lemma}
\begin{proof} 
Since the points in $S$ are at least $\delta$ apart, it follows that the balls $\CB(\xbs_i, \delta/2)$ are disjoint. Consider now the ball $\CB(\xbs, r+\delta/2)$. All of the balls $\CB(\xbs_i, \delta/2)$ are entirely contained within $\CB(\xbs, r+\delta/2)$. Since the balls $\CB(\xbs_i, \delta/2)$ are disjoint, it follows that
\begin{equation*}
    |S| \leq \frac{\vol\Big(\CB(\xbs, r+\delta/2)\Big)}{\vol\Big( \CB(\xbs_i, \delta/2)\Big)} = \bigg(\frac{2r}{\delta}+1\bigg)^D.
\end{equation*}
\end{proof}

Consider a family of annuli $\{R_t\}_{t=0}^\infty$ where $R_t = \CB(\xbs_j, 2^{t+1}\delta) \backslash \CB(\xbs_j, 2^t\delta)$. 
Inspired by \cite{Leeb2019}, we can interpret the right hand side of Eq. \eqref{eq:gaurchinThm_bounding_eValues_close_to_1} as a sum over $\{R_t\}_{t=0}^\infty$. The entries of $\VK^{(l)}_{nn}$ are defined as
\begin{equation*}
    [\VK^{(l)}_{nn}]_{ij} = \exp{\bigg(-\frac{\N{\xbs_i-\xbs_j}^2}{ 2 r_l^2}\bigg)},\quad \forall i,j\in [n],
\end{equation*}
where $r_l = 2^{-l}r_0$ for $r_0>0$. It follows

\begin{align*}
    \begin{split}
        \sum_{\substack{q=1,\\ q\neq j}}^n |[\VK^{(l)}_{nn}]_{jq}| = \sum_{t=0}^{\infty}\sum_{\xbs_q\in R_t} k^{(l)}(\xbs_j, \xbs_q) \leq \sum_{t=0}^\infty  \bigg(\frac{2^{t+2}\delta}{\delta}+1\bigg)^D 
        \exp\big(-(2^{t}\delta 2^{-1/2}r_l^{-1})^2\big)
    \end{split}
\end{align*}
where in the first term on the right-hand side we bound the number of summands using Lemma \ref{lemma:upperBoundOnNumberOfPointsInBall},
and in the second we use  $\|\xbs_q - \xbs_j\| \geq 2^t\delta$ for $\xbs_q\in R_t$. 
Note now that for all $T\ge 1$ there exists $C_T>0$ such that $\exp(-r^2)\leq C_T r^{-T}$ holds for all $r>0$.
Such a constant is given by the Lambert W function and satisfies $C_T =\left(\frac{T}{2{\rm e}}\right)^{T/2}$. 
Moreover, $2^{t+2}+1 \leq 2^{t+2+\alpha}$, for $\alpha\ge \ln(1+1/4)/\ln(2)$.
Thus, 
\begin{align*}
    \sum_{\substack{q=1,\\ q\neq j}}^n |[\VK^{(l)}_{nn}]_{jq}| \leq C_T \left(\frac{r_l}{\delta}\right)^T 2^{(2+\alpha)D+T/2}\sum_{t=0}^\infty 2^{t(D-T)}\leq C_{DT}\left(\frac{r_l}{\delta}\right)^T,
\end{align*}
where using $\sum_{t=0}^\infty 2^{t(D-T)}\leq 2$, which holds for $D-T< 0$, we let
\[ 2 \cdot 2^{D(2+\alpha)+T/2} C_T
\le2 \cdot 2^{D(2+\alpha)-T/2(1+1/\ln(2))}T^{T/2}=:C_{DT} ,\]
where we used $\exp(1)\ge 2^{1+1/\ln(2)}$.

We now consider the function
\begin{align*}
\begin{split}
    F(T) &\coloneqq 2^{-\frac{T}{2}(1+1/\ln{2})}2^{-lT}T^{\frac{T}{2}}\bigg(\frac{r_0}{\delta}\bigg)^T \\
    & = 2^{-\frac{T}{2}(1+1/\ln{2})}2^{-lT}2^{T/2\log_2T} 2^{T\log_2(r_0/\delta)}\\
    & = 2^{-T/2(B-\log_2 T)} = 2^{f(T)},
\end{split}
\end{align*}
Where $B=1+\frac{1}{\ln 2} + 2l - 2\log_2\bigg(\frac{r_0}{\delta}\bigg)$. $F$ is minimized by
\begin{equation*}
    T^{*} = 2^{B-1/\ln 2} = 2^{1+1/\ln 2 + 2l -2\log_2(r_0/\delta)-1/\ln 2} = 2\cdot 4^{l-\log_2(r_0/\delta)},
\end{equation*}
such that
\begin{equation*}
    F(T^{*}) = 2^{-\frac{T}{2}(B-\log_2 2^B - \log_2 2^{-1/\ln 2})} = 2^{\frac{-T^{*}}{2\ln 2}} = 2^{-\frac{1}{\ln 2}4^{l-\log_2 (r_0/\delta)}}.
\end{equation*}
Inserting this back and with $\alpha=\ln (1+1/4)/\ln 2$, we have
\begin{equation*}
    \sigma_{l,n} > 1-\sum_{\substack{q=1,\\ q\neq j}}^n |[\VK^{(l)}_{nn}]_{jq}| \geq 1 - 2^{1+\frac{1}{\ln{2}}((\ln{(1+1/4)}+2\ln{2}) D - g(l))},\quad g(l) =4^{l-\log_2{r_0/\delta}}
\end{equation*}
With $C_3=(\ln{(1+1/4)}+2\ln{2})$ this leads to
\begin{equation}
     0 < 1-\varepsilon(l) < \bigg(1+\big(1-2^{1+\frac{1}{\ln{2}}(C_3 D - g(l))}\big)/n\lambda\bigg)^{-1},
     \label{eq:garschgorin_bound}
\end{equation}
We note that the bound in Eq. \eqref{eq:garschgorin_bound} holds for $T^{*}>D$ which means that $    l > \log_2(\sqrt{D/2}r_0/\delta)$.\hfill\qedsymbol

\subsection*{C.3 Proof of Corollary \ref{cor:spectrally_bandlim_res}}

Follows from Eq. \eqref{eq:residual_relation_before_norm}-\eqref{eq:defining_eps_l} with $P^{(l)}_{nn} =P^{(l)}_{nn,k} + \big(P^{(l)}_{nn,k}\big)^\perp $, where $P^{(l)}_{nn,k}$ is the projection on the eigenvectors associated with the $k$ largest eigenvalues and $\big(P^{(l)}_{nn,k}\big)^\perp(\widehat{f}^{(l)}([\xbs_n])-f([\xbs_n])) = 0$.\hfill\qedsymbol

\end{document}